\documentclass[11pt]{article}
\usepackage{natbib}
\usepackage{algorithm}
\usepackage{verbatim}
\usepackage[noend]{algpseudocode}

\usepackage{amsmath}
\usepackage{tikz}
\usetikzlibrary{positioning}
\usepackage{amsmath,amsfonts,amsthm}
\usepackage{thmtools}
\usepackage{xspace}
\usepackage{xcolor}
\usepackage{url}
\usepackage{caption}
\usepackage{subcaption}
\usepackage{tikz}
\usepackage{paralist}
\usepackage{libertine}
\usepackage{dsfont}
\usepackage[
pagebackref,
colorlinks=true,
urlcolor=blue,
linkcolor=blue,
citecolor=blue,
]{hyperref}
\usepackage[nameinlink]{cleveref}
\usepackage{enumitem}

\usepackage{color-edits} \addauthor{pco}{brown}
\addauthor{rdk}{red}
\addauthor{mpk}{blue}

\usepackage{thm-restate}

\usepackage[margin=1.05in]{geometry}
\usepackage{txfonts}
\usepackage[T1]{fontenc}

\theoremstyle{definition}
\newtheorem{result}{Theorem} \newtheorem{definition}{Definition}[section]
\newtheorem{example}{Example}[section]
\newtheorem{theorem}{Theorem}[section]
\newtheorem{lemma}[theorem]{Lemma}
\newtheorem{proposition}[theorem]{Proposition}
\newtheorem{corollary}[theorem]{Corollary}

\newtheorem{remark}[theorem]{Remark}

\newcommand{\pb}{\mathbf{p}}
\newcommand{\xb}{\mathbf{x}}
\newcommand{\yb}{\mathbf{y}}

\newcommand{\Acal}{\mathcal{A}}

\newcommand{\Ccal}{\mathcal{C}}
\newcommand{\Dcal}{\mathcal{D}}

\newcommand{\Fcal}{\mathcal{F}}
\newcommand{\Gcal}{\mathcal{G}}
\newcommand{\Hcal}{\mathcal{H}}

\newcommand{\Kcal}{\mathcal{K}}
\newcommand{\Lcal}{\mathcal{L}}

\newcommand{\Rcal}{\mathcal{R}}

\newcommand{\Wcal}{\mathcal{W}}
\newcommand{\Xcal}{\mathcal{X}}
\newcommand{\Ycal}{\mathcal{Y}}

          \newcommand{\RR}{\mathbb{R}}

\newcommand{\reals}{{\mathbb{R}}}

\newcommand{\poly}{{\operatorname{poly}}}

\def\dist{\textnormal{\texttt{dist}}}

\newcommand{\argmax}{\text{argmax}}

\newcommand{\eps}{\varepsilon}

\newcommand{\ind}{\mathbb{I}}

\newcommand{\yt}{\tilde{y}}

\DeclareMathAlphabet\mathbfcal{OMS}{cmsy}{b}{n}

\newcommand{\X}{\mathcal{X}}
\newcommand{\Y}{\mathcal{Y}}

\newcommand{\R}{{\mathbb{R}}}

\newcommand{\C}{\mathcal{C}}

\DeclareMathOperator*{\argmin}{arg\,min}

\newcommand{\Ber}{\mathrm{Ber}}

\newcommand{\lall}{\mathcal{L}_{\mathrm{all}}}
\newcommand{\llip}{\mathcal{L}_{\mathrm{lip}}}
\newcommand{\lcvx}{\mathcal{L}_{\mathrm{cvx}}}
\newcommand{\lprop}{\mathcal{L}_{\mathrm{proper}}}
\newcommand{\lbv}{\mathcal{L}_{\mathrm{BV}}}

\newcommand{\dowal}{{\textsc{\large dowal}}\xspace}
\newcommand{\dowallong}{{Distributional Online Weak Agnostic Learner}\xspace}
\newcommand{\dowalboth}{{\dowallong (\dowal)}\xspace}
\newcommand{\dowalsymbol}{\ensuremath{\Acal_\mathrm{dowal}}\xspace}

\newcommand{\pcalalg}{{\textsc{\large apcal}}\xspace}
\newcommand{\pcalalglong}{{Augmented Proper Calibration}\xspace}
\newcommand{\pcalalgboth}{{\pcalalglong (\pcalalg)}\xspace}
\newcommand{\pcalalgsymbol}{\ensuremath{\Acal_\mathrm{apcal}}\xspace}
\newcommand{\owal}{{\textsc{\large owal}}\xspace}
\newcommand{\owallong}{{Online Weak Agnostic Learner}\xspace}
\newcommand{\owalalgsymbol}{\ensuremath{\Acal_\mathrm{owal}}}

\newcommand{\card}[1] {\left\vert #1 \right\vert}
\newcommand{\set}[1] {\left\{ #1 \right\}}

\newcommand{\omnireg}{\text{-}\mathrm{OmniRegret}}

\newcommand{\calerr}{\text{-}\mathrm{CalErr}}
\newcommand{\maerr}{\text{-}\mathrm{MAErr}}
\newcommand{\ucal}{\mathrm{UCal}}
\newcommand{\vcal}{\mathrm{VCal}}
\newcommand{\pcal}{\mathrm{PCalErr}}
\newcommand{\wprop}{\Wcal_\mathrm{proper}}
\newcommand{\wthres}{\W_{\Th}}
\newcommand{\pxy}{\mathbf{p}, \mathbf{x}, \mathbf{y}}
\newcommand{\owalreg}{\mathrm{OracleReg}}

\newcommand{\dowalreg}{\mathrm{DistOracleReg}}

\newcommand{\thrclass}[1]{\mathrm{Th}^{#1}}
\newcommand{\allthr}{\mathrm{Th}}

\newcommand{\skipping}[1]{}
\newcommand{\Th}{\mathrm{Th}}
\newcommand{\ReLU}{\mathrm{ReLU}}

\newcommand{\monoremap}{{\Upsilon}}

\newtheorem*{theorem*}{Theorem}
\newtheorem*{lemma*}{Lemma}
\newtheorem*{corollary*}{Corollary}
\newtheorem*{proposition*}{Proposition}
\newtheorem*{definition*}{Definition}
\newtheorem*{conjecture*}{Conjecture}

\newcommand{\Nbb}{\mathbb{N}}

\newcommand{\A}{\mathcal{A}}
\newcommand{\D}{\mathcal{D}}
\newcommand{\G}{\mathcal{G}}
\renewcommand{\H}{\mathcal{H}}

\renewcommand{\L}{\mathcal{L}}

\newcommand{\W}{\mathcal{W}}

\DeclareMathOperator*{\E}{\textnormal{\bf E}}
\let\Pr\relax
\DeclareMathOperator*{\Pr}{\textnormal{\bf Pr}}
\newcommand{\1}{{\mathbf{1}}}

\newcommand{\ps}{{p^*}}

\newcommand{\sgn}{\mathrm{sgn}}

\newcommand{\grad}{\nabla}

\title{Near-Optimal Algorithms for Omniprediction}
\author{
Princewill Okoroafor\thanks{\textbf{PO} was in part supported by a Linkedin-Cornell
Fellowship Grant.} \and
Robert Kleinberg \and
Michael P. Kim\thanks{\textbf{MPK:}~This research was supported by a gift to the LinkedIn-Cornell Bowers CIS Strategic Partnership.}~\thanks{Authors listed in reverse-alphabetical order.}
}

\begin{document}
\pagenumbering{gobble}
\begin{titlepage}
\maketitle

\begin{abstract}

Omnipredictors are simple prediction functions that encode loss-minimizing predictions with respect to a hypothesis class $\mathcal{H}$, simultaneously for every loss function within a class of losses $\mathcal{L}$.
In this work, we give near-optimal learning algorithms for omniprediction, in both the online and offline settings.
To begin, we give an oracle-efficient online learning algorithm that achieves $(\mathcal{L},\mathcal{H})$-omniprediction with $\tilde{O}(\sqrt{T \log |\mathcal{H}|})$ regret for any class of Lipschitz loss functions $\mathcal{L} \subseteq \mathcal{L}_\mathrm{Lip}$.
Quite surprisingly, this regret bound matches the optimal regret for \emph{minimization of a single loss function} (up to a $\sqrt{\log(T)}$ factor).
Given this online algorithm, we develop an online-to-offline conversion that achieves near-optimal complexity across a number of measures.
In particular, for all bounded loss functions within the class of Bounded Variation losses $\mathcal{L}_\mathrm{BV}$ (which include all convex, all Lipschitz, and all proper losses)
and any (possibly-infinite) $\mathcal{H}$, we obtain an offline learning algorithm that, leveraging an (offline) ERM oracle and $m$ samples from $\mathcal{D}$, returns an efficient $(\mathcal{L}_{\mathrm{BV}},\mathcal{H},\varepsilon(m))$-omnipredictor for $\varepsilon(m)$ scaling near-linearly in the Rademacher complexity of a class derived from $\mathcal{H}$ by taking convex combinations of a fixed number of elements of $\mathrm{Th} \circ \mathcal{H}$.

 \end{abstract}
\end{titlepage}
\tableofcontents

\clearpage

\pagenumbering{arabic}
\clearpage
\section{Introduction}

In recent years, the research community has investigated learning frameworks beyond the traditional objective of loss minimization.
Amidst research addressing concerns of fairness and robustness,
\citet{gopalan2021omnipredictors} introduced a powerful notion of robust learning, called \emph{omniprediction}.
An omnipredictor is a simple prediction function for binary outcomes that encodes loss-minimizing predictions \emph{simultaneously} for every loss within a broad class of loss functions $\L$.
In contrast to the conventional wisdom---that exploring different loss functions necessitates training different loss minimizers---omnipredictors allow a decision-maker to learn a single model $p$ (an omnipredictor) and subsquently decide which loss function $\ell$ is appropriate for their setting, without retraining $p$.

More technically, omniprediction is a parameterized guarantee, based on a class of loss functions $\L$ as well as a hypothesis class $\H$, and requires loss minimization for every $\ell \in \L$ with respect to $\H$.
\begin{definition*}
A prediction function $p:\X \to [0,1]$ is an $(\L,\H,\eps)$-omnipredictor
if for every $\ell \in \L$
    \begin{gather*}
        \E[\ell(k_\ell \circ p(x),y)] \le \min_{h \in \H}~ \E[\ell(h(x),y)] + \eps.
    \end{gather*}
\end{definition*}
That is, for every loss $\ell$ within the collection, $k_\ell \circ p$ is a loss minimizer for $\ell$.
Here, $k_\ell$ is a data-free post-processing of the predictions given by $p$ that ``type-checks'' the outputs given by $p$ to match $\ell$, specifically:
    $k_\ell(v) \triangleq \argmin_{a \in [0,1]} \E_{y \sim \Ber(v)}[\ell(a,y)]$.
Such a post-processing of $p$ is necessary for the omnipredictor definition, as different losses expect different ``types'' of optimal predictions.\footnote{For example, the optimal prediction for $\ell_2$ is $\Pr[y=1 \vert x]$, whereas for $\ell_1$, it's the most likely outcome in $\{0,1\}$.}

On its face, the guarantee of omniprediction significantly strengthens that of loss minimization with respect to a single loss function, especially when faced with a large, diverse class of losses.
A natural question, facing such a strong definition, is whether omnipredictors exist.
With some thought, it's not hard to see that the optimal predictor $\ps(x) = \Pr[y = 1 \vert x]$ is an omnipredictor for every $\L$ and $\H$.
That is, if we know the true conditional distribution of outcomes given covariates, we can make (statistically) optimal predictions according to every loss function.
In most learning settings, however, the optimal predictor $\ps$ is unattainable, both statistically and computationally.

A more delicate question is whether \emph{efficient} omnipredictors exist, with complexity scaling comparably to the hypotheses the omnipredictor is competing against.
Moreover, if such omnipredictors exist, can we learn them from a finite sample of data, using reasonable amounts of computation?
The original work introducing omniprediction showed the sweeping result that efficient\footnote{
Technically, we say an omnipredictor is efficient if it has a $\poly(1/\eps)$-sized circuit using oracle gates for $h \in \H$.} omniprediction is possible for the class of all convex loss functions.

\begin{theorem*}[Efficient Omnipredictors for Convex Losses Exist \citep{gopalan2021omnipredictors}] 
Let $\L_\mathrm{cvx}$ be the set of all convex, $1$-Lipschitz loss functions.
For any data distribution $\D$, hypothesis class $\H$, and $\eps \ge 0$, there exist efficient $(\L_\mathrm{cvx},\H,\eps)$-omnipredictors.
Further, such omnipredictors can be learned using a weak agnostic learner for $\H$, from $m = O(d_\H \cdot \poly(1/\eps))$ samples from $\D$, where $d_\H$ represents the VC (or fat-shattering) dimension of $\H$. 
\end{theorem*}

The original construction of efficient omnipredictors follows as a consequence of \emph{multicalibration}, a powerful notion for learning introduced in the algorithmic fairness literature by \citet{hkrr}.
Even though multicalibration is defined without any reference to loss minimization, \citet{gopalan2021omnipredictors} showed that, in fact, multicalibration implicitly guarantees loss minimization over every convex loss.
While powerful, deriving efficient omnipredictors from multicalibration is known to be expensive in terms of sample complexity \citep{hkrr,gopalan2022low}.

Subsequent work has investigated
whether omniprediction can be achieved more efficiently than multicalibration and for a more general class of loss functions.
The work of \citet{gopalan2022loss} introduced a general recipe, dubbed \emph{Loss OI}, for achieving omnipredictors (for more general classes of loss functions) via two weaker conditions, \emph{calibration} \citep{dawid1985calibration} and \emph{multiaccuracy} \citep{kim2019multiaccuracy}.
Despite qualitative progress, the quantitative bounds for learning omnipredictors have not improved considerably.
The sample complexity of the omniprediction procedure presented in \cite{gopalan2022loss} scales as $O(1/\eps^{10})$, considerably larger than the optimal $\eps$-dependence for loss minimization of $\Theta(1/\eps^2)$.

One concrete barrier to improved bounds for omniprediction is the reliance on calibration.
At least in the online setting, calibration is known to require asymptotically more samples than loss minimization \citep{qiao2021stronger,dagan2024breaking}.
This ``calibration bottleneck'' has inspired researchers to look for alternative variants of calibration that are strong enough to give omniprediction-style guarantees, but weak enough to be achieved efficiently.
Towards this end, \citet{kleinberg2023u} recently gave an algorithm that achieves simultaneous $O(\sqrt{T})$-regret for all \emph{proper} losses; subsequently, \citet{hu2024calibrationerrordecisionmaking} and \citet{roth2024forecasting} showed that similar bounds can be achieved for swap regret.
These settings are non-contextual, so they do not directly guarantee omniprediction.
Still, the results raise the prospect of simultaneous loss minimization at the price of minimizing a single loss.
Indeed, \citet{garg2024oracle} showed a restricted setting in which online omniprediction is possible, albeit inefficiently, with near-optimal $\tilde{O}(\sqrt{T \log |\Hcal|})$ regret.

Despite a flurry of progress since its introduction, key questions about omniprediction remained unanswered. 
Most importantly, \emph{what is the sample complexity of learning efficient omnipredictors?} In the online setting,
what is the best omniprediction regret attainable? And how does the computational complexity of
learning omnipredictors depend on the number of samples used? Underlying all of these questions
is a methodological question about the relationship between calibration and omniprediction: \emph{Is
calibration necessary for efficient omniprediction, and if not, what weaker properties suffice?}

\paragraph{Our Contributions.}
In this work, we establish the near-optimal 
complexity of learning efficient omnipredictors.
Remarkably, we show that for an extremely broad class of loss functions, there is essentially no cost to omniprediction compared to loss minimization.
In both the online and distributional settings, we give oracle-efficient learning algorithms that establish upper bounds on the omniprediction regret (sample complexity, respectively) that match the lower bounds for minimization of a single loss function, up to low order factors.
Our main contributions can be summarized as follows:

\begin{itemize}
    \item \textbf{Proper Calibration.}~~We define a new variant of calibration---\emph{proper calibration}---that is weaker than full calibration, but suffices for omniprediction.
    Following the Loss OI framework of \citet{gopalan2022loss}, we show that proper calibration actually characterizes ``Decision OI'' and, in doing so, show that proper calibration, paired with $\Delta\L \circ \H$-multiaccuracy \footnote{$\Delta \L$ refers to the derived class $\{\Delta \ell(t) = \ell(t,1) - \ell(t,0) : \ell \in \L \}$}, implies $(\L,\H)$-omniprediction. 
\item \textbf{Near-Optimal Online Omniprediction.}~~We devise an online learning algorithm, based on Blackwell Approachability 
\citep{blackwell,pmlr-v19-abernethy11b}, that achieves proper calibration and multiaccuracy with near-optimal regret.
    Consequently, we obtain $\tilde{O} \left( \sqrt{T \cdot \log |\H|} \right)$ omniprediction regret, for any finite hypothesis class $\H$ and any class $\L \subseteq \L_\mathrm{BV}$ of bounded variation losses; in particular, $\L_\mathrm{BV}$ includes all Lipschitz losses $\L_\mathrm{Lip}$, their restriction to convex losses $\lcvx$, as well as all proper losses $\L_\mathrm{prop}$ and all bounded convex losses, irrespective of Lipschitz continuity.
    \item \textbf{Oracle-Efficiency.}~~We extend this basic approach, showing that, for any hypothesis class $\H$ (including infinite classes), online omniprediction reduces to \emph{Online Weak Agnostic Learning}.
    This reduction establishes an upper bound on the omniprediction regret in terms of
{the sequential Rademacher complexity of the derived class $\Delta \Lcal \circ \Hcal$.}
\item \textbf{Efficient Offline Omnipredictors.}~~To achieve optimal omniprediction in the distributional setting, we develop an offline learning algorithm inspired by our online learning algorithm.
    In some sense, our algorithm can be viewed as a online-to-offline conversion, but running the conversion naively results in \emph{inefficient} omniprediction.
    Instead, we devise a more sophisticated framework that maintains efficiency in samples, computation, and the resulting omnipredictor complexity simultaneously.
    In all, our algorithm returns an efficient randomized omnipredictor 
(that is evaluated by sampling $\poly(1/\eps)$ hypotheses from $\H$ and postprocessing its output), using an \emph{offline} ERM oracle, with sample complexity scaling near-linearly in the \emph{offline} Rademacher complexity of a class derived from $\mathcal{H}$ by taking convex combinations of a fixed number of elements of $\mathrm{Th} \circ \mathcal{H}$.

\end{itemize}
\mpkreplace{
\paragraph{Organization.}
The remainder of the manuscript is organized as follows.
The bulk of the introduction is dedicated to a detailed overview of our contributions (\Cref{sec:overview}).
We highlight our main results, emphasizing an intuitive understanding of our techniques.
We conclude the introduction with a discussion of related works and open directions (\Cref{sec:related}).
Beginning in \Cref{sec:prelim}, we give a complete presentation of our results, including formal definitions and proofs.
Throughout \Cref{sec:overview}, we aim to give pointers to the formal presentations of results in the body of the manuscript.
}

\subsection{Overview of Contributions}
\label{sec:overview}

\paragraph{Basic Preliminaries.}
We work in both an online and distributional prediction setting: inputs $x \in \X$ come from a discrete domain and outcomes $y \in \Y = \{0,1\}$ are binary.
Omniprediction aims to learn a single predictor $p:\X \to [0,1]$ that guarantees loss minimization for every loss within a class $\L \subseteq \set{\ell: [0,1] \times \Y \to \R}$ compared to the best hypothesis from a class $\H \subseteq \set{h:\X \to [0,1]}$.
In the distributional setting, we assume a fixed but unknown distribution $\D$ supported on $\X \times \Y$; when not specified, expectations are taken over $\D$.

We denote the set of all bounded 
loss functions\footnote{Throughout, we restrict our attention to measurable loss functions.} as $\lall = \{\ell:[0,1] \times \Y \to [-1,1]\}$
and consider loss classes defined based on functional properties (in the first argument).
These loss classes include bounded variation $\lbv$, $1$-Lipschitz $\llip$, and (Lipschitz) convex losses $\lcvx$, with the following inclusions.
\begin{gather*}
\lall \supsetneq \lbv \supsetneq \llip \supsetneq \lcvx
\end{gather*}
One very important class of losses are the proper losses $\lprop$.
A loss function $\ell$ is \emph{proper} if for $y \sim \Ber(v)$, predicting $v$ is an optimal strategy; that is, $\E[\ell(v,y)] \le \E[\ell(u,y)]$ for any $u \in [0,1]$.
Note that the class of bounded proper losses $\lprop$ is a subset of $\lbv$, but incomparable to $\llip$.

A key object in the study of omniprediction is the \emph{discrete derivative} of a loss,
$$\Delta \ell (v) = \ell (v, 1) - \ell (v, 0)$$
which intuitively
captures whether $\ell$ distinguishes between $y = 1$ and $y=0$ at a given prediction $v \in [0,1]$.

Calibration is an essential notion for our discussion.
We work with the generic notion of \emph{weighted calibration}, which allows us to instantiate different variants easily.

\begin{definition*}[Weighted Calibration Error]
    Fix a class of functions $\W \subseteq \set{w:[0,1] \to [-1,1]}$,
    called \emph{weight functions}.
    The $\W$-weighted calibration error is defined distributionally for any predictor $p:\X \to [0,1]$,
    and sequentially, 
    for any sequence of predictions, contexts, and outcomes $\pxy$, as follows.
    \begin{gather*}
        \W\calerr(p) = \sup_{w \in \W} \big|\E[w(p(x)) \cdot (y - p(x))]\big| \qquad
        \Wcal\calerr (\pxy) = \sup_{w \in \Wcal} \left| \sum_{t=1}^T w(p_t(x_t)) (y_t - p_t(x_t)) \right|
    \end{gather*}
\end{definition*}
The standard notion of calibration ($\ell_1$-calibration) corresponds to taking
$\W_\mathrm{all} = \set{ w: [0,1] \to [-1,1] }$.
Threshold functions are a key weight class, where
    $\W_{\Th} = \{\Th_\theta : \theta \in [0,1]\}$
for $\Th_\theta(p) = \sgn(\theta - p)$.

Another important notion in the development of omnipredictors is multiaccuracy \citep{hkrr,kim2019multiaccuracy}.
Multiaccuracy is parameterized by a hypothesis class $\H$ and guarantees that the residual in predictions have no nontrivial correlation with any $h \in \H$.
\begin{definition*}[Multiaccuracy]
    Fix a hypothesis class $\H \subseteq \{h:\X \to [-1,1]\}$.  The $\H$-multiaccuracy error is defined distributionally for any predictor $p:\X \to [0,1]$,
    and sequentially, 
    for any sequence of predictions, contexts, and outcomes $\pxy$, as follows.
    \begin{gather*}
    \H\maerr(p) = \sup_{h \in \H}\big| \E[h(x) \cdot (y - p(x))]\big| \qquad
    \H\maerr (\pxy) = \sup_{h \in \H} \left| \sum_{t=1}^T h(x_t) (y_t - p_t(x_t)) \right|
    \end{gather*}
\end{definition*}

A complete set of preliminaries is given in \Cref{sec:prelim}.

\subsubsection{Decision OI as Proper Calibration}
\label{sec:intro:pcal}
Our approach to omniprediction follows the framework put forth by \citet{gopalan2022loss} of Loss Outcome Indistinguishability (OI).
As our first contribution, we give a novel characterization of one of the key components of the framework, Decision OI, in terms of a notion, which we call \emph{proper calibration}.
We develop this notion and its properties in \Cref{sec:proper-cal}.
In particular, working with proper calibration reveals a more efficient scheme for achieving omniprediction, as we overview next.

The OI paradigm \citep{oi} frames learning as indistinguishability.
A predictor $p:\X \to [0,1]$ satisfies OI if outcomes generated based on $p$ ``look like'' real-world outcomes.
Concretely, OI compares samples from the real world $(x,y) \sim \D$ and modeled samples $(x,\yt)$ where $\yt \sim \Ber(p(x))$ is resampled based on the predictor $p$.
Loss OI guarantees $(\L,\H)$-omniprediction via two sub-conditions, Hypothesis OI and Decision OI, using efficient tests defined by the losses $\ell \in \L$ and hypotheses $h \in \H$.
\begin{align}
    &\textbf{Hypothesis OI: \qquad}& \E_{x,y \sim \D}[\ell(h(x),y)] ~~&\approx_\eps \E_{\substack{x \sim \D\\\yt \sim \Ber(p(x))}}[\ell(h(x),\yt)]&&\forall \ell \in \L, h \in \H \label{eqn:HypOI}\\
    &\textbf{Decision OI: \qquad}& \E_{x,y \sim \D}[\ell(k_\ell \circ p(x),y)] ~~&\approx_\eps \E_{\substack{x \sim \D\\\yt \sim \Ber(p(x))}}[\ell(k_\ell \circ p(x),\yt)]&& \forall \ell \in \L\label{eqn:DecOI}
\end{align}
Under these OI conditions, omniprediction follows immediately.
\begin{theorem*}[\citet{gopalan2022loss}] \label{thm:loss-oi}
    If $p$ satisfies $(\L,\H,\eps)$-Hypothesis OI and $(\L,\eps)$-Decision OI, then $p$ is an $(\L,\H,2\eps)$-omnipredictor.
\end{theorem*}
This OI argument follows by switching from expectations in the real world (i.e., the LHS of (\ref{eqn:HypOI}) and (\ref{eqn:DecOI})) to the modeled world (RHS).
In the modeled world, outcomes are sampled from our predictor $\yt \sim \Ber(p(x))$, so $k_\ell \circ p(x)$ is the \emph{statistically-optimal} predictor; thus, in the real world, $\E[\ell(h(x),y)] \ge \E[\ell(k_\ell \circ p(x),y)] - 2\eps$ by OI (losing an additive $\eps$ to switch from the real world to the modeled world and back).

\cite{gopalan2022loss} go on to show that Hypothesis OI is equivalent to a certain multiaccuracy condition, for the class $\Delta \L \circ \H = \{\Delta \ell \circ h : \ell \in \L,\ h \in \H\}$.
They also show how the Decision OI error can be expressed as the following weighted calibration condition.
\begin{align}
    \big|\E_{x,y \sim \D}[\ell(k_\ell \circ p(x),y)] - \E_{\substack{x \sim \D\\\yt \sim \Ber(p(x))}}[\ell(k_\ell \circ p(x),\yt)]\big|
    &= \big|\E_{x,y \sim \D}[\Delta \ell(k_\ell \circ p(x)) \cdot (y-p(x))]\big|\label{eqn:DecOIasWCal}
\end{align}
While standard calibration implies this weighted calibration condition, we argue that it can be simplified into the following notion, based on proper losses.
\begin{definition*}[Proper Calibration]
    Let $\W_\mathrm{proper} = \{\Delta \ell : \ell \in \lprop\}$.
    A predictor $p:\X \to [0,1]$ is $\eps$-proper calibrated if $\W_\mathrm{proper}\calerr(p) \le \eps$; that is,
    \begin{gather*}
        \sup_{\ell \in \lprop} \big|\E[\Delta\ell(p(x)) \cdot (y - p(x)) ]\big| \le \eps
    \end{gather*}
\end{definition*}
At first, this notion may seem insufficient to deal with improper losses.
But when we compose a non-proper loss $\ell \in \L$ with the optimal post-processing $k_\ell$, we can effectively treat the loss as proper.
Formally, for any loss $\ell$, there exists a proper loss $\ell_\mathrm{proper} \in \lprop$ such that $\Delta\ell(k_\ell(\cdot)) = \Delta \ell_\mathrm{proper}(\cdot)$.
The restriction to proper calibration allows us to exploit structural properties of proper losses.
In particular, as in \citep{li2022optimization,kleinberg2023u}, we lean on a characterization of proper losses in terms of threshold functions.
In all, we can show the following characterization of Decision OI in terms of proper calibration, and in terms of threshold-weighted calibration, which can be achieved efficiently.

\begin{result}\label{result:proper-equiv}
$\L_\mathrm{all}$-Decision OI, Proper Calibration, and $\W_{\Th}$-calibration are equivalent.
Formally, for any predictor $p:\X \to [0,1]$ (or sequence of predictions $\mathbf{p}$), the errors can be related as follows:
\begin{gather*}
    \W_{\Th}\textrm{-}\mathrm{CalErr}(p) \le \wprop\calerr(p) = \L_\mathrm{all}\textrm{-}\mathrm{DecOIErr}(p) \le 2 \cdot \W_{\Th}\textrm{-}\mathrm{CalErr}(p)
\end{gather*}
Thus, if $p$ is $(\Delta \L \circ \H,\eps)$-multiaccurate and $(\W_\Th,\eps)$-calibrated, then $p$ is an $(\L,\H,3\eps)$-omnipredictor.
Further, there exists an online algorithm that guarantees $O(\sqrt{T \log T})$ proper calibration regret.
\end{result}

With this characterization in hand, we can enforce proper calibration (and thus Decision OI) by auditing the predictor with threshold functions.
Incorporating proper calibration---rather than $\ell_1$-calibration---into the algorithmic framework of \citet{gopalan2022loss} results in statistical improvements, but does not realize sample optimality.
For completeness, we analyze this approach in \Cref{app:mcboost}.
Importantly, as we describe in the next section, we can enforce proper calibration simultaneously with multiaccuracy in the online setting to obtain near-optimal omniprediction regret, which we subsequently leverage to establish statistical near-optimality in the offline setting.

\subsubsection{Blackwell Approachability for Online Omniprediction}
\label{sec:intro:blackwell}

In \Cref{sec:online_omni}, we describe our strategy to learn predictions that obtain near-optimal omniprediction regret.
Our algorithmic approach to online omniprediction is based on Blackwell Approachability \citep{blackwell} to solve a vector-valued game, defined by the proper calibration and multiaccuracy constraints.
At a high level, we employ a framework developed by \citet{pmlr-v19-abernethy11b} to use multiplicative weights over the set of constraints defined by the vector-valued game.
Within this framework, we devise an explicit ``halfspace oracle'' used by the forecaster to play optimally given the dual weights.
We begin with an overview of the intuition behind our approach, and describe our oracle-efficient implementation in the subsequent section.

To build intuition, we focus on the omniprediction setting where the loss class $\L$ and hypothesis class $\H$ are finite.
As discussed above, our learning goal is to achieve low regret for proper calibration and online multiaccuracy simultaneously.
\footnote{
For completeness, we define the problem, and show how to achieve oracle-efficient optimal regret bounds in \Cref{sec:online-ma}.} 
Per Theorem~\ref{result:proper-equiv}, the multiaccuracy and proper calibration regret necessary for omniprediction can be expressed as follows.
\begin{gather*}
    \textbf{Multiaccuracy:}~~~ \max_{\substack{\ell \in \L,\\ h \in \H}}~ \left| \sum_{t=1}^T \Delta \ell \circ h(x_t) (y_t - p_t) \right| \qquad
    \textbf{Proper Calibration:}~~~ \sup_{\theta \in [0,1]} \left| \sum_{t=1}^T \Th_\theta(p_t) (y_t - p_t) \right|
\end{gather*}
In other words, the learner's job is to choose a sequence of predictions $\mathbf{p}$ that, when playing against an adversarial sequence of contexts $\mathbf{x},\mathbf{y}$, guarantee low regret over a \emph{worst-case} choice over loss-hypothesis pairs $(\ell,h) \in \L \times \H$ (multiaccuracy) and thresholds $\theta \in [0,1]$ (proper calibration).
In such a setting, we can formulate the learner's task as a Blackwell Approachability game with vector-valued payoffs.

Concretely, we can imagine the following finite-dimensional payoff vector $\vec{u}$.
Given a prediction $p$, input $x \in \X$, and outcome $y \in \{0,1\}$,
the multiaccuracy constraints are indexed by loss-hypothesis pairs for each $\ell \in \L$ and $h \in \H$ and signs $s \in \{+,-\}$, and the proper calibration constraints are indexed by (appropriately-discretized) thresholds $\theta \in \{0,\eps,\hdots,1\}$, also signed by $s \in \{+,-\}$:
\begin{gather*}
    \textbf{MA:}~~\vec{u}_{\ell,h,s}(p,x,y) = s \cdot \Delta \ell \circ h(x) \cdot (y-p)
    \qquad\qquad
    \textbf{PC:}~~\vec{u}_{\theta,s}(p,x,y) = s \cdot \Th_\theta(p) \cdot (y-p)
\end{gather*}
The learner's goal in playing this game is to make the payoff vector $\vec{u}$ ``approach'' the origin, driving the worst-case violation of any constraint towards $0$.
To achieve this approachability, we can run multiplicative weights over the coordinates in $\vec{u}$, to maintain a dual halfspace $\vec{w}$ to witness violations of the constraints;
then, given the halfspace, the algorithm computes an explicit optimal strategy to hedge against the choice of outcome $y_t \in \{0,1\}$.
Concretely, given a context $x_t$, we consider the following weighted function $f$, which (as a function of $p$), maps predictions on the interval $[0,1]$ to the range $[-1,1]$.
\begin{equation}\label{eq:finite-f-tech}
f(x_t,p) = \sum_{\ell,h,s \in \{\pm\}} w_{\ell,h,s} \cdot s \cdot \Delta \ell \circ h(x_t) + \sum_{\theta, s \in \{\pm\}} w^t_{\theta,s} \cdot s \cdot \Th_{\theta}(p)
\end{equation}
Intuitively, $f(x_t,p) \cdot (y-p)$ captures the error (in multiaccuracy or proper calibration) that the learner may incur from predicting $p$ on outcome $y$.
Note that, for a fixed $p$, the adversary may choose the sign of this error through the choice of $y$.
Thus, to minimize the potential error incurred---regardless of the outcome $y$---our algorithm plays a mixture between adjacent predictions $p$ and $p'$ where $f(x_t,p) \le 0$ and $f(x_t,p') > 0$, so that potential negative and positive error cancel in expectation.\footnote{Note that the constraint functions defined by $\Th_\theta$ are not continuous, so $f(x_t, \cdot)$ need not have a zero between $p$ and $p'$.}
By choosing the prediction interval appropriately, we can guarantee that the combined regret grows slowly.
\begin{result}
\label{result:online}
    There exists an online algorithm that for any finite class of bounded loss functions $\L$ and finite hypothesis class $\H$, guarantees expected $(\L,\H)$-omniprediction regret $O \left( \sqrt{T \log(\card{\H}\card{L}T}) \right)$.
\end{result}

\subsubsection{Oracle-Efficient Online Omniprediction}
\label{sec:intro:online}

Of course, the framework described above suffers from tracking weights for each multiaccuracy constraint explicitly.
To achieve omniprediction more efficiently, or for infinite loss/hypothesis classes, we need a more sophisticated approach.
To achieve these goals, we must generalize the algorithmic approach to avoid explicit dependence on a finite class of loss functions $\L$ and hypotheses $\H$. We will do this using an \emph{Online Weak Agnostic Learner} for the class $\Delta \Lcal \circ \Hcal$, introduced in  \citep{chen2012online,brukhim2020online,beygelzimer2015optimal}.

Our online algorithm for achieving $(\L,\H)$-omniprediction consists of an interaction between two sub-algorithms.
\begin{itemize}
    \item The first algorithm is an \owallong (\owal).
    The \owal is responsible for producing a sequence of adaptively-chosen functions $q_1,\hdots,q_T:\X \to [-1,1]$ such that enforcing a sequential multiaccuracy condition with respect to the $q_t$ implies low multiaccuracy regret with respect to all of $\Delta \L \circ \H$.
    This online sparsification task is reminiscent of ``scaffolding sets'' problem, studied by \citet{burhanpurkar2021scaffolding} in the offline setting.
    \item The second algorithm is an \pcalalgboth.  The \pcalalg produces a sequence of prediction functions $p_t$ with low proper calibration regret on the sequence of $(x_t,y_t)$; simultaneously, it enforces multiaccuracy with respect to the sequence of tests provided by the \owal.
\end{itemize}
In more detail, these sub-algorithms satisfy the following semantics.

\paragraph{\owallong.}
We use the \owal abstraction to produce a sequence of functions such that $\set{q_t}$-multiaccuracy 
implies online $\Delta \L \circ \H$-multiaccuracy.
Specifically, the \owal solves the following online learning task:
identify a sequence of functions $q_1,\hdots,q_T:\X \to [-1,1]$ whose sequential multiaccuracy with respect to the sequence $p_t$ is at least as large as the multiaccuracy violation of every $c \in \Delta \L \circ \H$ in hindsight:
\begin{gather}
    \max_{c \in \Delta \L \circ \H} \sum_{t=1}^T c(x_t)(y_t - p_t(x_t)) \le \sum_{t = 1}^T q_t(x_t)(y_t-p_t(x_t)) + \owalreg_{\Delta \L \circ \Hcal}
\end{gather}
In other words, auditing against the sequence of $T$ different $q_t$ test functions (one at each time step) guarantees that the overall multiaccuracy regret is bounded.

\paragraph{\pcalalglong.}
The \pcalalg algorithm takes in a sequence of data $(x_1,y_1),\hdots,(x_T,y_T)$, and is responsible for producing a sequence of predictor functions $p_1,\hdots,p_T:\X \to [0,1]$ 
that satisfy proper calibration over the given sequence.
Additionally, the proper calibrator's input is augmented to receive a sequence of functions 
$q_1,\hdots,q_T:\X \to [-1,1]$, and is responsible for simultaneously ensuring a sequential multiaccuracy guarantee with respect to these functions. In all, the \pcalalg is required to guarantee the following regret bounds on its sequence.
\begin{align}\label{eq:apcal_guarantee}
\sup_{\theta \in [0,1]} \left| \sum_{t=1}^T \Th_\theta(p_t (x_t)) (y_t - p_t (x_t)) \right| \le \tilde{O}\left(\sqrt{T}\right) & \qquad \mbox{and} \qquad
    \sum_{t=1}^T q_t(x_t)(y_t - p_t(x_t)) \le \tilde{O}\left(\sqrt{T}\right)
\end{align}
Our implementation of the \pcalalg algorithm follows a similar approach to the finite omniprediction algorithm described earlier.
As before, the algorithm runs multiplicative weights over each proper calibration constraint, indexed by the thresholds $\theta \in \{0,\eps,\hdots,1\}$.
In contrast, however, rather than maintaining a weight per multiaccuracy constraint, our \pcalalg maintains a \emph{single} weight for multiaccuracy.
This weight determines the importance (and is updated according to violations) of the sequence of $q_t$ functions that it receives from the \owal.

Our eventual algorithm is still based on Blackwell Approachability, but approaches multiaccuracy implicitly through the sparse set of constraints provided by the \owal.
Specifically, in place of the high-dimensional payoff vector $\vec{u}(p,x,y)$ that enforced multiaccuracy constraints in the algorithm sketched above, we substitute the scalar $q_t(x)(y-p)$ where $q_t$ is the output of the \owal at time $t$. (The proper calibration constraints are still enforced using 
a $2T$-dimensional vector $\vec{u}_{\theta,s}(p,x,y)$ as before.) Then,
at each step, our online algorithm bases its decision on a new function $f(x_t, p)$. 
\begin{equation}
f(x_t,p) =  w_{\mathrm{ma}} \cdot q_t (x_t) + \sum_{\theta, s \in \{\pm\}} w^t_{\theta,s} \cdot s \cdot \Th_{\theta}(p)
\end{equation}
Instead of maintaining weights over $|\Hcal| \times |\Lcal| + 2T$ multiaccuracy and proper calibration 
constraints to compute the function as defined in \Cref{eq:finite-f-tech}, 
we have collapsed all of the multiaccuracy constraints into a single dimension represented by the weight $w_{\mathrm{ma}}$, 
thereby reducing the number of weights the algorithm must maintain to $2T+1$.
To the best of our knowledge, this 
method of ``sparsifying'' Blackwell Approachability by outsourcing many dimensions of the payoff vector
to an auxiliary algorithm (in this case, the \owal) that detects constraint violations in those dimensions
is novel and may be of independent interest.

Once we implement each of these components, the online omniprediction guarantee follows immediately.
Combining the guarantees from the \pcalalg and \owal, we obtain an oracle-efficient online learning algorithm with the following properties.
\begin{result}
    There exists an oracle-efficient online algorithm that for any class of loss functions $\L$ and any hypothesis class $\H$, given an online weak agnostic learner for $\Delta \L \circ \H$, guarantees expected $(\L,\H)$-omniprediction regret $O\left(\sqrt{T \log T} +  \owalreg_{\Delta \L \circ \Hcal} (T)\right)$.
\end{result}

\paragraph{Basis Decompositions for Infinite Loss Classes.}

As in prior works, given an infinite class of loss functions $\L$, we obtain omniprediction by designing a basis $\G$ that allows us to get uniform approximations to $\Delta \ell \in \Delta\L$.
Then, using this basis in place of $\Delta \L$, we obtain omniprediction via a weak learner for $\G \circ \H$.

Using and refining bases developed by \cite{gopalan2022loss,gopalan2024regression}, we give online omniprediction guarantees for notable classes of loss functions.
We define and construct these bases for loss classes formally in \Cref{sec:online_omni}, but give a high-level summary here.
Concretely, we obtain omniprediction guarantees from \owal as outlined in \Cref{table:owals}.
\begin{table}[h!]
    \centering
    \begin{tabular}{c|c}
         {\bf Loss Class}& {\bf OWAL Oracle} \\
         \hline
         GLM Losses $\L_\mathrm{GLM}$ & $\H$\\
         $1$-Lipschitz, Convex $\lcvx$& $\ReLU^{1/T} \circ \H$\\
         $1$-Lipschitz $\llip$& $\Th^{1/T} \circ \H$\\
         Proper $\lprop$& {$\Th \circ \H$} \\
         Bounded Variation $\lbv$ & {$\Th \circ \H$}
    \end{tabular} 
    \caption{Summary of Loss Classes $\L$ and the associated Online Weak Agnostic Learning Oracle sufficient to achieve $(\L,\H)$-omniprediction for hypothesis class $\H$ via \Cref{alg:online-omni}.}
    \label{table:owals}
\end{table}

A few comments are in order.
First, notationally, for two classes of functions, we use $\G \circ \H$ to refer to the class of compositions $\{g \circ h : g \in \G,\ h \in \H\}$.
The oracles we need are for hypothesis classes derived by composition with $\H$.
Specifically, we consider applying ReLU's and threshold functions.
\begin{gather*}
\ReLU = \{\ReLU_\theta(\cdot): \theta \in [0,1]\} \text{ where } \ReLU_\theta(p) = \max\{0, p - \theta\}\\
\Th = \{\Th_\theta(\cdot) : \theta \in [0,1]\} \text{ where } \Th_\theta(p) = \sgn(\theta-p).
\end{gather*}
Importantly, while these classes are defined for $\theta \in [0,1]$, we use the superscript notation to denote a fixed precision.
For instance, $\Th^{\gamma}$ considers threshold functions $\Th_\theta$ for $\theta \in \{0,\gamma,2\gamma,\hdots,1\}$.

We obtain omniprediction for Lipschitz, convex loss functions leveraging an \owal for (fixed-precision) ReLU's over hypotheses in $\H$.
Notably, this oracle is different (and indeed a stronger oracle) than that used by \cite{gopalan2021omnipredictors} to obtain $(\lcvx,\H)$-omnipredictors; indeed, $\H$-multicalibration is sufficient for convex omniprediction and can be obtained using a weak agnostic learner for $\H$.
But as discussed, this strategy is statistically less efficient.
Per \citep{gopalan2022loss}, we can achieve omniprediction for Generalized Linear Model losses using a learner for $\H$ (because $\Delta \L_\mathrm{GLM} \circ \H = \H$).

The remaining loss classes considered use an oracle for some class of threshold functions over $\H$.
Omniprediction for the weakest of these classes, Lipschitz losses, can be achieved using an online learning oracle for fixed-precision threshold functions.
While bounded variation losses (and proper losses which are a subset $\lprop \subseteq \lbv$) also follow from an oracle for thresholds applied to $\H$, the basis necessary to obtain uniform approximations of $\ell \in \lbv$, is actually uncountably infinite.
Indeed, the basis consists of a V-shaped loss for \emph{every} $v \in [0,1]$, not simply a discrete finite approximation of the interval.
Thus, in principle, we require an oracle for learning arbitrarily precise thresholds over $h \in \H$.

\subsubsection{Near-Optimal Offline Omnipredictors}
\label{tech:offline}
\label{sec:intro:offline}

Online-to-offline conversion is a classic tool from learning theory:  given a data set sampled from the target distribution $\D$, simulate the online learner on the samples (in arbitrary order); then, output the uniform mixture over predictions suggested at each ``timestep'' of the online algorithm.
With a near-optimal online algorithm for omniprediction regret, one may hope that a corresponding distributional learning algorithm would follow from a standard online-to-offline conversion.
Running such a conversion naively, however, results in an \emph{inefficient} omniprediction algorithm, in a number of ways.
\begin{itemize}
    \item \textbf{Representation Complexity.}~
    ~To achieve online omniprediction unconditionally (without assuming an oracle), our approach maintains explicit weights for each hypothesis $h \in \H$.
    Converting this online algorithm into an offline omnipredictor produces a circuit of enormous complexity, scaling linearly in $\card{\H}$.
    Ideally, our offline algorithm would prove the existence of efficient omnipredictors for all hypothesis classes.
    \item \textbf{Statistical Complexity.}~The regret bounds we achieve are near-optimal in the online setting, but naturally, depend on the \emph{sequential dimension} of the class $\Delta \L \circ \H$.
    The sequential dimension can be arbitrarily (even infinitely) larger than the statistical dimension (e.g., Littlestone vs.\ VC).
    In the standard conversion of the online learner, we inherit the dependence on the sequential dimension.
    \item \textbf{Computational Complexity.}~Finally, in aiming for computational efficiency, our online algorithm relies on an online weak agnostic learner.  To learn omnipredictors in the distributional setting, we may hope to reduce the task of learning offline omnipredictors to an empirical risk minimizer.
\end{itemize}
We address each of these inefficiencies.
In particular, we develop two parallel strategies for adapting our online algorithm to the distributional setting.
The first approach, which runs multiplicative weights over a finite cover of $\Delta \L \circ \H$, is computationally-inefficient, but achieves near-optimal statistical complexity \emph{for every} class of loss functions $\L$.
The second approach manages to achieve efficiency based on an (offline) Empirical Risk Minimizer (ERM), while achieving statistical optimality for the class of proper losses $\lprop$ (or alternatively all bounded variation losses $\lbv$).
Technically, however, this result is incomparable to our first approach.
Our generalization argument leverages properties of the class $\Th \circ \H$, so we pay for the Rademacher complexity of $\Th \circ \H$ even for loss classes that have a much simpler $\Delta \L \circ \H$ (e.g., convex losses $\lcvx$).

\paragraph{Statistically-Efficient Offline Omnipredictors.}
To begin, in \Cref{sec:offline-omni}, we give a conversion strategy that addresses the representation complexity and statistical complexity.
While online-to-batch conversions are standard in learning theory, bounding the resulting error in our setting turns out to be very subtle.
As usual, the distributional error the offline omnipredictor suffers can be broken down in terms of the regret achieved by the online learner and the generalization of the empirical statistics to their distributional analogues.
We discuss each contribution separately.

To give a sufficiently strong bound on the regret, we need to design an online learner for omniprediction that, when fed data from the distribution $\D$, achieves error that scales with the statistical
complexity, rather than the sequential complexity of the derived class.
For this goal, we draw from the literature on \emph{Online Hybrid Learning}.
In this framework, features are drawn i.i.d. from a distribution, while the labels may be adversarially chosen. Similarly, in our offline setting, although the features are i.i.d., the labels $y_t - p_t(x_t)$ are adaptively chosen since the $p_1, \ldots, p_T$ are constructed in an online manner. \citet{wu2022expectedworstcaseregret} and \citet{lazaric2009hybrid} demonstrate that applying a multiplicative weights algorithm over a "stochastic cover" of the hypothesis class yields an online hybrid learner with near-optimal dependence on the offline dimension of the hypothesis class for any convex, Lipschitz loss. 
Since the loss in our \owal is convex and 1-Lipschitz, the statistical result follows immediately.
While this strategy requires an explicit (inefficient) execution of multiplicative weights over the ``stochastic cover'' of hypotheses, the representation complexity arises from the fact that
at each timestep, a single hypothesis is sampled.
In all, the online-to-offline conversion outputs a distribution supported on $T = \poly(1/\eps)$ (postprocessed) hypotheses.

With an algorithm that allows us to achieve low regret on the empirical statistics, we need to prove generalization.
Recall that in the online setting, our strategy was to bound omniprediction regret above by proper calibration error plus
multiaccuracy error. However, upper bounds on multiaccuracy error in the sequential setting don't generalize readily to 
the distributional setting, because the ``labels'' used to define multiaccuracy error are the residuals, $y_t - p_t(x_t)$,
rather than the raw labels $y_t$ themselves. Hence, even when the training set $\{(x_t,y_t)\}_{t=1}^T$ consists of i.i.d.\ draws 
from $\Dcal$, the residuals $y_t - p_t(x_t)$ depend on an entire initial segment of the training set
because $p_t$ is trained using an online algorithm that sees the examples $(x_1,y_1),\ldots,(x_{t-1},y_{t-1})$.
This statistical adaptivity makes it very difficult (if not impossible) to prove that the multiaccuracy error converges to its distributional quantity.

To circumvent this difficulty, rather than proving separate generalization bounds for proper calibration error
and multiaccuracy error, we prove generalization bounds for the statistics that are actually essential for omniprediction.
Namely, we give generalization for the two expected losses that appear in the definition of omnipredictors: the loss of the best hypothesis in $\Hcal$ and the loss obtained by postprocessing
the omnipredictor's predictions using $k_{\ell}$. The first of these two losses has no dependence on the
sequence of predictors selected by our algorithm, so its generalization bound 
follows by standard Rademacher complexity arguments.
The second generalization bound
requires more care. Here, we make use of the fact that loss of the post-processed predictions, 
$\ell(k_{\ell}(p(x)),y)$, is equivalent to a proper loss $\ell_{\mathrm{proper}}(p(x),y)$, which in turn 
is a weighted combination of ``V-shaped losses'' as in \citep{li2022optimization,kleinberg2023u}. 
In fact, since our omnipredictor always outputs predictions in the
discrete set $[1/T] = \{0, 1/T, 2/T, \ldots, 1\}$, we argue that we 
only need to prove generalization for each of the $T\!+\!1$ V-shaped
losses indexed by this set. This set of losses is small enough that
a simple martingale argument suffices to conclude the generalization bound.

Ultimately, after running our online-to-offline conversion, the algorithm returns a randomized predictor, constituting a distribution over $\poly(1/\eps)$ deterministic predictors each obtained by postprocessing the output of a hypothesis from $\H$.
Throughout the manuscript, we abbreviate this property by writing that the predictor ``mixes over $\poly(1/\eps)$ postprocessed hypotheses from $\H$.''
In all, we derive the following theorem, which holds for any $\H$ unconditionally.

\begin{result}
\label{result:offline}
    There exists an algorithm $\A$ that for any distribution $\D$ supported on $\X \times \{0,1\}$, for any class of loss functions $\L \subseteq \lbv$, any hypothesis class $\H$, and $\eps > 0$, learns an $(\L,\H,\eps)$-omnipredictor with the following properties:
    \begin{itemize}
        \item $\A$ returns a randomized omnipredictor that mixes over $\poly(1/\eps)$ postprocessed hypotheses from $\H$.
\item $\A$ uses $m \le \tilde{O}(d_{\Delta \Lcal \circ \Hcal}/\eps^2)$ samples drawn i.i.d.\ from $\D$, where $d_{\Delta \L \circ \H}$ denotes the VC dimension of $\Delta \L \circ \H$ or the fat-shattering dimension at scale $\eps$ in the case of a real-valued class.
\end{itemize}
\end{result}
As in the online case, we can instantiate the result using an appropriate basis in place of $\Delta \L$.
As such, for each of the loss-hypothesis class pairs highlighted in \Cref{table:owals}, the distributional algorithm depends on the corresponding
offline statistical complexity of the class.

In other words, even though the result is computationally-inefficient, it establishes a near-optimal statistical complexity that adapts to the complexity of the loss class $\L$.
If better bases are developed for $\Delta \L$, the statistical dependence of the result will adapt to depend on the complexity of the resulting class.

\paragraph{Oracle-Efficient Offline Omnipredictors.}
While the offline algorithm described above 
is sample-efficient and representation-efficient, it suffers in one respect:  it is computationally-inefficient and runs in time linear in the number of hypotheses.
Addressing this inefficiency, in \Cref{sec:oracle-omni}, we devise an offline algorithm to learn efficient omnipredictors with respect to any loss class $\Lcal \subseteq \lbv$ and hypothesis class $\Hcal$ that
is computationally-efficient given an ERM oracle for the hypothesis class $\Th \circ \Hcal$.
In all, we show the following theorem.
\begin{result}[Oracle-Efficient Omniprediction]
\label{res:offline_bv}
    There exists an oracle-efficient algorithm $\A$ that for any distribution $\D$ supported on $\X \times \{0,1\}$, for any class of loss functions $\L \subseteq \lbv$, any hypothesis class $\H$, and $\eps > 0$, learns an $(\L,\H,\eps)$-omnipredictor with the following properties:
    \begin{itemize}
        \item $\A$ returns a randomized omnipredictor that mixes over $\poly(1/\eps)$ postprocessed hypotheses from $\H$.
        \item Given $m$ samples, the error of the omnipredictor returned by $\A$ scales with the \emph{offline} Rademacher complexity $\eps(m) = \tilde{O} \left( \mathsf{rad}_m(\Ccal) + \mathsf{rad}_m(\Th \circ \mathrm{conv}(\Ccal)) \right) + O \left(\sqrt{\ln m / m} \right)$ where $C$ is the negative closure of $\Th \circ \Hcal$.
\item $\A$ is oracle-efficient, making $\poly(1/\eps)$ calls to an \emph{offline} ERM oracle for $\Th \circ \H$.
    \end{itemize}
\end{result}
\sloppy Our oracle-efficient offline algorithm is an adaptation of the online algorithm for achieving $(\L,\H)$-omniprediction, which simultaneously maintains low proper calibration regret and low multiaccuracy regret. Recall that our online algorithm consists of an interaction between two sub-algorithms: the \pcalalgboth and the \owallong (\owal).
One way to achieve computational efficiency would be to make the \owal computationally efficient with respect to an ERM oracle without losing sample efficiency. However, the problem of designing hybrid online learning algorithms (of which the \owal is a special case) that are both oracle-efficient and sample-efficient has remained an open problem for many years. 
We bypass this challenge of implementing an oracle-efficient \owal by introducing a different interface, a \dowallong (\dowal), to replace the \owal.
The \dowal is responsible for producing the sequence of functions such that $\set{q_t}$-multiaccuracy 
implies $\Delta \L \circ \H$-multiaccuracy over the distribution $\D$. 
More formally, the \dowal must solve the following hybrid learning task:
identify a sequence of functions $q_1,\hdots,q_T:\X \to [-1,1]$ that 
witness the distributional multiaccuracy error (over $\D$) of the
sequence $p_t$ as effectively as the best
fixed $c \in \Delta \L \circ \H$ in hindsight:
\begin{gather}
    \max_{c \in \Delta \L \circ \H} \sum_{t=1}^T \E_{(x,y) \sim \D}[c(x)(y - p_t(x))] \le \sum_{t = 1}^T \E_{(x,y) \sim \D}[q_t(x)(y-p_t(x))] + \mathrm{Regret_{dowal}}(T) \label{eqn:dowal}
\end{gather}
The key distinction between the \owal and the \dowal is that the \owal provides a guarantee over the online sequence while the \dowal provides a guarantee over the distribution.

Importantly, the oracle-efficient offline algorithm partitions its samples drawn from $\D$ into two sets of size $m = T$, denoted by 
$D_{\mathrm{apcal}}$ and $D_{\mathrm{dowal}}$, and provides each
sub-algorithm only with its half of the data. 
By splitting samples, and designing 
careful interfaces through which the sub-algorithms interact, we are able to maintain the martingale 
difference property we need for generalization of the proper calibration error, while using a uniform
convergence bound to argue generalization for the multiaccuracy tests generated by the \dowal.
In particular, a martingale argument on the \pcalalg guarantee in \Cref{eq:apcal_guarantee} leads to the following distributional guarantee
\begin{align*}
\sup_{\theta \in [0,1]} \left| \E_{(x,y) \sim \Dcal} [\Th_\theta(p_t (x)) (y - p_t (x))] \right| \le \tilde{O}\left(\sqrt{T}\right) & \qquad \mbox{and} \qquad
    \E_{(x,y) \sim \Dcal} [q_t(x)(y - p_t(x))] \le \tilde{O}\left(\sqrt{T}\right)
\end{align*}
Combining this with \Cref{eqn:dowal} allows us to conclude that the uniform mixture $\hat \pb$ over the sequence of predictors $p_1,\hdots,p_T$ is approximately proper calibrated and $\Delta \L \circ \H$-multiaccurate over the distribution $\D$ with error that scales as $\tilde{O}\left(\sqrt{1/T}\right) + \frac{1}{T} \mathrm{Regret_{dowal}}(T)$

\paragraph{Implementing the \dowal via Low-Regret Learning}
The \dowal must achieve the regret guarantee expressed in \Cref{eqn:dowal}.
We implement the \dowal by designing a procedure that satisfies a corresponding regret bound with respect to the \emph{empirical} counterparts of the distributional expected values.
\begin{gather}
    \max_{c \in \Delta \L \circ \H} \sum_{t=1}^T \E_{(x,y) \sim D_\mathrm{dowal}}[c(x)(y - p_t(x))] \le \sum_{t = 1}^T \E_{(x,y) \sim D_\mathrm{dowal}}[q_t(x)(y-p_t(x))] + \tilde{O} \big( \sqrt{T} \big) \label{eqn:empiricaldowal}
\end{gather}
Then, separately, we establish uniform convergence of the empirical quantities of interest to their distributional values.

The problem of selecting $q_1,\ldots,q_T$ to minimize the regret on the right side of~\eqref{eqn:empiricaldowal} 
reduces to online linear optimization, using the fact that the distribution $D_{\mathrm{dowal}}$ has support size $m$. 
We embed each $c \in \Delta \L \circ \H$ as a vector in $[-1,1]^m$ by evaluating $c(x)$ at each $x \in \{x_1,\ldots,x_m\}$,
and we observe that in this embedding, two convenient properties hold.
\begin{enumerate}
    \item The objective function $\E_{(x,y) \sim D_\mathrm{dowal}}[c(x)(y - p_t(x))]$ is a linear function 
of the vector representing $c$.
    \item An ERM oracle for $\Delta \L \circ \H$ enables us to optimize any linear function over the
embedding of $\Delta \L \circ \H$.
\end{enumerate}
Our implementation of the \dowal solves the regret minimization problem embodied by~\eqref{eqn:empiricaldowal} 
using a Follow-The-Regularized-Leader (FTRL) procedure. Although our online learning problem is
$m$-dimensional, we show that its decision set and loss vectors have a structure that ensures a dimension-independent 
$O(\sqrt{T})$ regret bound for FTRL with the entropy regularizer. As for the oracle complexity of implementing FTRL,
each iteration requires approximately minimizing a strongly convex function over the embedding of 
$\Delta \L \circ \H$ in $[-1,1]^m$; we reduce this strongly convex minimization problem to a sequence of ERM oracle calls
via the Frank-Wolfe procedure.

\paragraph{Uniform Convergence for the \dowal.}
The \dowal is initialized with a dataset $D_\mathrm{dowal}$ drawn from $\D$ and guarantees low empirical multiaccuracy regret.
The key question is: how many samples must $D_\mathrm{dowal}$ contain in order to guarantee that the empirical quantities in \Cref{eqn:empiricaldowal} converge to imply the distributional guarantee of \Cref{eqn:dowal}?

To establish convergence, we analyze the Rademacher complexity of functions of the form $q_t(x) \cdot (y - p_t(x))$ that might arise in our computations.
This analysis requires us to be very careful about the complexity of the functions that may be returned by each of the \dowal (i.e., $q_1,\hdots,q_T$) and the \pcalalg (i.e., $p_1,\hdots,p_T$).

\subsection{Discussion of Results and Related Works}
\label{sec:related}

To conclude the introduction, we discuss the significance of our results and highlight some interesting questions that we leave open.
We also contextualize our results within the prior work on omnipredictors.

\paragraph{The Complexity of Omnipredictors.}
In the distributional context, we learn omnipredictors that can be implemented with circuits of $\poly(1/\eps)$-size using oracle gates for $h \in \H$, qualitatively matching the circuit complexity of prior omnipredictor constructions \citep{gopalan2021omnipredictors,gopalan2022loss}.
That said, there are key differences in the omnipredictors produced by our algorithm that achieves near-optimal statistical complexity and prior omnipredictors.
In particular, our omnipredictors are actually \emph{randomized}:  our omnipredictors mix over a collection of efficient predictors.
While the use of randomness can be removed in certain contexts, in generality, randomness is a critical aspect of the online-to-offline conversion.
Naturally, this discrepancy suggests questions about the complexity of statistically-optimal omniprediction.
\emph{Can our omnipredictors be derandomized?} Or, \emph{Is randomness necessary to achieve sample-optimal omnipredictors?}

The sample complexity of our oracle-efficient omniprediction algorithm for bounded-variation losses exhibits an additional dependence on the class of threshold functions over the convex hull of $\Ccal = \Th \circ \Hcal$, in contrast to the dependence of our sample-optimal omniprediction algorithm for bounded-variation losses. Although $\Ccal$ may have finite and even small dimension, the class $\Th(\mathrm{cvx}(\Ccal))$ can have infinite dimension---for instance, when $\Ccal$ consists of subintervals of $[0,1]$. Nevertheless, our analysis depends only on the Rademacher complexity $\mathrm{rad}(\Th(\mathrm{cvx}_k(\Ccal)))$. Since $\dim(\Th(\mathrm{cvx}_k(\Ccal)))$ only grows as  $O(k \cdot \dim(\Ccal))$ in the worse case and our algorithm employs $k = \tilde{O}(1/\epsilon^2)$, the resulting sample complexity scales as $\tilde{O}(\dim(\Ccal)/\epsilon^4)$. Reducing this dependence to a nearly quadratic rate in $\epsilon$ remains an important open problem. Our results show that offline omniprediction with $\tilde{O}(\dim(\Ccal)/\epsilon^2)$ samples is statistically achievable; the challenge lies in attaining comparable rates through an oracle-efficient procedure.

\emph{Does oracle-efficient omniprediction for bounded-variation losses inevitably incur an additional sample-complexity overhead?}

Beyond sample and randomness complexity, there are a number of other measures on which we can compare constructions of omnipredictors.
For instance, the boosting-based learning framework used in all prior omnipredictor constructions \citep{hkrr} produces \emph{deep} omnipredictors, with $\poly(1/\eps)$ layers of computation.
In contrast, our randomized omnipredictors are essentially as shallow as possible: we output a distribution supported on $\poly(1/\eps)$ predictors, each of which is implemented by a postprocessing of hypotheses from $\H$.

As the community understands the problem of omniprediction better, a complexity theory is emerging, where we can ask precise questions about what is and isn't possible along many axes of complexity measures.
Precise accounting of complexity measures may be particularly interesting for applications of omnipredictors and multicalibrated predictors within complexity theory and pseudorandomness \citep{casacuberta2024complexity}.
Our work makes significant strides in settling the complexity of omniprediction along important measures, but also raises a number of interesting questions for future research.

\paragraph{Towards Optimal Decision Making.}
This work adds to a growing collection of works investigating when strong omniprediction-style guarantees can be achieved with regret comparable to a single loss.
The majority of the results in this area work in the online, non-contextual setting.
A trend in these works is identifying alternatives to full sequential calibration, which inherently suffers $\Omega(T^{1/2+\eps})$ regret, and until this year, was only known to be achievable in $O(T^{2/3})$ regret \citep{foster1998asymptotic,dagan2024breaking}.

First in this line of work, \citet{kleinberg2023u} introduced the notion of U-Calibration, or simultaneous loss minimization with respect to all proper loss functions.
Leveraging the V-shaped basis for proper losses (also studied by \citet{li2022optimization}) and a 
representation of the Hedge algorithm \citep{freund1997decision} for 
agents with V-shaped losses as a composition of randomized prediction with postprocessing,
they give a $O(\sqrt{T})$-regret online learner to achieve U-Calibration.
Subsequent work established optimal U-Calibration regret bounds for the multi-class setting \citep{luo2024optimal}.

\citet{hu2024calibrationerrordecisionmaking} study how to achieve swap regret bounds simultaneously for all proper losses (i.e., a strengthened version of U-Calibration to allow for a different choice of optimal action in hindsight per prediction interval).
They introduce the notion of \emph{calibrated decision loss} (CDL) and show how to achieve $O(\sqrt{T} \log T)$ regret for CDL, guaranteeing the same bounds for swap regret.
In the Appendix of \cite{okoroafor2025omni}, we show that CDL is actually incomparable to proper calibration.
In this sense, CDL measures calibration error differently than proper calibration: CDL cannot be used to achieve omniprediction via Decision OI, and proper calibration does not give swap regret guarantees.

Most pertinent to our study, \citet{garg2024oracle} introduced the study of omniprediction in the online setting.
They give oracle-efficient algorithms for a number of related problems, including the stronger problem of online multicalibration, which in turn gives the stronger guarantee of swap omniprediction for convex losses \citep{gopalan2024swap}.
The strength of their guarantees comes at a cost, obtaining $O(T^{3/4})$ swap regret.
In addition to this oracle-efficient result, they also include a result that shows, in restricted settings of loss and hypothesis class, $\tilde{O}(\sqrt{T \log \card{\H}})$ omniprediction regret is information-theoretically possible.

Following \citet{garg2024oracle} and concurrent with the development of our work, \citet{dwork2024fairnessinfinityoutcomeindistinguishableomniprediction} continued the study of online omniprediction.
While their study is quite broad, motivated by fairness concerns in professional networks, \citet{dwork2024fairnessinfinityoutcomeindistinguishableomniprediction} develop online kernel outcome indistinguishability as a core algorithmic primitive.
They give an algorithm for online omniprediction, with respect to differentiable strongly convex losses (a subclass of $\lcvx$) and hypotheses from a Reproducing Kernel Hilbert Space (RKHS), that attains $O(\sqrt{T})$ regret in its dependence on $T$.
While they obtain an optimal dependence on $T$, their dependence on other paramaters can be inefficient. For finite hypothesis classes for example, the regret scales as $O(\sqrt{T |\Hcal|})$, whereas our online algorithm would obtain $O(\sqrt{T \log |\Hcal|})$ for an analogous omniprediction problem.

Finally, in closely related concurrent work, \citet{hu2024omnipredictingsingleindexmodelsmultiindex} study the problem of omniprediction for Single Index Models, which generalize the goal of omniprediction for GLM losses.
They achieve improved sample complexity of $\eps^{-4}$, as a function of the approximation parameter.
While their motivations and results are similar to ours, their techniques are very different, involving a new analysis of the Isotron algorithm \citep{kalai2009isotron,kakade2011efficient}.

\paragraph{Omniprediction and Multi-Group Fairness.}
Since its introduction by \cite{gopalan2021omnipredictors}, omniprediction has been studied in numerous works.
\citet{gopalan2022loss} introduced Loss Outcome Indistinguishability as a method for obtaining omnipredictors, working within the Outcome Indistinguishability (OI) paradigm developed by \citet{oi,dwork2022beyond}.
Within this work, the notion of Decision OI (which we establish is equivalent to Proper Calibration) draws directly from the work of \cite{zhao2021calibrating}, who introduced the related notion of Decision Calibration for the multi-class prediction setting.
Omniprediction has also been studied in the context of constrained predictors \citep{hu2023omnipredictors,globus2023multicalibrated} and in other prediction settings, including regression \citep{gopalan2024regression} and performative prediction \citep{kim2023making}.

The study of omniprediction is tightly connected to the study of multi-group fairness \citep{hkrr,kim2019multiaccuracy}.
While omniprediction has always been known to follow from multicalibration, recent work showed a connection in the reverse direction.
Specifically, \citet{gopalan2024swap} study a stronger notion of omniprediction called \emph{swap omniprediction}, where the omnipredictor $p$ must compete against a benchmark where a different hypothesis $h \in \H$ may be chosen per loss function $\ell \in \L$ and level set of $p$.
In other words, the constant prediction $p(x) = v$ must out-compete every $h \in \H$ with respect to every $\ell \in \L$ on the set of $x$ such that $p(x) = v$.
\citet{gopalan2024swap} demonstrate that Swap Omniprediction for convex losses actually characterizes Multicalibration.
In fact, they show both notions are equivalent to Swap Squared Error Minimization, concurrently studied by \cite{globus2023multicalibration}.

\paragraph{Online Learning, Approachability and Unbiased Predictions.} 
Our algorithms follow a general recipe for solving multi-objective online learning problems via Blackwell’s approachability theorem \cite{blackwell}, as outlined in \cite{pmlr-v19-abernethy11b}. Once the reduction from online omniprediction to approachability is established, we employ exponential weights combined with an explicit halfspace oracle to achieve our online omniprediction objectives. The use of exponential weights is standard in approachability-based methods that aim to approach the negative orthant, as seen in prior work \citep{perchet2015exponential, lee2022onlineminimaxmultiobjectiveoptimization}. A central technical challenge in our setting is the derivation of an explicit halfspace oracle, which plays a crucial role in our complexity analysis.

\cite{noarov2023highdimensionalpredictionsequentialdecision} present a general algorithm for online prediction problems where the goal is to produce a sequence of predictions that remain unbiased with respect to a finite set of conditioning events. Their approach, like ours, is based on Blackwell’s approachability theorem. In particular, once results on proper calibration and threshold-weighted calibration are established, their bounds can be used to derive regret upper bounds for finite hypothesis and loss classes. However, their algorithm does not provide the representation and oracle complexity guarantees that we achieve. Our work requires explicitly solving the halfspace oracle and carefully controlling the complexity of the sequence of predictors generated. Furthermore, their results do not extend to the infinite hypothesis and loss class settings for which we establish guarantees.

\subsection{Acknowledgements}
We would like to enthusiastically thank Karthik Sridharan for many helpful discussions throughout the development of this work about online and statistical learning theory.
In particular, Karthik provided invaluable resources and ideas in the development of our online-to-offline conversions.
and comments on online learning and statistical complexities.

\section{Model and Preliminaries}
\label{sec:prelim}

\paragraph{Omniprediction for Binary Outcomes}
We are concerned with the binary prediction setting where there is a distribution $\Dcal$ over pairs of feature vectors $x \in \Xcal$ and binary outcomes $y \in \Ycal = \{0,1\}$. With only sample access to this distribution, the goal is to learn a predictor $p: \Xcal \rightarrow [0,1]$ that outperforms a class of hypothesis functions $\Hcal = \{ h: \Xcal \rightarrow [0,1] \}$ (e.g decision trees, neural networks) over a range of loss functions $\Lcal = \{ \ell: [0,1] \times \{0,1\} \rightarrow [-1,1] \}$. Formally, we define an omnipredictor as follows:

\begin{definition}[Omnipredictor \cite{gopalan2021omnipredictors}]
Let $\Hcal$ be a family of functions on $\Xcal$ and let $\Lcal$ be a family of loss functions. The predictor $p : \Xcal \rightarrow [0,1]$ is an $(\Lcal, \Hcal, \eps)$-omnipredictor if for every $\ell \in \Lcal$ there exists a function $k_\ell : [0,1] \rightarrow [0,1]$ so that 
\[
\E_{(x,y) \sim \Dcal} \left[ \ell ( k_\ell (p(x)), y) \right] \leq \min_{h \in \Hcal} \E_{(x,y) \sim \Dcal} \left[ \ell ( k_\ell (h(x)), y) \right] + \eps
\]
\end{definition}
One can think of an omnipredictor as a model of the world (or distribution) that is sufficient for the learner to perform at least as well as the best hypothesis in $\Hcal$ with respect to the loss functions in $\Lcal$. The post-processing function $k_\ell$ is defined as follows:
\[
k_\ell (p) = \argmin_{q \in [0,1]} \E_{y \sim \mathrm{Ber}(p)} [\ell (q, y)]
\]

\paragraph{Online Omniprediction} 
In the online version of the omniprediction, we will consider a sequential setting where each round $t \in [T]$, a context $x_t \in \Xcal$ arrives. On observing $x_t$, a forecaster makes a prediction $p_t$ of $\E[y_t | x_t]$ and then observes $y_t$ which may be adversarily chosen. Equivalently, each timestep $t \in [T]$, the forecaster chooses a prediction function $p_t: \Xcal \rightarrow [0,1]$, and the adversary, unaware of the forecaster's choice, chooses a pair $(x_t, y_t) \in \Xcal \times \Ycal$.
We assume the adversary is fully aware of history $H_t$ of predictions made by the forecaster up until timestep $t-1$.
We measure the performance of the forecaster over a range of loss functions in a class $\Lcal$ and a benchmark hypothesis class $\Hcal$. For a sequence of $T$ predictions $\mathbf{p}$ and sequence of context, outcome pairs $\mathbf{x}, \mathbf{y}$, define the forecaster's regret as follows:
\begin{align}
    (\Lcal, \Hcal)\omnireg (\pxy) 
    &=  \E \left[ \max_{\{h \in \Hcal, \ell \in \Lcal\}} \sum_{t=1}^T \ell (k_\ell (p_t(x_t)), y_t) - \ell (h (x_t), y_t) \right]
\end{align}
where the expectation is over the randomness of the forecaster.

When the class of loss functions is the class of all bounded proper scoring losses, this objective becomes a contextual version of the U-calibration objective in \cite{kleinberg2023u}.

\begin{definition}[Multiaccuracy Error]\label{def:maerr}
Let $\Ccal = \{ c: \Xcal \rightarrow [-1,1] \}$ be a family of hypothesis functions. For a sequence of $T$ predictions $\pb$ and context, outcome pairs $\xb, \yb$, define 
\[
    \Ccal\maerr (\pxy) = \max_{c \in \Ccal} \left| \sum_{t=1}^T c(x_t) (y_t - p_t(x_t) ) \right|
\]
\end{definition}

We present the notion of weighted calibration error.
\begin{definition}[Weighted Calibration Error]\label{def:weighted-calerr}
Let $\Wcal = \{ w: [0,1] \rightarrow [-1,1] \}$ be a family of weight functions. For a sequence of $T$ predictions $\pb$ and context, outcome pairs $\xb, \yb$, define 
\[
    \Wcal\calerr (\pxy) = \sup_{w \in \Wcal} \left| \sum_{t=1}^T w(p_t(x_t)) (y_t - p_t(x_t)) \right|
\]
Note that when $\Wcal$ is the set of all functions $w: [0,1] \rightarrow [-1,1]$, this becomes the $\ell_1$-calibration error. 
\end{definition}
The key distinction between multiaccuracy error and weighted calibration error is that the functions in calibration error can also depend on the predictions of the predictor.

\paragraph{Loss Families, Hypothesis Classes and ERM Oracles}
The loss functions we consider are functions $\ell: [0,1] \times \{ 0,1\} \rightarrow \reals$ that take a binary outcome and a prediction and assigns a real value $\ell (p,y)$. 
Let $\llip$ denote the set of losses $\ell$ that are $1$-Lipschitz in $p$ and $\lcvx$ denote the set of losses convex in $p$.

\begin{definition}[Proper Losses]\label{def:proper_loss}
A loss function $\ell$ is said to be proper if 
\[
\E_{y \sim \mathrm{Ber}(p^*)} [\ell (p^*, y)] \leq \E_{y \sim \mathrm{Ber}(p^*)} [\ell (q, y)]
\]
for all $q \in [0,1]$.
Let $\lprop$ denote the set of all such losses. Note that this is also the set of all losses for which $p \in k_\ell (p)$.
\end{definition}

\begin{definition}[Bounded Variation]
  A function $f : [0,1] \to \mathbb{R}$ has \emph{bounded variation} if the quantity 
  $ V(f) = \sup \left\{ \left. \sum_{i=1}^n |f(x_i) - f(x_{i-1})| \; \right| 0 = x_0 < x_1 < \cdots < x_n = 1 \right\}$
  is finite. The class $\lbv$ of bounded variation losses consists of all loss functions $\ell(p,y)$ taking values in $[-1,1]$
  that satisfy $V(\Delta \ell) \leq 2$.
\end{definition}

\begin{definition}[Discrete Derivative]\label{def:discrete_deri}
\[
\Delta \ell (t) = \ell (t, 1) - \ell (t, 0)
\]
Observe that $\ell (t,y) = y \Delta \ell (t) + \ell(t, 0)$. For a class of loss functions $\Lcal$, we will refer to $\Delta \Lcal = \{\Delta\ell : \ell \in \Lcal \}$. We will characterize families of loss functions based on the complexity of the discrete derivate class.
\end{definition}

\begin{definition}[ERM Oracle]
Our algorithms will require a empirical risk minimization oracle for a specified class of hypothesis functions $\Hcal$. \sloppy This oracle will take as a input a sequence of samples $(x_1, y_1) \ldots, (x_T, y_T)$ and returns a hypothesis $h^* \in \Hcal$ such that 
\[
h^* \in \argmin_{h \in \Hcal} \sum_{t=1}^T h(x_t) y_t
\]
\end{definition}

\paragraph{Online Learning and Blackwell's Approachability Theorem}

In online learning, the learner's objective is to select a sequence of actions from a given set such that the average regret asymptotically approaches zero, irrespective of the loss functions chosen by an adversary.

Blackwell's approachability theorem \citep{blackwell} generalizes repeated two-player zero-sum games to setting where payoffs are vector-valued. In this framework, at each time step $t$, Player 1 selects an action $a_t \in A$, Player 2 responds with $b_t \in B$, and Player 1 receives a vector-valued payoff $u(a_t, b_t) \in \RR^d$. The action sets $A$ and $B$ are compact convex subsets of finite-dimensional spaces, and the payoff function $u$ is biaffine over $A \times B$. Player 1's goal is to ensure that the average payoff vector converges to a closed convex target set $S \subseteq \RR^d$. Formally, given $S$, Player 1 chooses actions so that, regardless of Player 2’s choices, the distance between the average payoff and $S$ approaches zero as $T \rightarrow \infty$. 
\begin{equation} \label{eq:approach}
  \dist \left(\frac{1}{T} \sum_{t=1}^T u(a_t, b_t), S \right)  \rightarrow 0 \quad \text{as} \quad T \rightarrow \infty
\end{equation}
Player 1's actions can adapt based on previous outcomes. A set $S$ is approachable if Player 1 can guarantee this convergence. Blackwell's theorem characterizes this by stating that $S$ is approachable if and only if every closed halfspace containing $S$ is approachable.

Many online learning problems can be effectively solved by a reduction to Blackwell's approachability theorem. In fact, \cite{pmlr-v19-abernethy11b} show that approachability and no-regret learning are equivalent. In this paper, we reduce the online omniprediction problem to an approachability problem and apply the regret minimization reduction techniques from \cite{pmlr-v19-abernethy11b} to solve the resulting approachability problem.

\paragraph{Combinatorial Dimensions for Learning}
Given a feature space $\Xcal$ and a fixed distribution $D|_\Xcal$, let $S = \{x_1,\ldots,x_m\}$ be a set of examples drawn
i.i.d. from $D|_\Xcal$. Furthermore, let $\mathcal{F}$ be a class of functions $f : \Xcal \to [-1,1]]$.

\begin{definition}[$\alpha$-shattering] We say $\Fcal$ $\alpha$-shatters the set $S$ if there exists values $v_1, \ldots, v_m \in [0,1]$ such that for all $A \subseteq S$, there exists $f_A \in \Fcal$ such that 
\begin{align*}
\forall x_i \in A, \quad f_A (x_i) \geq v_i + \alpha \\
\forall x_i \in S - A, \quad f_A (x_i) \leq v_i - \alpha 
\end{align*}
The fat shattering dimension of $\Fcal$ at scale $\alpha$ is the size of the largest $\alpha$-shattered set. For binary valued class $\Fcal \subseteq \{0,1\}^\Xcal$, the VC dimension of $\Fcal$ is the size of the largest $1/2$-shattered set. 
\end{definition}

\begin{definition} The \textit{empirical Rademacher complexity} of $\mathcal{F}$ is defined to be

\[\hat{\mathsf{rad}}_m(\mathcal{F}; S) = \mathbb{E}_\sigma \left[\sup_{f\in\mathcal{F}} \left(\frac{1}{m}\sum_{i=1}^m \sigma_i f(x_i)\right)\right]\]

where $\sigma_1,\ldots,\sigma_m$ are independent random variables uniformly chosen from $\{-1,1\}$. We will refer
to such random variables as Rademacher variables.
\end{definition}

\begin{definition} The \textit{statistical Rademacher complexity} of $\mathcal{F}$ is defined as
\[\mathsf{rad}_m(\mathcal{F}) = \mathbb{E}_{S \sim \Dcal}[\hat{\mathsf{rad}}_m(\mathcal{F}; S)]\]
\end{definition}
The following is a well know result in learning theory:

\begin{lemma}
Fix distribution $D \vert_\Xcal$ and parameter $\delta \in (0, 1)$. If $\mathcal{F} \subseteq \{ f : \Xcal \to [-1, 1] \}$ and $S = \{ x_1, \ldots, x_m \}$ is drawn i.i.d. from $D \vert_\Xcal$, then with probability $\geq 1 - \delta$ over the draw of $S$, for every function $f \in \mathcal{F}$,
\[
\mathbb{E}_D[f(x)] \leq \mathbb{E}_S [f(x)] + 2\mathsf{rad}_m(\mathcal{F}) + \sqrt{\frac{\ln(1/\delta)}{m}}. \tag{1}
\]
In addition, with probability $\geq 1 - \delta$, for every function $f \in \mathcal{F}$,
\[
\mathbb{E}_D[f(x)] \leq \mathbb{E}_S[f(x)] + 2\hat{\mathsf{rad}}_m(\mathcal{F}) + 3\sqrt{\frac{\ln(2/\delta)}{m}}. \tag{2}
\]
\end{lemma}

\begin{lemma}[\cite{karthikLec}]\label{lem:fat-rademacher} Fat-shattering dimension and Rademacher complexities are related as follows:
\[
\tilde\Omega \left( \inf_{\alpha > 0} \left\{ 4 \alpha + \frac{12}{\sqrt{m}} \int_\alpha^1 \sqrt{K \, \text{fat}_\delta(\mathcal{F}) \log \frac{2}{\delta}} \, d\delta \right\} \right)
\leq 
\mathsf{rad}_m(\mathcal{F})
\leq \inf_{\alpha > 0} \left\{ 4 \alpha + \frac{12}{\sqrt{m}} \int_\alpha^1 \sqrt{K \, \text{fat}_\delta(\mathcal{F}) \log \frac{2}{\delta}} \, d\delta \right\}
\]
where $\tilde\Omega$ hides log factors in $m$ and $K$ is a universal constant.
\end{lemma}

These parameters have been generalized to the online learning setting using binary trees. A $\Xcal$-valued tree $\mathbf{x}$ of depth $n$ is a rooted complete binary tree with nodes labeled by elements of $\Xcal$. We identify the tree $\mathbf{x}$ with the sequence $(\mathbf{x}_1,\ldots,\mathbf{x}_m)$ of labeling functions $\mathbf{x}_i : \{\pm1\}^{i-1} \mapsto \Xcal$ which provide the labels for each node. Here, $\mathbf{x}_1 \in \Xcal$ is the label for the \textit{root} of the tree, while $\mathbf{x}_i$ for $i > 1$ is the label of the node obtained by following the path of length $i-1$ from the root, with $+1$ indicating 'right' and $-1$ indicating 'left'. A \textit{path} of length $m$ is given by the sequence $\sigma = (\sigma_1,\ldots,\sigma_m) \in \{\pm1\}^m$. For brevity, we shall often write $\mathbf{z}_t(\sigma)$, but it is understood that $\mathbf{z}_t$ only depends only on the prefix $(\sigma_1,\ldots,\sigma_{t-1})$ of $\sigma$.

\begin{definition}
The \textit{sequential Rademacher complexity} of $\mathcal{F}$ on a $\Xcal$-valued tree $\xb$ is defined to be

\[\hat{\mathsf{srad}}_m(\mathcal{F}; \mathbf{x}) = \mathbb{E}_\sigma \left[\sup_{f\in\mathcal{F}} \left(\frac{1}{m}\sum_{i=1}^m \sigma_i f(\mathbf{x}_i (\sigma))\right)\right]\]
\end{definition} 

\begin{definition}
The \textit{sequential Rademacher complexity} of $\mathcal{F}$ is defined as
\[\mathsf{srad}_m(\mathcal{F}) = \sup_{\xb}[\hat{\mathsf{srad}}_m(\mathcal{F}; \xb)]\]
\end{definition} 

\paragraph{Concentration.}
The Azuma-Hoeffding Inequality is an essential component of showing concentration in online algorithms and online-to-offline conversions.
\begin{lemma}[Azuma-Hoeffding's Inequality] 
If $X_1,\ldots,X_T$ is a martingale difference sequence, and for every $t$,
with probability 1, $|X_t| \leq M$. Then with probability $1-\delta$,
\[
\left|\sum_{t=1}^T X_t\right| \leq M\sqrt{2T\ln\frac{2}{\delta}}.
\]
\end{lemma}
 \section{Proper Calibration}
\label{sec:proper-cal}
In this section, we introduce the notion of proper calibration (\Cref{def:proper_cal}).
We show that proper calibration is actually a restatement of the notion of Decision OI, used by \citet{gopalan2022loss} to achieve omniprediction.
Because proper calibration is defined in terms of proper loss functions, we leverage a characterization of proper losses by the V-shaped losses \citep{li2022optimization,kleinberg2023u} to give an efficient strategy for achieving proper calibration.

\paragraph{Defining Proper Calibration.}
We define proper calibration as an instance of weighted calibration.
In particular, we use a class of weight functions $\wprop$ associated with proper scoring rules, where each weight function $w \in \wprop$ corresponds to the discrete derivative (\Cref{def:discrete_deri}) of a proper loss function.
Concretely, $\wprop$ is defined as the following set of weight functions.
\begin{gather*}
    \wprop = \{ \Delta \ell : \ell \in \lprop \}
\end{gather*}
Proper calibration ensures that the predictions, weighted by the discrete derivatives of any proper loss, do not correlate with the residual prediction error.
\begin{definition}[Proper Calibration] \label{def:proper_cal}
\sloppy 
For a sequence of $T$ predictions $\pb$ and context, outcome pairs $\xb, \yb$, the proper calibration error is given by $ \wprop \calerr (\pxy)$ i.e
\[
    \pcal(\pxy) = \sup_{w \in \Wcal_\mathrm{proper}} \left| \sum_{t=1}^T w(p_t(x_t)) (y_t - p_t(x_t)) \right|
\]
When $\xb, \yb$ are drawn from a fixed distribution $\Dcal$, we shall refer to the expected proper calibration error over the distribution by $\pcal_\Dcal (\pb)$ defined as follows:
\[
    \pcal_\Dcal (\pb) = \sup_{w \in \Wcal_\mathrm{proper}} \left| \E_{\Dcal} \left[ w(p(x)) (y - p(x)) \right] \right|
\]
\end{definition}
In our understanding, proper calibration is a novel notion of calibration not known to be implied by other notions, other than full $\ell_1$-calibration.
We provide a more complete comparison of proper calibration to prior notions of calibration in \Cref{app:comparisons}.
Next, we describe why proper calibration is natural and useful in the context of omniprediction.

\subsection{Omniprediction via Proper Calibration and Multiaccuracy}

To motivate proper calibration, we recall the notion of Decision OI introduced by \citet{gopalan2022loss}.
We restate their definition and extend it to include the sequential prediction setting.

\begin{definition}[Decision OI]
    Fix a collection of loss functions $\L$.
    For a predictor $p:\X \to [0,1]$, the  Decision OI error over a distribution $\D$ is given by $\L\textrm{-}\mathrm{DecOIErr}_\D(p)$.
\[ \L\textrm{-}\mathrm{DecOIErr}_\D(p) = \sup_{\ell \in \L}\card{\E_{x,y \sim \D}[\ell(k_\ell \circ p(x),y)] - \E_{\substack{x \sim \D\\\yt \sim \Ber(p(x))}}[\ell(k_\ell \circ p(x),\yt)]}
\]
    For a sequence of $T$ predictions $\pb$ and context, outcome pairs $\xb, \yb$, the  Decision OI error is given by $ \L\textrm{-}\mathrm{DecOIErr}(\pxy)$.
\[
    \L\textrm{-}\mathrm{DecOIErr}(\pxy) = \sup_{\ell \in \L} \left| \sum_{t=1}^T (\ell(k_\ell \circ p_t(x_t),y_t) - \E_{\yt \sim \Ber(p_t(x_t))}\ell(k_\ell \circ p_t(x_t),\yt)) \right|
\]
\end{definition}
As in all notions of outcome indistinguishability, Decision OI compares the expected value of some test on outcomes sampled from the ``real'' world (i.e., $y_t$) versus the value on outcomes sampled from a ``modeled'' world, where outcomes are sampled according to our predictions $\yt \sim \Ber(p_t(x_t))$.

\citet{gopalan2022loss} showed that Decision OI plus a certain multiaccuracy condition suffice for omniprediction; they use full $\ell_1$-calibration to achieve Decision OI.
We show that, for any loss class $\L$, proper calibration suffices to imply Decision OI.
In fact, Proper Calibration is equivalent to Decision OI for the class of all loss function $\lall = \{\ell : [0,1] \to [-1,1]\}$.
\begin{theorem}
\label{thm:pcal-decoi}
For any predictor $p:\X \to [0,1]$ and distribution $\D$,
\begin{gather*}
    \lall\textrm{-}\mathrm{DecOIErr}_\D(p) = \pcal_\Dcal(p).
\end{gather*}
For every sequence of $T$ predictions $\pb$ and context, outcome pairs $\xb, \yb$
\begin{gather*}
    \lall\textrm{-}\mathrm{DecOIErr}(\pxy) = \pcal(\pxy).    
\end{gather*}
\end{theorem}
\begin{proof}
In the distributional setting, we rewrite the definition of Decision OI.
We write the difference of losses in terms of the difference of their discrete derivatives,
\begin{align*}
    \lall\textrm{-}\mathrm{DecOIErr}_\D(p)
    &= \sup_{\ell \in \lall}\big|\E_{x,y \sim \D}[\ell(k_\ell \circ p(x),y)] - \E_{\substack{x \sim \D\\\yt \sim \Ber(p(x))}}[\ell(k_\ell \circ p(x),\yt)]\big| \\
    &=\sup_{\ell \in \lall}\big|\E_{x,y \sim \D}[y \cdot \Delta\ell(k_\ell \circ p(x)) - p(x) \cdot \Delta \ell(k_\ell \circ p(x))]\big|
\end{align*}
where the equality follows by the fact that, in the modeled world, $\E[\yt \vert x] = p(x)$ and the definition of $\Delta \ell$.
Then, we can combine terms and consider the supremum.
\begin{align*}
    &\sup_{\ell \in \lall}\big|\E_{x,y \sim \D}[\Delta \ell(k_\ell \circ p(x)) \cdot (y-p(x))]\big| \\
    &= \sup_{\ell \in \lprop}\big|\E_{x,y \sim \D}[\Delta \ell(p(x)) \cdot (y-p(x))]\big|
\end{align*}
The final equality, here, follows from the fact that for any loss, $\ell$ composed with the optimal post-processing $k_\ell$, can be viewed as a proper loss.
Formally, for any $\ell : [0,1] \to [-1,1]$, there exists a proper loss $\ell' \in \lprop$ (namely, $\ell'(p,y) = \ell(k_{\ell}(p), y)$) such that $\Delta \ell (k_\ell( p )) = \Delta\ell'(p)$, for all $p \in [0,1]$.

Note that the same argument works in the sequential setting, establishing the analogous equality.
\end{proof}

Using this view on Proper Calibration, we can immediately apply the Loss OI framework to achieve omnipredictors, with proper calibration replacing Decision OI / calibration.
\begin{lemma}\label{lem:pcal-to-omni} 
For a sequence of $T$ predictions $\pb$ and context, outcome pairs $\xb, \yb$,
$
(\Lcal, \Hcal)\omnireg (\pxy) \leq \pcal (\pxy) + (\Delta\Lcal \circ \Hcal)\maerr (\pxy)
$
\end{lemma}
The proof of lemma mirrors that of Proposition 4.5 in \cite{gopalan2022loss} in that it explicitly follows the Loss OI framework introduced in the paper.
For completeness and self-containment, we give a proof of the statement in the sequential setting.
\begin{proof}[Proof of \Cref{lem:pcal-to-omni}]
Recall that
\[
(\Lcal, \Hcal)\omnireg (\pxy) 
= \max_{\{h \in \Hcal, \ell \in \Lcal\}} \sum_{t=1}^T \ell (k_\ell(p_t(x_t)), y_t) - \ell (h (x_t), y_t)
\]
For a fixed timestep $t \in [T]$, we expand the inner expression as follows:
\begin{align}
\ell &(k_\ell(p_t(x_t)), y_t) - \ell (h (x_t), y_t) \\
&\leq \ell (k_\ell(p_t(x_t)), y_t) + \E_{\yt \sim \mathrm{Ber}(p_t(x_t))}[\ell (h(x_t), \yt)] - \E_{\yt \sim \mathrm{Ber}(p_t(x_t))}[\ell (k_\ell(p_t(x_t)), \yt)] - \ell (h (x_t), y_t) \\
&= \left[ \ell (k_\ell(p_t(x_t)), y_t) - \E_{\yt \sim \mathrm{Ber}(p_t(x_t))}[\ell (k_\ell(p_t(x_t)), \yt)] \right] + \left[ \E_{\yt \sim \mathrm{Ber}(p_t(x_t))}[\ell (h(x_t), \yt)] - \ell (h (x_t), y_t) \right] 
\end{align}
where the second line follows from the definition of $k_\ell$. Now we simplify both terms separately. For the first term, we have 
\begin{align}
\ell &(k_\ell(p_t(x_t)), y_t) - \E_{\yt \sim \mathrm{Ber}(p_t(x_t))}[\ell (k_\ell(p_t(x_t)), \yt)] \\
&= [y_t \Delta \ell (k_\ell(p_t(x_t))) + \ell (k_\ell(p_t(x_t)), 0)] - [p_t \Delta \ell (k_\ell(p_t(x_t))) + \ell (k_\ell(p_t(x_t))), 0)]  \\
&= (y_t - p_t(x_t)) \Delta \ell (k_\ell(p_t(x_t)))
\end{align}
where line (6) uses the fact that $\ell (p,y) = y (\ell(p,1) - \ell(p,0)) + \ell (p,0)$ for $y \in \{0,1\}$
Similarly, for the second term, we have 
\begin{align}
\E_{\yt \sim \mathrm{Ber}(p_t(x_t))} &[\ell (h(x_t), \yt)] - \ell (h(x_t), y_t) \\
&= [p_t (x_t) \Delta \ell (h(x_t)) + \ell (h(x_t), 0)] - [y_t \Delta \ell (h(x_t)) + \ell (h(x_t), 0)]  \\
&= (p_t (x_t) - y_t) \Delta \ell (h(x_t))
\end{align}
Plugging these simplifications back into the regret term, we obtain
\begin{align}
(\Lcal, \Hcal)\omnireg (\pxy) 
&\leq \max_{\{h \in \Hcal, \ell \in \Lcal\}} \sum_{t=1}^T (y_t - p_t (x_t)) \Delta \ell (k_\ell(p_t(x_t))) + (p_t (x_t) - y_t) \Delta \ell (h(x_t)) \\
&\leq \max_{\{\ell \in \Lcal\}} \left| \sum_{t=1}^T (y_t - p_t (x_t)) \Delta \ell (k_\ell(p_t (x_t))) \right| + \max_{\{h \in \Hcal, \ell \in \Lcal\}} \left| \sum_{t=1}^T (p_t (x_t) - y_t) \Delta \ell (h(x_t)) \right| \\
\label{eqline:lprop}
&\leq \max_{\{\ell \in \lprop\}} \left| \sum_{t=1}^T (y_t - p_t(x_t)) \Delta \ell (p_t(x_t)) \right| + \max_{\{h \in \Hcal, \ell \in \Lcal\}} \left| \sum_{t=1}^T (p_t (x_t) - y_t) \Delta \ell (h(x_t)) \right| \\ 
&\leq \pcal(\pxy) + (\Delta\Lcal \circ \Hcal)\maerr (\pxy)
\end{align}
where line~\eqref{eqline:lprop} follows from the fact that $\Delta\ell$ composed with its $k_\ell$ function corresponds to some $\Delta \ell_\mathrm{proper}$ for some proper loss $\ell_\mathrm{proper} \in \lprop$.
\end{proof}

Naturally, proper calibration can be incorporated into the original algorithms for learning omnipredictors via Loss OI.
In \Cref{app:mcboost}, we describe an adaptation of the boosting-style algorithm of \citet{gopalan2022loss} that outputs an omnipredictor using a multiplicative factor of $1/\eps^6$ fewer samples than the original algorithm based on $\ell_1$-calibration.

\subsection{Approximating Proper Calibration with Weighted Calibration over Thresholds}
To achieve calibration for all proper weight functions, we will take advantage of the basis decomposition of proper scoring losses in \citep{li2022optimization, kleinberg2023u}. To this end, we introduce a notion of calibration, weighted by threshold functions that utilizes the basis functions as weight functions to approximate proper calibration effectively. 
Concretely, we consider the following characterization of proper losses.
\begin{lemma}[V-shaped proper losses \citep{li2022optimization, kleinberg2023u}]
For $v \in [0,1]$, define the proper loss $\ell_v (p, y) = (y - v) \text{sgn} (v - p)$.
Then every $\ell \in \lprop$ can be expressed as a convex combination of these $v$-shaped proper losses. That is, for every $\ell \in \lprop$, there exists nonnegative coefficients $c_v(\ell)$ such that $\int_0^1 c_v (\ell) dv \leq 2$ and
\[
\ell (p,y) = \int_{0}^1 c_v(\ell) \ell_v (p, y) dv
\]
\end{lemma}
Immediately, this characterization also gives a characterization of the discrete derivatives of proper losses.
\begin{corollary}\label{cor:v-disc}
For every $\ell \in \lprop$, there exists nonnegative coefficients $c_v(\ell)$ such that $\int_0^1 c_v (\ell) dv \leq 2$ and
\[
\Delta\ell (p) = \int_{0}^1 c_v(\ell) \Delta\ell_v (p) dv
\]
where $\Delta \ell_v (p) = \text{sgn} (v - p)$ is a \{-1,1\} threshold function at \( v \).
\end{corollary}

In the \Cref{lem:tcal-decom}, we show that proper calibration and calibration weighted by threshold functions are within a constant factor of each other.
Specifically, we let $\wthres$ denote the set of weight functions associated with $V$-shaped proper scoring rules, where each weight function $w$ corresponds to the discrete derivative of a $V$-shaped proper loss function i.e $\wthres = \{ \text{sgn} (v - p) : v \in [0,1] \}$.
Then, we obtain the following tight approximation of the proper calibration error.

\begin{lemma}\label{lem:tcal-decom}
For a sequence of $T$ predictions $\pb$ and context, outcome pairs $\xb, \yb$,
\[
\wthres\calerr (\pxy) \leq \pcal (\pxy) \leq 2 \wthres\calerr (\pxy)
\]
\end{lemma}
\begin{proof}
We wish to show that
\[
\wthres \calerr (\pxy) \leq \Wcal_{\mathrm{proper}}\calerr (\pxy) \leq 2 \wthres \calerr (\pxy)
\]
By \Cref{cor:v-disc}, we know that for every $\ell \in \lprop$, there exists nonnegative coefficients $c_v(\ell)$ such that $\int_0^1 c_v (\ell) dv \leq 2$
\[
\Delta\ell (p) = \int_{0}^1 c_v(\ell) \Delta\ell_v (p) dv
\]
\begin{align}
\max_{w \in \Wcal_\mathrm{proper}} \left| \sum_{t=1}^T w(p_t(x_t)) (y_t - p_t(x_t)) \right| 
&= \max_{\ell \in \lprop} \left| \sum_{t=1}^T \Delta\ell (p_t(x_t))(y_t - p_t(x_t)) \right| \\
&= \max_{\ell \in \lprop} \left| \sum_{t=1}^T \left( \int_{0}^1 c_v(\ell) \Delta\ell_v (p_t(x_t)) dv \right) (y_t - p_t(x)) \right| \tag{\Cref{cor:v-disc}} \\
&= \max_{\ell \in \lprop} \left| \int_{0}^1 c_v(\ell) \left( \sum_{t=1}^T \Delta\ell_v (p_t(x_t)) (y_t - p_t(x)) \right) dv  \right| \\
&\leq \max_{\ell \in \lprop} \left| \int_{0}^1 c_v(\ell) \sup_{v \in [0,1]} \left( \sum_{t=1}^T \Delta\ell_v (p_t(x_t)) (y_t - p_t(x)) \right) dv  \right| \\
&\leq \max_{\ell \in \lprop} \left| \int_{0}^1 c_v(\ell) dv  \right| \sup_{v \in [0,1]} \left| \sum_{t=1}^T \Delta\ell_v (p_t(x_t)) (y_t - p_t(x)) \right| \\
&\leq 2 \sup_{v \in [0,1]} \left| \sum_{t=1}^T \Delta\ell_v (p_t(x_t)) (y_t - p_t(x)) \right| 
\end{align}
Thus,
\[
\Wcal_{\mathrm{proper}}\calerr (\pxy) \leq 2 \wthres \calerr (\pxy)
\]
\sloppy The first inequality $\wthres \calerr (\pxy) \leq \Wcal_{\mathrm{proper}}\calerr (\pxy)$ follows from the fact that $\wthres \subset \Wcal_{\mathrm{proper}}$
\end{proof}

\Cref{lem:tcal-decom} bounds the proper calibration in terms of a notion of weighted calibration with an uncountably infinite collection of weight functions $\wthres$.
Towards a practically-realizable algorithm for proper calibration, we show that for sufficiently discretized predictions, weighted calibration with respect to a discretized collection of thresholds suffices to bound the proper calibration error.
Let $\wthres^\gamma$ denote a $\gamma$-discretization of $\wthres$. That is, $\wthres^\gamma = \{ \text{sgn} (v - p) : v \in \{0, \gamma, 2\gamma, \ldots, 1\} \}$.

\begin{lemma}[$\gamma$-discretized thresholds]\label{lem:eps-thres}
For a sequence of $T$ predictions $\pb$ with values in $\{0, \gamma, 2\gamma, \ldots, 1\}$ and context, outcome pairs $\xb, \yb$,
\[
\wthres \calerr (\pxy) = \wthres^\gamma \calerr (\pxy)
\]
\end{lemma}
\begin{proof}[Proof of \Cref{lem:eps-thres}]
Recall that 
\[
\wthres \calerr (\pxy) = \max_{v \in [0,1]} \left| \sum_{t=1}^T \text{sgn}(v - p_t(x_t)) \cdot (y_t - p_t(x_t)) \right|.
\]
Since the predictions \(p_t(x_t)\) take values only in \(\{0, \gamma, \ldots, 1\}\), the sign function \(\text{sgn}(v - p_t(x_t))\) can only change at points in \(\{0, \gamma, \ldots, 1\}\). Hence, for any \(v \in [0,1]\), there exists some \(v' \in \{0, \gamma, \ldots, 1\}\) such that:
\[
\text{sgn}(v - p_t(x_t)) = \text{sgn}(v' - p_t(x_t)).
\]
This implies
\[
\sum_{t=1}^T \text{sgn}(v - p_t(x_t)) \cdot (y_t - p_t(x_t)) = \sum_{t=1}^T \text{sgn}(v' - p_t(x_t)) \cdot (y_t - p_t(x_t)).
\]
Therefore
\[
\max_{v \in [0,1]} \left| \sum_{t=1}^T \text{sgn}(v - p_t(x_t)) \cdot (y_t - p_t(x_t)) \right| = \max_{v' \in \{0, \gamma, \ldots, 1\}} \left| \sum_{t=1}^T \text{sgn}(v' - p_t(x_t)) \cdot (y_t - p_t(x_t)) \right|.
\]
Thus
\[
\wthres \calerr (\pxy) = \wthres^\gamma \calerr (\pxy).
\]
\end{proof}

\subsection{Algorithm for Proper Calibration in the Online Setting}
In this subsection, we present an algorithm that guarantees proper calibration in the online setting at the rate of $O(\sqrt{T \log T})$.
\Cref{alg:tcal} achieves this by ensuring proper calibration by auditing with threshold weight functions per \Cref{lem:tcal-decom}.
Moreover, the rate of our algorithm is nearly tight; no algorithm can guarantee expected threshold calibration error at a rate of $\Omega (\sqrt{T})$.
This lower bound is unsurprising, but follows by \Cref{par:prop_to_ucal}, where we show that U-calibration lower bounds a multiplicative factor of proper calibration.

\paragraph{Overview of \Cref{alg:tcal}:}
The algorithm is based on Blackwell's Approachability Theorem. We define a two player game where the adversary player selects $z_t = (x_t, y_t) \in \Xcal \times \{0,1\}$ and the learner selects $p_t: \Xcal \rightarrow [1/T]$. Both players are allowed to play randomized strategies but since the learner observes $x_t$, we can simplify things and only consider $y_t \in \{0,1\}$ and $p_t = p_t (x_t)$. We design the payoff vector of this game to reflect our objective of threshold calibration. That is, define
\[
u_{v, s} (p_t, y_t) = s(y_t - p_t) \mathrm{sgn}(v - p_t) \quad \text{for} \ v \in [1/T], s \in \{ +, -\}
\] 
Observe that after T rounds of interaction, 
$\wthres \calerr (\pxy) = \max_{v,s} \sum_{t \in [T]} u_{v,s} (p_t, z_t)$. Therefore, we design the learner's target set to be the set of all vectors $u$ with coordinates less than $1/T$.

We use exponential weights update method in \Cref{line:tcal_exp} to generate sequence of halfspaces $w^t$ with coordinates for every $v \in [1/T], s \in \{ +, -\}$. Given a halfspace $w^t$, the algorithm computes the function $f( q)$ defined in \Cref{line:f_eq} and the algorithm chooses the distribution to sample $p_t$ using the strategy described in \Cref{line:tcal-str-start} to \Cref{line:tcal-str-end}.

In all, we establish the following regret bound.
\begin{theorem}[Proper Calibration Upper Bound]\label{thm:thres-alg}
\Cref{alg:tcal} guarantees expected proper calibration error of $O \left(\sqrt{T \ln T} \right)$
\end{theorem}
Formally, we leverage the following lemmas that give guarantees on the Blackwell optimal response (in \Cref{line:tcal-str-start} to \Cref{line:tcal-str-end}) and the dual player's strategy (in \Cref{line:tcal_exp}). This dual player strategy bears similarity in structure to the algorithms used for classical calibration \citep{foster1998asymptotic,onlinemultivalid, okoroafor2024faster}.

\begin{algorithm}[h!]
\caption{Proper Calibration}
\label{alg:tcal}
\textbf{Input:} Sequence of samples $\{ y_1, \ldots, y_T \}$ \\
\textbf{Output:} Sequence of (randomized) predictors $p_1, \ldots, p_T$
\begin{algorithmic}[1]
\For{each $t \in [T]$}
    \State Let $w_{v, s}^t := \frac{\exp \left( \eta \sum_{i=1}^{t-1} u_{v, s} (p_i, z_i)\right)}{ 
    \sum_{v', s'} \exp \left( \eta \sum_{i=1}^{t-1} u_{v', s'} (p_i, z_i)\right)}$ for all $v \in [1/T], s \in \{+,-\}$ \label{line:tcal_exp}
\State Compute
\[
f(q) = \sum_{v, s} w^t_{v,s} \cdot s \cdot \mathrm{sgn}(v - q)
\] \label{line:tcal_f_eq}
    \If{$f(0) \leq 0$} \label{line:tcal-str-start}
    \State Predict $p_t = 0$
    \ElsIf{$f(1) > 0$}
    \State Predict $p_t = 1$
    \Else
    \State Find adjacent probabilities $q_i, q_{i+1}$ such that $f(q_i) \cdot f(q_{i+1}) \leq 0$ 
\State Predict $p_t = q_i$ with prob $\frac{|f(q_{i+1})|}{|f(q_i)| + |f(q_{i+1})|}$ and $p_t = q_{i+1}$ with prob $\frac{|f( q_i)|}{|f(q_i)| + |f(q_{i+1})|}$ \label{line:tcal-str-end}
    \EndIf 
    \State Observe $x_t$, predict $p_t$, and then observe $y_t$
\EndFor
\end{algorithmic}
\end{algorithm}

\begin{lemma}[Halfspace Approachability]\label{lem:tcal_halfspace}
Given a halfspace $w$, the strategy described in \Cref{line:tcal-str-start} to \Cref{line:tcal-str-end} outputs a distribution $p_t$ over $[1/T]$ such that $\E_{p_t} [\langle w, u (p_t, z_t) \rangle] \leq 1/T$ for any choice of $z_t$
\end{lemma}
\begin{proof}
We consider the cases in the strategy separately:

\textbf{Case 1:} If $f(0) \leq 0$, predict $p_t = 0$.
Then for $v \in [1/T], s \in \{+, -\}$,
\[
w_{v,s} u_{v,s} (p_t, z_t) = s(y_t - p_t) \mathrm{sgn}(v - p_t)w_{v,s} = s y_t w_{v,s}
\]
Summing over values of $v,s$ and applying the definition of $f$ in \Cref{line:tcal_f_eq}, we get 
\[
\langle w, u (p_t, z_t) \rangle = y_t f(0) \leq 0 
\quad \text{for any choice of} \ y \in \{0, 1\}
\]

\textbf{Case 2:} If $f(1) > 0$, predict $p_t = 1$.
Then for $v \in [1/T], s \in \{+, -\}$,
\[
w_{v,s} u_{v,s} (p_t, z_t) = s(y_t - p_t) \mathrm{sgn}(v - p_t)w_{v,s} = s y_t w_{v,s}
\]
Summing over values of $v,s$ and applying the definition of $f$ in \Cref{line:tcal_f_eq}, we get 
\[
\langle w, u (p_t, z_t) \rangle = (y_t - 1) f(1) \leq 0 \quad \text{for any choice of} \ y \in \{0, 1\}
\]

\textbf{Case 3:} Find adjacent probabilities $q_i, q_{i+1}$ such that $f(q_i) \cdot f(q_{i+1}) \leq 0$. \\
We wish to bound $\E_{p_t} [\langle w, u (p_t, y) \rangle]$ for both possible outcomes $y \in \{ 0,1\}$.
Let's fix $v,s$ and compute $\E_{\pb} [w_{v,s} u_{v,s}]$
\[
\frac{|f(q_{i+1})|}{|f(q_i)| + |f(q_{i+1})|} \left[ s(y - q_i)\mathrm{sgn}(v - q_i)w_{v,s} \right] +
\frac{|f(q_i)|}{|f(q_i)| + |f(q_{i+1})|} \left[ s(y - q_{i+1})\mathrm{sgn}(v - q_{i+1})w_{v,s} \right]
\]
Summing over all pairs of $(v,s)$ and applying the definition of $f$ in \Cref{line:tcal_f_eq}, we obtain
\begin{align}
&\frac{|f(q_{i+1})|}{|f(q_i)| + |f(q_{i+1})|} \left[ (y - q_i)f(q_i) \right] +
\frac{|f(q_i)|}{|f(q_i)| + |f(q_{i+1})|} \left[ (y - q_{i+1})f(q_{i+1}) \right] \\
&\leq \frac{|f(q_i)||f(q_{i+1})|}{|f(q_i)| + |f(q_{i+1})|} (q_{i+1} - q_i) \\
&\leq 1/T \tag{since $f(q) \in [-1,1]$}
\end{align}
\end{proof}

\begin{lemma}[Exponential Weight Updates \cite{arora2012multiplicative}]\label{lem:tcal_exp_weight}
The exponential weight updates in \Cref{line:tcal_exp} provide a sequence of vectors $w^t$ such that 
\[
\max_{w: ||w||_1 = 1} \left\langle w, \sum_{t \in [T]} u (p_t, z_t) \right\rangle \leq \sum_{t=1}^T \langle w^t, u (p_t, z_t) \rangle + O \left( \sqrt{T \ln T} \right)
\]
\end{lemma}

With these lemmas in place, we can prove \Cref{thm:thres-alg}.

\begin{proof}[Proof of \Cref{thm:thres-alg}]
\sloppy We wish to bound the expected threshold calibration error, that is, $\E_\pb \left[\max_{v,s} \sum_{t \in [T]} u_{v,s} (p_t, z_t) \right]$ where the expectation is over the randomness in the sampling of $p_t$. Note that this is the same as 
$\E_{\pb} \left[ \max_{w: ||w||_1 = 1} \left\langle w, \sum_{t \in [T]} u (p_t, z_t) \right\rangle \right]$.
\begin{align*}
\E_{\pb} \left[ \max_{w: ||w||_1 = 1} \left\langle w, \sum_{t \in [T]} u (p_t, z_t) \right\rangle \right] & \leq 
\E_{\pb} \left[
\sum_{t=1}^T \langle w^t, u (p_t, z_t) \rangle \right] + O \left( \sqrt{T \ln T} \right) \tag{by \Cref{lem:tcal_exp_weight}} \\
& = 
\sum_{t=1}^T \E_{\mathbf{p}} \langle w^t, u (p_t, z_t) \rangle + O \left( \sqrt{T \ln T} \right)
\tag{by linearity of expectation} \\
& = 
\sum_{t=1}^T \E_{w^t}  \left[ \left. \E_{p_t} \langle w^t, u (p_t, z_t) \rangle \right| w^t \right] + O \left( \sqrt{T \ln T} \right) \tag{by law of iterated expectations} \\
& \leq \left( \sum_{t=1}^T \tfrac{1}{T} \right)  + O \left( \sqrt{T \ln T} \right) \tag{by \Cref{lem:tcal_halfspace}} \\
& \leq O \left( \sqrt{T \ln T} \right)
\end{align*}
\end{proof}
In all, proper calibration is an intermediate notion:  strong enough to imply omniprediction via the Loss OI framework, but weak enough that it can be achieved asymptotically efficiently compared to $\ell_1$-calibration.

\subsection{\pcalalglong}
\label{response:apcal}
To aid with our goal of achieving proper calibration and multiaccuracy simultaneously, we introduce the following augmented proper calibration problem. 
We are given a sequence of samples $(x_1, y_1), \ldots, (x_T, y_T)$ and hypothesis functions 
$q_1, \ldots, q_T \, : \, \X \to [-1,1]$ and the goal is to output a sequence of predictors $p_1,\ldots,p_T \, : \, 
\X \to [0,1]$ that is proper calibrated, i.e.~$\pcal (\pxy)$ is sublinear in $T$, while also satisfying the sequential multiaccuracy guarantee $\sum_{t=1}^T q_t (x_t) (y_t - p_t(x_t)) \in o(T)$ with respect to the sequence of input hypothesis functions $q_1, \ldots, q_T \, : \, \X \to [-1,1]$. To solve this problem we present the algorithm, \pcalalgboth, and state its performance guarantees in both the sequential and distributional settings.

\begin{proposition}[Augmented Proper Calibrator]\label{lem:proper-recal}
Given an online sequence of samples $(x_1, y_1), \ldots, (x_T, y_T)$ and hypothesis functions $q_1, \ldots, q_T$,
\Cref{alg:proper_recal}
outputs a sequence of predictors $p_1, \ldots, p_T$ such that with probability at least $1 - \delta$,
\begin{enumerate}
    \item $\pcal (\pxy) \leq O \left( \sqrt{T \ln T/\delta} \right)$
    \item $\sum_{t=1}^T q_t (x_t) (y_t - p_t(x_t)) \leq O \left( \sqrt{T \ln T/\delta} \right)$.
\end{enumerate}
If the samples $\{(x_t,y_t)\}_{t=1}^T$ are drawn i.i.d.\ from a distribution $\D$
(while the functions $q_1,\ldots,q_T$ may still be specified online by an adaptive adversary)
then the uniform distribution $\hat \pb = \mathrm{unif}(p_1,\ldots,p_T)$ satisfies the 
following bounds with probability at least $1 - \delta$.
\begin{enumerate}
    \item $\pcal_\D(\hat \pb) \le O \left(\sqrt{\frac{\ln(T/\delta)}{T}} \right)$
    \item $\tfrac 1T \sum_{t = 1}^T \E_{(x,y) \sim \D}[q_t(x)(y-p_t(x))] \le O \left(\sqrt{\frac{\ln(T/\delta)}{T}} \right)$.
\end{enumerate}The algorithm runs in $\mathrm{poly}(T)$ timesteps per iteration and each $p_t$ is simply a monotonic remapping of $q_t$
to $[1/T]$. In other words, there exists a monotone non-decreasing function $\monoremap_t : [-1,1] \to [1/T]$, randomly sampled by the algorithm, such that $p_t (x) = \monoremap_t (q_t(x))$ for all $x$.
\end{proposition}

\begin{algorithm}[p]
\caption{\pcalalglong}
\label{alg:proper_recal}
\begin{algorithmic}[1]
\Procedure{MinZeroCrossing}{$g,q$}
\Comment{Find lexicographically minimal $(j,\lambda)$ such that $g(q,j,\lambda) \leq 0$}
    \If{$g(q,0,0) \leq 0$}
        \State \textbf{return} (0,0)
    \Else
        \For{$j=0, \frac1T, \frac2T, \ldots, \frac{T-1}{T}$}
            \If{$g(q,j,0) \cdot g(q,j,1) \leq 0$}
            \Comment{\vspace{0.5ex} Solve $(1-\lambda) g(q,j,0) + \lambda g(q,j,1) = 0$ for $\lambda$}
                \State \vspace{0.5ex} Let $ \lambda = \frac{|g(q,j,0)|}{|g(q,j,0)| + |g(q,j,1)|} $
                \State \textbf{return} $(j,\lambda)$
            \EndIf
        \EndFor 
        \State \textbf{return} (1,1)
    \EndIf
\EndProcedure \\
\Procedure{APCAL}{$T$}
    \State \textbf{Input:} Sequence of samples $(x_1, y_1), \ldots, (x_T, y_T)$
    \State \textbf{Input:} Sequence of hypothesis functions $q_1, \ldots, q_T: \Xcal \rightarrow [-1,1]$
    \State \textbf{Output:} Sequence of predictors $p_1, \ldots, p_T: \Xcal \rightarrow [0,1]$
\State Initialize weights $w_{v,s}^1 = \tfrac{1}{2T+3}$ for $v \in [1/T], s \in \{+, -\}$ \Comment{for proper calibration constraints}
    \State Initialize weight: $w^{1}_\mathrm{sma} = \tfrac{1}{2T+3}$ \Comment{for sequential multiaccuracy constraint}
    \State Set $\eta = \sqrt{\log(T) / T}$ \Comment{learning rate for exponential weight updates}
    \For{$t = 1$ to $T$}
        \State Sample $\zeta \in [0,1]$ uniformly at random. 
        \State Define functions 
$f_t, \, g_t, \, \psi_t, \monoremap_t$ as follows.
        \label{line:apcal-mainloop-start}
        \begin{align} 
            \label{line:apcal-ft-definition}
            f_t(q,j) & = w^t_{\mathrm{sma}} \cdot q + \sum_{v,s} w_{v,s}^t \cdot \Th_v(j) \\
            \label{line:apcal-gt-definition}
            g_t(q,j,\lambda) & = (1-\lambda) f_t(q,j) \, + \, \lambda f_t \left( q, \min\{j + \tfrac1T, 1\} \right) \\
            \label{line:apcal-psit-definition}
            \psi_t(q) & = \mbox{\sc MinZeroCrossing}(g_t, q) \\
            \label{line:apcal-monoremap-definition}
            \monoremap_t(q) & = \begin{cases}
            j & \mbox{if } \psi_t(q) = (j,\lambda) \mbox{ where } \lambda > \zeta \\
            j+ \tfrac1T & \mbox{if } \psi_t(q) = (j,\lambda) \mbox{ where } \lambda \le \zeta
         \end{cases}
        \end{align}
\[
        \]
        \State \textbf{Receive} $q_t$
        \State \textbf{Return} predictor $p_t = \monoremap_t \circ q_t$  \label{line:apcal-pt-definition}
        \State \textbf{Observe} $(x_t, y_t)$
        \State Perform exponential weight updates with losses $s \cdot (y_{t} - p_{t}(x_t)) \cdot \Th_v (p_{t}(x_t))$ for each $v,s$ and 
        \hspace*{1cm} $(y_{t} - p_{t}(x_t)) \cdot q_t(x_t)$ for the sequential multiaccuracy constraint:
            \begin{align*}
                W^{t} &= e^{- \eta \cdot (y_{t} - p_{t}(x_t)) \cdot q_t(x_t)} w^t_{\mathrm{sma}} + \sum_{v,s} e^{- \eta \cdot s \cdot (y_{t} - p_{t}(x_t)) \cdot \Th_v (p_{t}(x_t))} w_{v,s}^t \\        w^{t+1}_{\mathrm{sma}} & =  e^{- \eta \cdot (y_{t} - p_{t}(x_t)) \cdot q_t(x_t)} w^t_{\mathrm{sma}} \, \big/ \, W^t \\
                w^{t+1}_{v,s} & =  e^{- \eta \cdot s \cdot (y_{t} - p_{t}(x_t)) \cdot \Th_v (p_{t}(x_t))} w_{v,s}^t \, \big/ \, W^t 
            \end{align*}   
    \EndFor
\EndProcedure
\end{algorithmic}
\end{algorithm}

\paragraph{Overview of \Cref{alg:proper_recal}:}
The algorithm follows the design ideas of \Cref{alg:tcal}. We consider a two player game that combines the objectives of proper calibration and sequential multiaccuracy. The adversary selects $z_t = (x_t, y_t) \in \Xcal \times \{0,1\}$ and the learner selects $p_t: \Xcal \rightarrow [1/T]$. We construct a payoff vector where the first $2T+2$ components are for the proper calibration objective
\[
u_{v, s} (p_t, z_t) = s(y_t - p_t (x_t)) \Th_v (p_t (x_t)) \quad \text{for} \ v \in [1/T], s \in \{ +, -\}
\] 
and the last component captures the sequential multiaccuracy constraint from the sequence of hypothesis functions $q_t$
\[
u_{\mathrm{ma}} (p_t, z_t) = (y_t - p_t (x_t)) q_t(x_t) 
\]
The learner's target set is the set of vectors $u$ whose coordinates are bounded by 0 i.e the negative orthant. We will maintain weights for these $2T+3$ dimensions of the payoff vector using exponential weights. 

Concretely, in \Cref{alg:proper_recal}, we use exponential weights update to generate sequence of halfspaces $w^t \in \R^{2T + 3}$ with coordinates for every pair $v \in [1/T], s \in \{+,-\}$ and the sequential multiaccuracy constraint. Then for each timestep $t$, the algorithm constructs the predictor based on the strategy described from \Cref{line:apcal-mainloop-start} to \Cref{line:apcal-pt-definition}.

The following lemmas are used in the proof of the proposition. 
\begin{lemma}[Halfspace Approachability]\label{lem:precal-halfspace}
For every timestep $t \in [T]$, the predictor $p_t (x)$ constructed from input $q_t$ in line~\ref{line:apcal-pt-definition} satisfies the following guarantee for all $x \in \Xcal, y \in \{0,1\}$:
\[
\E_{p_t} \left[ w^t_\mathrm{sma} \cdot q_t(x)(y - p_t (x)) + \sum_{v,s} w_{v,s}^t \cdot s \cdot \Th_v(p_t(x)) \cdot (y - p_t (x)) \right] \leq \frac{1}{T}
\]
\end{lemma}
\begin{proof}[Proof of \Cref{lem:precal-halfspace}]
Recall the definition of $f_t(q,j)$ in \Cref{line:apcal-ft-definition}, and recall that $p_t(x) = \monoremap_t(q_t(x))$.
We can rewrite the objective above as showing that for all $x,y$, and $q = q_t(x)$,
\[
\E_{j \sim \monoremap_t(q)} \left[ f_t(q, j) \cdot (y - j) \right] \leq \tfrac{1}{T}
\]
We now consider three cases:
\begin{itemize}
    \item Case 1: If $f_t(q, 0) \leq 0$, then $g_t(q,0,0) \leq 0$ so $\psi_t(q) = (0,0)$ and 
    $\monoremap_t(q) = 0$ as long as $\zeta > 0$ (which happens almost surely).
\begin{align*}
\E_{j \sim \monoremap_t(q)} \left[ f_t(q, j) \cdot (y - j) \right] 
&= f_t(q,0) y \\
&\leq \max \{0, f_t(q, 0) \} \tag{for any value of $y \in \{0,1\}$} \\
&\leq 0 \tag{by starting assumption $f_t(q, 0) \leq 0$}
\end{align*}
\item Case 2: If $f_t(q, j) > 0$ for all $j \in [1/T]$, then $g_t(q,j,\lambda) > 0$ for all 
    $(j,\lambda) \in [1/T] \times [0,1]$, so $\psi_t(q) = (1,1)$ and $\monoremap_t(q) = 1$ for any value of $\zeta$. 
\begin{align*}
\E_{j \sim \monoremap_t(q)} \left[ f_t(q, j) \cdot (y - j) \right] 
&= f_t(q,1) (y-1) \\
&\leq \max \{0, -f_t(q, 1) \} \tag{for any value of $y \in \{0,1\}$} \\
&\leq 0 \tag{by starting assumption $f_t(q, 1) > 0$}
\end{align*}
\item Case 3: There exist adjacent probabilities $i, i + \frac1T$ in $[1/T]$ such that 
\begin{equation} \label{eq:apcal-neg-product}
f_t \left(q, i \right) \cdot f_t \left(q, i + \tfrac1T \right) \leq 0 .
\end{equation}
Then, for the minimum $i \in [1/T]$ satisfying~\eqref{eq:apcal-neg-product}, the values
$g_t(q,i,0)$ and $g_t(q,i,1)$ are differently signed, so there is a unique $\lambda \in [0,1]$
satisfying $g_t(q,i,\lambda) = 0$, and $\psi_t(q) = (i,\lambda)$. 
\begin{align*}
&\E_{j \sim \monoremap_t(q)} \left[ f_t(q, j) \cdot (y - j) \right] \\
&= \Pr(\zeta > \lambda) \cdot f_t(q,i) (y-i) + \Pr(\zeta \leq \lambda) \cdot f_t(q,i+\tfrac{1}{T}) (y-i-\tfrac{1}{T}) \\
&= (1-\lambda) f_t(q,i) (y-i) + \lambda f_t(q,i+\tfrac{1}{T}) (y-i-\tfrac{1}{T}) \\
&= g_t(q,i,\lambda) (y-i) - \frac{\lambda f_t(q,i+\tfrac1T)}{T} \\
&\leq 1/T \tag{since $g_t(q,i,\lambda) = 0, \lambda \in [0,1]$, and $|f_t(q,i+\frac1T)| \leq 1$.} 
\end{align*}
\end{itemize}
\end{proof}

\begin{lemma}[Halfspace Concentration]\label{lem:halfspace-conc}
The following inequality holds with probability at least $1 - \delta/2$ for the sequence of inputs $(x_1,y_1, q_1), \ldots, (x_T, y_T, q_T)$, weights $w^1, \ldots, w^T$ and outputs $p_1, \ldots, p_T$:
\[
\sum_{t=1}^T \left[ w^t_\mathrm{sma} \cdot q_t(x_t)(y_t - p_t (x_t)) + \sum_{v,s} w_{v,s}^t \cdot s \cdot \Th_v(p_t(x_t)) \cdot (y_t - p_t (x_t)) \right] \leq O \left( \sqrt{T \ln 1/\delta} \right)
\]
\end{lemma}
\begin{proof}[Proof of \Cref{lem:halfspace-conc}]
The conclusion of the lemma can be rewritten as showing that 
\[
\sum_{t=1}^T f(x_t, p_t(x_t)) (y_t - p_t(x_t)) \leq O \left( \sqrt{T \ln 1/\delta} \right)
\]
We showed in \Cref{lem:precal-halfspace} that $\E_{j \sim p_t(x)} \left[ f(x, j) \cdot (y - j) \right] \leq \tfrac{1}{T}$. Now we apply Azuma-Hoeffding by defining the martingale difference sequence 
$Z_t = f(x_t, p_t(x_t)) (y_t - p_t(x_t)) - \E_{p_t(x_t)} \left[ f(x_t, p_t(x_t)) \cdot (y_t - p_t(x_t)) \right]$, where the expectation is over the randomness of the sampling of $p_t$. Note that $Z_t$ is bounded by $2$. Applying Azuma-Hoeffding, we have that with probability at least $1 - \delta$,
\[
\sum_{t=1}^T Z_t \leq O \left( \sqrt{T \ln 1/\delta} \right)
\]
Since $\sum_{t=1}^T \E_{p_t(x_t)} \left[ f(x_t, p_t(x_t)) \cdot (y_t - p_t(x_t)) \right] \leq T \cdot 1/T = 1$, we obtain the desired result.
\end{proof}

We conclude this subsection with the proof of \Cref{lem:proper-recal}.

\begin{proof}[Proof of \Cref{lem:proper-recal}] First, we upper bound the expected error, in the sequential setting, of the predictors the algorithm outputs. 
Let $u(p_t, z_t)$ denote the $2T+3$ dimensional payoff vector where the first $2T+2$ coordinates correspond to the proper calibration constraints and the $2T+3$-th coordinate corresponds to the sequential multiaccuracy constraint. The algorithm maintains a sequence of weight vectors $w^t$ over these $2T+3$ constraints that ensures the following inequalities are satisfied by any weight vector $w \succeq \mathbf{0}$ with $\|w\|_1=1$.
\begin{align*}
& \left\langle w, \sum_{t \in [T]} u (p_t, z_t) \right\rangle   \\
&\leq 
\sum_{t=1}^T \langle w^t, u (p_t, z_t) \rangle + O \left( \sqrt{T \ln T} \right) \tag{by exponential weights guarantee} \\
& = 
\sum_{t=1}^T \left[w^t_{\mathrm{sma}} \cdot q_t (x_t) (y_t - p_t(x_t)) + \sum_{v,s} w_{v,s}^t \cdot s \cdot \Th_v(p_t(x_t)) \cdot (y - p_t (x_t)) \right] + O \left( \sqrt{T \ln T} \right) \\
&\leq O \left( \sqrt{T \ln 1/\delta} \right) + O \left( \sqrt{T \ln T} \right) \tag*{(by \Cref{lem:halfspace-conc})} \\
&\leq O \left( \sqrt{T \ln T/\delta} \right) 
\end{align*}
Instantiating the inequality while setting $w$ equal to the first $2T+2$ standard basis vectors, we obtain an upper bound on the $\wthres^{1/T}$-weighted calibration error of the sequence 
$\pxy$. Observing that each predictor $p_t$ takes values in the set $[1/T] = \{0, \frac1T, \frac2T, \ldots, 1\}$
and that the proper calibration error of a $[1/T]$-valued
predictor is bounded by twice its $\wthres^{1/T}$-weighted calibration error \Cref{lem:eps-thres}, we obtain the stated upper bound on $\pcal(\pxy)$. Instantiating the inequality while setting
$w$ equal to the last standard basis vector, we obtain the upper bound on $\sum_{t=1}^T q_t (x_t) (y_t - p_t(x_t))$.

To bound the expected error of $\hat \pb = \mathrm{unif}(p_1,\ldots,p_T)$ in the distributional setting we again use a martingale argument. 
Define $$Z_t (v,s) =  \E_{(x,y) \sim \Dcal} [s \cdot \Th_v (p_t (x)) (y - p_t (x))] - s \cdot \Th_v(p_t (x_t)) (y_t - p_t (x_t))$$ for $v \in [1/T], s \in \{+,-\}$. Observe that $Z_t (v,s)$ for $t = 1, \ldots, T$ forms a martingale difference sequence with bounded variance of 2. 
Similarly, let $$Z_t^\mathrm{ma} = \E_{(x,y) \sim \Dcal} [q_t (x) (y - p_t (x))] - q_t(x_t) (y_t - p_t (x_t)).$$ 
$Z_t^\mathrm{ma}$ for $t = 1, \ldots, T$ forms a martingale difference sequence with bounded variance of 2. 
Thus, applying Azuma-Hoeffding's inequality together with union bound over values of $v \in [1/T], s \in \{+,-\}$ and $Z_t^\mathrm{ma}$, we get that, with probability at least $1 - \delta$, for all $v \in [1/T], s \in \{+,-\}$
\begin{align*}
\left| \frac{1}{T} \sum_{t=1}^T Z_t (v,s) \right| \leq  \sqrt{\frac{\ln \tfrac{2T}{\delta} }{T}} 
& \qquad \mbox{and} \qquad
\left| \frac{1}{T} \sum_{t=1}^T Z_t^\mathrm{ma} \right| \leq  \sqrt{\frac{\ln \tfrac{2T}{\delta} }{T}} 
\end{align*}
The distributional error bounds now follow by observing that 
\begin{align*}
    \forall v,s \; \; 
    \frac1T \sum_{t=1}^T \E_{(x,y) \sim \Dcal} [s \cdot \Th_v (p_t (x)) (y - p_t (x))] & = 
    \frac1T \sum_{t=1}^T s \cdot \Th_v(p_t (x_t)) (y_t - p_t (x_t)) \, + \, 
    \frac1T \sum_{t=1}^T Z_t (v,s) \\
    \frac1T \sum_{t=1}^T \E_{(x,y) \sim \Dcal} [q_t(x) (y - p_t(x)] & = 
    \frac1T \sum_{t=1}^T q_t(x_t) (y_t - p_t(x_t)) \, + \, 
    \frac1T \sum_{t=1}^T Z_t^{\mathrm{ma}}
\end{align*}
and that we have already shown all terms on the right side of both inequalities are bounded
above by $O \left( \sqrt{\frac{\log T/\delta}{T}} \right).$

The algorithm runs in $\mathrm{poly}(T)$ timesteps per iteration and each $p_t$ is defined in line~\ref{line:apcal-pt-definition} as a composition 
$\monoremap_t \circ q_t$ where $\monoremap_t : [-1,1] \to [1/T]$. To see that 
$\monoremap_t(q)$ is a monotone function of $q$, observe in equation~\eqref{line:apcal-ft-definition} that 
for each $j \in [1/T]$, $f_t(q,j)$ is a monotone increasing function of $q$ because $w^t_{\mathrm{sma}} > 0$.
Hence, for each $(j,\lambda)$, the function $g_t(q,j,\lambda)$ is a monotone increasing function of $q$. 
It follows that $\psi_t(q)$ is a monotone (lexicographically) non-decreasing function of $q$: if 
$\psi_t(q) = (j,\lambda)$ and $q' < q$, then $g_t(q',j,\lambda) < g_t(q,j,\lambda) \leq 0$, so 
$\psi_t(q')$ cannot be lexicographically greater than $(j,\lambda)$. 
For any fixed $\zeta$, the lexicographic monotonicity of $\psi_t$ implies the
monotonicity of $\monoremap_t$. 
\end{proof}

 \section{Online Omniprediction}
\label{sec:online_omni}

In this section, we show that online omniprediction for a class of hypothesis functions $\Hcal$ and a class of loss functions $\Lcal$ can be achieved using an online weak agnostic learner for $\Delta \Lcal \circ \Hcal$ with regret that scales with $\tilde O \left(\sqrt{T} \right)$ and the regret of the online weak agnostic learner. 
Online weak agnostic learning, as studied in \citep{chen2012online, brukhim2020online, beygelzimer2015optimal}, outputs a sequence of predictions whose correlation with the outcome sequence is not much less than that of the best fixed hypothesis in hindsight.

\begin{definition}[Online Weak Agnostic Learner] \label{def:owal}
Let \( \Ccal \subset \{ c : \Xcal \to [-1,1] \} \) be a hypothesis class, and fix a failure probability \( \delta \in (0,1) \). An \emph{online weak agnostic learner} interacting over \( T \) rounds computes a sequence of randomized strategies \( Q_1, \ldots, Q_T \in \Delta(\Xcal \to [-1,1]) \), where each \( Q_t \) is a distribution over predictors \( \Xcal \to [-1,1] \). When called as a subroutine (``oracle'') within another algorithm, the learner outputs random samples from the distribution $Q_t$ corresponding to the current interaction round.

The learner interacts with an adaptive adversary under the \emph{simultaneous-move} (SM) protocol: at each round \( t \), the learner selects \( Q_t \in \Delta(\Xcal \rightarrow [-1,1]) \), and the adversary simultaneously selects a context-label pair \( (x_t, y_t) \in \Xcal \times [-1,1] \), potentially as a function of past interactions but without access to the specific realization \( q_t \sim Q_t \).

With probability at least \( 1 - \delta \) over the learner’s internal randomness, the learner guarantees:
\[
\max_{c \in \Ccal} \sum_{t=1}^T c(x_t) y_t \leq \sum_{t=1}^T \E_{q_t \sim Q_t} [q_t(x_t) y_t] + \owalreg_\Ccal^\delta(T),
\]
where \( \owalreg_\Ccal^\delta(T) \) is the regret bound of the learner with respect to the class \( \Ccal \).
\end{definition}

\begin{remark}
In our omniprediction setting, the learner will be used in a \emph{delayed-label} (DL) setting, where the label $y_t$ at round \( t \) may depend on one or more realizations \( q_t^1,\ldots,q_t^k \sim Q_t \) but the context $x_t$ may only depend on realizations $q_s^j \; (s<t)$ drawn from $Q_1,\ldots,Q_{t-1}$ in previous interaction rounds. In such cases, one may convert the SM-regret guarantee into a high-probability DL guarantee via implicit resampling; see \Cref{lem:resampling}.
\end{remark}

Using an online weak agnostic learner for $\Delta \Lcal \circ \Hcal$, we present the following algorithm for omniprediction.

\begin{algorithm}[h!]
\caption{Online Omniprediction}
\label{alg:online-omni}
\begin{algorithmic}[1]
\Procedure{OnlineOmniprediction}{$T$}
    \State \textbf{Input:} Sequence of samples $(x_1, y_1), \ldots, (x_T, y_T)$
    \State \textbf{Input:} A failure probability parameter, $\delta>0$.
    \State \textbf{Input:} An online weak agnostic learner, \owalalgsymbol, for function class $C = \{-1, +1\} \cdot \Delta \Lcal \circ \Hcal$ with $k$-fold resampling (\Cref{lem:resampling}) for \( k = O(T \log(T/\delta)) \) and failure probability $\delta/2$
    \State \textbf{Input:} An augmented proper calibrator (\Cref{alg:proper_recal}), \pcalalgsymbol, with failure probability $\delta/2$
    \State \textbf{Output:} Sequence of predictors $p_1, \ldots, p_T: \Xcal \rightarrow [0,1]$
    \State \textbf{Receive} $q_1$ from \owalalgsymbol
    \For{$t = 1$ to $T$}
        \State \textbf{Send} $q_t$ and sample $(x_t, y_t)$ to \pcalalgsymbol
        \State \textbf{Receive} $p_t$ from \pcalalgsymbol
        \State \textbf{Send} sample $(x_t, y_t - p_t(x_t))$ to \owalalgsymbol
        \State \textbf{Receive} $q_{t+1}$ from \owalalgsymbol
    \EndFor
\EndProcedure
\end{algorithmic}
\end{algorithm}

\begin{theorem}\label{thm:online-omni}
Given an online sequence of samples $(x_1, y_1), \ldots, (x_T, y_T)$, a class of loss functions $\Lcal$, a class of hypothesis functions $\Hcal$, an online weak agnostic learner for $\Delta \Lcal \circ \Hcal$, and a failure probability $\delta$, \Cref{alg:online-omni} can be initialized to guarantee $(\Lcal, \Hcal)-$online omniprediction 
regret with probability at least $1 - \delta$
\[
O \left( \sqrt{T \ln T/\delta} \right) + \owalreg_{\Ccal}^\delta (T) 
\]
\end{theorem}
\begin{proof}[Proof of \Cref{thm:online-omni}]
Let $\Ccal = \Delta \Lcal \circ \Hcal$.
The online weak agnostic learner with $k$-fold resampling (\Cref{lem:resampling}) for \( k = O(T \log(T/\delta)) \) guarantees with probability at least $1 - \delta$
\[
\max_{c \in \Ccal} \sum_{t=1}^T c(x_t) (y_t - p_t(x_t)) \leq \sum_{t=1}^T q_t(x_t) (y_t - p_t(x_t)) + \owalreg_{\Ccal}^\delta (T) + O \left( \sqrt{T \ln T/\delta} \right)
\]
since \Cref{alg:online-omni} sets the labeling function to $r_t (x,y) = y - p_t (x)$. 
The \pcalalglong algorithm guarantees with probability at least $1 - \delta$, $\pcal (\pxy) \leq O \left( \sqrt{T \ln T/\delta} \right)$ and $\sum_{t=1}^T q_t (x_t) (y_t - p_t(x_t)) \leq O \left( \sqrt{T \ln T/\delta} \right)$ as stated in \Cref{lem:proper-recal}.
The second guarantee is the same as the first term RHS on the online weak agnostic learner guarantee. Thus, we obtain
\[
\max_{c \in \Ccal} \sum_{t=1}^T c(x_t) (y_t - p_t(x_t)) \leq O \left( \sqrt{T \ln T/\delta} \right) + \owalreg_{\Ccal}^\delta (T) 
\]
\sloppy This allows us to conclude proper calibration and $\Ccal$-multiaccuracy, which together implies
$
(\Lcal, \Hcal)\omnireg \leq O \left( \sqrt{T \ln T/\delta} \right) + \owalreg_{\Ccal}^\delta (T) 
$
\end{proof}

\begin{corollary}[of \Cref{thm:online-omni,thm:owal_rates}]
Given a hypothesis class $\Hcal$ and a class of loss functions $\Lcal$ such that the composed class $\Delta \Lcal \circ \Hcal$ has bounded sequential Rademacher complexity, then there exists a forecaster that guarantees expected omniprediction regret bounded by 
\[
O \left( \sqrt{T \ln T} + \mathsf{srad}_T ( \Delta \Lcal \circ \Hcal) \right)
\]
where $\mathsf{srad}_T(\Ccal)$ refers to the sequential Rademacher complexity of a class $\Ccal$.
\end{corollary}

\subsection{Online Weak Agnostic Learning}
In this subsection, we discuss the online weak agnostic learning problem in \Cref{def:owal}, showing the existence of online weak agnostic learners whose regret scales optimally with the sequential rademacher complexity of the hypothesis class.

First, we note that when the labels are boolean valued and the hypothesis class are constrained to be boolean functions, online weak agnostic learning is equivalent to agnostic online binary classification problem where in each iteration the learner suffers loss of $\mathbf{1}[\hat y_t \neq y_t]$ and the goal is to minimize the cummulate loss.

When $\Ccal$ is finite, a multiplicative weights algorithm can be used to achieve online weak agnostic learning with regret bounded by $O \left(\sqrt{T \ln |\Ccal|} \right)$.

\begin{lemma}[Multiplicative Weights]\label{lem:mw_owal}
Given a finite hypothesis class $C$, \Cref{alg:mw_owal} implements an online weak agnostic learner whose regret scales as $O \left(\sqrt{T \ln |\Ccal|} \right)$ with probability 1.
\end{lemma}

\begin{algorithm}[H]
\caption{Multiplicative Weights for Online Weak Agnostic Learning}
\label{alg:mw_owal}
\begin{algorithmic}[1]
\State \textbf{Input:} Finite hypothesis class $\Ccal$, learning rate $\eta = \sqrt{\frac{\ln |\Ccal|}{T}}$
\State Initialize weights $w_1(c) = \frac{1}{|\Ccal|}$ for all $c \in \Ccal$
\For{$t = 1, \ldots, T$}
    \State Receive context $x_t$
    \State Output prediction $q_t(x_t) = \sum_{c \in \Ccal} w_t(c) c(x_t)$
    \State Observe label $y_t \in [-1,1]$
    \State Update weights: $w_{t+1}(c) = \frac{w_t(c) \exp(\eta c(x_t)y_t)}{\sum_{c' \in \Ccal} w_t(c') \exp(\eta c'(x_t)y_t)}$
\EndFor
\end{algorithmic}
\end{algorithm}

For infinite hypothesis classes, the achievable regret values is characterized by the sequential Rademacher complexity of the hypothesis class.

\begin{theorem}\label{thm:owal_rates}
Given a hypothesis class $C$ with finite sequential Rademacher complexity, there exists an algorithm that implements an online weak agnostic learner with regret that scales as $O(T \cdot \mathsf{srad}_T (C))$. Moreover, no algorithm can do better than $\Omega (T \cdot \mathsf{srad}_T (C))$.
\end{theorem}

To prove this theorem, we need to formulate the rates for online weak agnostic learning as the minimax value of a game.

\begin{definition}[Minimax Value of Online Weak Agnostic Learning]

Given a hypothesis class $\Ccal \subset \{ c: \Xcal \rightarrow [-1,1]\}$. Consider randomized learners who predict a distribution $q_t \in \Delta (\Xcal \rightarrow [-1,1])$ and sample $p_t$ from this distribution on every round $t$. We define the value of the game as
\begin{equation}
\mathcal{V}_T^{\text{owal}}(\mathcal{C})= \inf_{p_1 \in \mathcal{Q}} \sup_{(x_1, y_1) \in \mathcal{X} \times \Ycal} \mathbb{E}_{p_1 \sim q_1} \cdots \inf_{q_T \in \mathcal{Q}} \sup_{(x_T, y_T) \in \mathcal{X}, \Ycal} \mathbb{E}_{p_T \sim q_T} \left[\sup_{c \in \mathcal{C}} \sum_{t=1}^T c(x_t) y_t - \sum_{t=1}^T p_t(x_t) y_t\right]
\tag{1}
\end{equation}
\end{definition}

\begin{proof}[Proof of \Cref{thm:owal_rates}]
Since the losses are linear and 1-Lipschitz in the variable $y_t$, the upper bound result follows by applying Theorem 8 of \cite{rakhlin2015online} to bound the sequential Rademacher complexity of the hypothesis class.
Thus, 
\[
\mathcal{V}_T^{\text{owal}}(\mathcal{C}) \leq 2 T \cdot \mathsf{srad}_T(C)
\]
To show the lower bound we simply choose the label sequence $y_1, \ldots, y_T$ to be i.i.d.\ Rademacher random variables, thus, it immediately follows that for any sequence $x_1, \ldots, x_T$
\[
\mathcal{V}_T^{\text{owal}}(\mathcal{C}) \geq 
\E_{y_1, \ldots, y_T} \left[\sup_{c \in \mathcal{C}} \sum_{t=1}^T c(x_t) y_t - \sum_{t=1}^T p_t (x_t) y_t\right]
\]
Since $y_1,\ldots, y_T$ are i.i.d.\ Rademacher random variables, the second term becomes zero in expectation regardless of the choice of $p_t$ sequence, giving us the desired result.
\[
\mathcal{V}_T^{\text{owal}}(\mathcal{C}) \geq T \cdot \mathsf{srad}_T(C)
\]
\end{proof}

\begin{remark}
    \label{rmk:srad-finite-class}
    For a finite class $\mathcal{C}$ the unnormalized sequential Rademacher complexity satisfies 
    $T \cdot \mathsf{srad}_T(\mathcal{C}) \leq O \left( \sqrt{ T \log |\mathcal{C}| } \right)$, 
    hence \Cref{thm:owal_rates} recovers the same regret bound as 
    \Cref{lem:mw_owal}, albeit without supplying an explicit algorithm. 
\end{remark}

Due to the linearity of the loss function and the fact that the online weak agnostic learner is allowed to output values in the $[-1,1]$ interval, there exist \textit{deterministic} online weak agnostic learners for any class $C$ that is online learnable. This is because we can always predict the expectation of the \textit{randomized} online weak agnostic learner and achieve the same guarantee. However, such deterministic online weak agnostic learner might be more computationally difficult to implement, so we will only assume access to a \textit{randomized} online weak agnostic learner.

\paragraph{Reducing from Delayed-Label to Simultaneous-Move.}

In our omniprediction framework, learning is modeled as an interaction between a randomized online weak agnostic learner and the Augmented Proper Calibration algorithm. At each round \( t \), the learner outputs a randomized predictor \( Q_t \in \Delta(\Xcal \rightarrow [-1,1]) \), from which a hypothesis \( q_t \sim Q_t \) is drawn. This hypothesis is then used by the calibration algorithm to compute a prediction \( p_t \), which defines the loss function (via \( y_t \)) presented to the learner at that round. Thus, the label at round \( t \) may depend on one or more realized samples \( q_t \), and hence the protocol corresponds to a \emph{delayed-label} (DL) model.

However, our definition of an online weak agnostic learner assumes the \emph{simultaneous-move} (SM) protocol: at each round \( t \), the learner outputs a randomized distribution \( Q_t \), and the adversary chooses \( (x_t, y_t) \) without knowledge of the specific realization \( q_t \sim Q_t \). This discrepancy between the learner’s guarantee and its downstream usage necessitates a reduction from the DL to the SM model. To achieve this, we invoke a \emph{$k$-fold resampling} technique: at each round \( t \), we independently sample \( k = O(T \log(T/\delta)) \) hypotheses \( q^1_t, \ldots, q^k_t \sim Q_t \) and use the empirical average \( \bar{q}_t := \frac{1}{k} \sum_{i=1}^k q^i_t \) as the predictor used to define the loss.

This transformation enables a high-probability guarantee even under the stronger DL adversary model:

\begin{lemma}[High-Probability Guarantee via $k$-Fold Resampling] \label{lem:resampling}
Let \( Q_1, \ldots, Q_T \in \Delta(\Xcal \rightarrow [-1,1]) \) be the randomized predictors computed by an online weak agnostic learner satisfying the following regret bound with probability $1-\delta/2$ under the simultaneous-move protocol:
\[
\max_{c \in \Ccal} \sum_{t=1}^T c(x_t) y_t \leq \sum_{t=1}^T \E_{q_t \sim Q_t}[q_t(x_t) y_t] + \owalreg_\Ccal^\delta(T).
\]
Suppose that at each round \( t \), we sample \( k = O(T \log(T/\delta)) \) independent hypotheses \( q^1_t, \ldots, q^k_t \sim Q_t \) and define the empirical predictor \( \bar{q}_t := \frac{1}{k} \sum_{i=1}^k q^i_t \). Then, with probability at least \( 1 - \delta \),
\[
\max_{c \in \Ccal} \sum_{t=1}^T c(x_t) y_t \leq \sum_{t=1}^T \bar{q}_t(x_t) y_t + \owalreg_\Ccal^\delta(T) + \sqrt{T \log(T/\delta)}.
\]
\end{lemma}

\begin{proof}[Proof of \Cref{lem:resampling}]
Since $x_t$ does not depend on the realizations $q^1_t, \ldots, q^k_t \sim Q_t$ we have 
$\E[q^i_t(x_t)] = \E_{q_t \sim Q_t}[q_t(x_t)]$ for $i=1,\ldots,k$. By Hoeffding's inequality, 
\[
    \Pr \left( \big| \bar{q}_t(x_t) - \E_{q_t \sim Q_t}[q_t(x_t)] \big| \, > \, \sqrt{\frac{\log(T/\delta)}{T}} \right) 
    \; < \; \frac{\delta}{2T} .
\]
Using a union bound over \( T \) rounds, along with the fact that $-1 \leq y_t \leq 1$, we have that with probability at least \( 1 - \delta/2 \),
\[
\left| \sum_{t=1}^T \left( \bar{q}_t(x_t) y_t - \E_{q_t \sim Q_t}[q_t(x_t) y_t] \right) \right| \leq \sqrt{T \log(T/\delta)}.
\]
Combining this with the high-probability regret bound under the SM protocol completes the proof.
\end{proof}

\paragraph{Online Weak Agnostic Learning for $\{-1,+1\} \times \Ccal$.}
Our algorithms for omniprediction will actually use an online weak agnostic learning oracle for $\{-1,+1\} \times \Ccal$. However, one should think of this as equivalent to $\Ccal$, because of the following lemma.
\begin{lemma}
Given an online weak agnostic learner for \(\mathcal{C}\), one can construct an online weak agnostic learner for \(\{-1,+1\}\times \mathcal{C}\) that uses two calls to the original learner for \(\mathcal{C}\) at each round. The resulting algorithm’s regret satisfies
\[
\owalreg_{\{-1,+1\}\times\Ccal}^\delta(T) 
\;\le\;
\owalreg_{\Ccal}^\delta(T) \;+\; O\bigl(\sqrt{T}\bigr).
\]
\end{lemma}

\begin{proof}
The main idea is as follows: we run two copies of the online weak agnostic learner for \(\mathcal{C}\), one with the labels unchanged (``positive sign'') and one with the labels negated (``negative sign''), and then apply multiplicative weights to combine their predictions.

Concretely, let \(\owal^+\) be the online weak agnostic learner run on the sequence \(\{(x_1,y_1),\dots,(x_{t-1},y_{t-1})\}\), and let \(\owal^-\) be another instance run on \(\{(x_1,-y_1),\dots,(x_{t-1},-y_{t-1})\}\). In each round \(t\):
\begin{enumerate}
\item We query \(\owal^+\) to get a predictor \(q^+_t\) and query \(\owal^-\) to get a predictor \(q^-_t\).  
\item We maintain an exponential-weight ``score'' for each of the two experts, reflecting their cumulative performance so far. Specifically, we use a learning rate \(\eta = T^{-1/2}\), and multiply the weight of the positive-sign expert by \(\exp\bigl(\eta\sum_{u=1}^{t-1} q^+_u (x_u) \,y_u\bigr)\), whereas the weight of the negative-sign expert is multiplied by \(\exp\bigl(-\eta\sum_{u=1}^{t-1} q^-_u (x_u) \,y_u\bigr)\).  
\item We normalize these two weights and form a convex combination:
\[
q_t \;=\; \alpha_t \,q^+_t \;+\; (1-\alpha_t)\,q^-_t,
\]
where
\[
\alpha_t 
\;=\; 
\frac{\exp\bigl(\eta\,\sum_{u=1}^{t-1}q^+_u (x_u)\,y_u\bigr)}
{\exp\bigl(\eta\,\sum_{u=1}^{t-1}q^+_u (x_u)\,y_u\bigr)
\;+\;\exp\bigl(-\eta\,\sum_{u=1}^{t-1}q^-_u (x_u)\,y_u\bigr)}.
\]
\item We output the single deterministic predictior \(q_t\) for round \(t\), and after seeing \(y_t\), update the two experts' weights accordingly.
\end{enumerate}

By the usual analysis of multiplicative weights, the combination's cumulative reward \(\sum_{t=1}^T q_t\,y_t\) is within \(O(\sqrt{T})\) of the better of the two experts (i.e., \(\owal^+\) or \(\owal^-\)). Moreover, each of these two experts has online weak agnostic regret at most \(\owalreg_{\Ccal}^\delta(T)\) relative to the best function in \(\mathcal{C}\) with sign fixed. Thus, we obtain
\[
\sum_{t=1}^T q_t (x_t) \,y_t 
\;\;\ge\;\;
\max_{\substack{s\in\{+,-\}\\c\in \Ccal}}
\sum_{t=1}^T s\,c(x_t)\,y_t
\;-\;
\owalreg_{\Ccal}^\delta(T)
\;-\;
O\bigl(\sqrt{T}\bigr).
\]
Rearranging yields the stated bound on \(\owalreg^\delta_{\{-1,+1\}\times\Ccal}(T)\).
\end{proof}

\subsection{Omniprediction for Finite Classes}

As an easy first application of \Cref{thm:online-omni}, we assume the loss class $\Lcal$ and hypothesis class
$\Hcal$ are both finite, and we obtain online omnipredictors whose regret depends near-optimally on $T$ and only 
logarithmically on $|\Lcal|$ and $|\Hcal|$.

\begin{corollary}[Omniprediction for Finite Loss Classes]\label{thm:finite-omni}
Given a finite class of hypothesis functions $\Hcal$ and a finite class of bounded loss functions $\Lcal$, \Cref{alg:online-omni} guarantees expected omniprediction regret of $O \left(\sqrt{T \ln(|\Hcal| \cdot |\Lcal| \cdot T)} \right)$ using \Cref{alg:mw_owal} as the online weak agnostic learning oracle.
\end{corollary}

The proof of the corollary is immediate by combining \Cref{thm:online-omni} with \Cref{lem:mw_owal} and observing
that the cardinality of the class $\Delta \Lcal \circ \Hcal$ is bounded above by $|\Hcal| \cdot |\Lcal|$.

\subsection{Omniprediction for Infinite Loss Classes with Finite Approximate Basis}

While achieving omniprediction for infinite loss classes might seem challenging, a key observation is that if a class of loss functions $\Fcal$ 
can be approximated by a simpler set of ``basis functions'' $\Gcal$, then multiaccuracy for $\Delta \Gcal \circ \Hcal$ 
implies multiaccuracy for $\Delta \Fcal \circ \Hcal$ with bounded excess regret. In the sequel, we will apply this abstract
approximation result, taking our basis functions to be thresholds or ReLUs (or finite subsets thereof) and using them to approximate 
important infinite classes of losses.

\begin{definition}[Approximate Basis] \label{def:approx-basis}
Let $\Gamma$ be a set and $\Fcal = \{ f: \Gamma \rightarrow [-1,1] \}$ a class of functions on $\Gamma$.
We say that a set $\Gcal = \{ g: \Gamma \rightarrow [-1,1] \}$
is an \emph{$\eps$-basis for $\Fcal$ with sparsity $d$ and coefficient norm $\lambda$}, if for every function $f \in \Fcal$, 
there exists a finite subset $\{g_1,\ldots,g_d\} \subseteq \Gcal$ and 
coefficients $c_1,c_2,\ldots,c_d \in [-1,1]$
satisfying 
\begin{equation} \label{eq:finite-approx-basis}
    \forall x \in \Gamma \; \; \left| f(x) - \sum_{i=1}^d c_i g_i(x) \right| \leq \eps
    \qquad \mbox{and} \qquad
    \sum_{i=1}^d |c_i| \leq \lambda .
\end{equation}
In the special case when $\Gcal$ itself has $d$ elements, we say
$\Gcal$ is a \emph{finite $\eps$-basis for $\Fcal$ of \textit{\textbf{size}} $d$ and coefficient norm $\lambda$}.
\end{definition}

A useful property of approximate bases is that the approximation is preserved under 
post-composition with any class of functions. 
\begin{lemma}[Finite approximate bases are preserved under post-composition]
    \label{lem:postcomposition}
    Let $\Gamma_0, \Gamma_1$ be sets and $\Hcal = \{ h : \Gamma_0 \to \Gamma_1\}$ a 
    class of functions from $\Gamma_0$ to $\Gamma_1$. 
    If $\Fcal = \{ f : \Gamma_1 \rightarrow [-1,1] \}$ is any class of functions on $\Gamma_1$ and
    $\Gcal = \{g : \Gamma_1 \rightarrow [-1,1]\}$ is an $\eps$-basis
    for $\Fcal$ with sparsity $d$ and coefficient norm $\lambda$, then
    $\Gcal \circ \Hcal$ is an $\eps$-basis for $\Fcal \circ \Hcal$
    with sparsity $d$ and coefficient norm $\lambda$.
\end{lemma}
\begin{proof}
    Consider any $f \in \Fcal$.
    If coefficients $c_1,\ldots,c_d$ satisfy property~\eqref{eq:finite-approx-basis} 
    in the definition of $\eps$-basis with coefficient norm $\lambda$, then 
    for all $h \in \Hcal$, we have
    \[
    \forall x \in \Gamma_0 \; \; \left| f(h(x)) - \sum_{i=1}^d c_i g_i(h(x)) \right| \leq \eps
    \qquad \mbox{and} \qquad
    \sum_{i=1}^d |c_i| \leq \lambda .
    \]
    which confirms that $\Gcal \circ \Hcal$ is an $\eps$-basis for
    $\Fcal \circ \Hcal$ with coefficient norm 1.
\end{proof}

Approximate bases allow us to extend multiaccuracy from the basis functions to the entire class $\Fcal$.

\begin{lemma}\label{lem:ma-basis-decom}
Let $\Fcal = \{ f: [0,1] \rightarrow [-1,1]\}$ be a class of functions 
and $\Gcal = \{ g_1,\ldots,g_d : [0,1] \rightarrow [-1,1] \}$ an
$\eps$-basis for $\Fcal$ with sparsity $d$ and coefficient norm $\lambda$. 
Then for any sequence of $T$ predictions $\pb$ and context, outcome pairs $\xb, \yb$ 
\[
\Fcal\maerr (\pxy) \leq \lambda \cdot \Gcal\maerr (\pxy) + \eps T
\]
Similarly, in the distributional setting, for any distribution $\D$ on $\X \times [0,1]$ and any predictor $p : \X \to [0,1]$, 
\[
\Fcal\maerr_\Dcal(p) \leq \lambda \cdot \Gcal \maerr_\Dcal (p) + \eps .
\]
\end{lemma}
\begin{proof}We can write every $f \in \Fcal$ as follows: $f(z) = \sum_{i \in [d]} c_i(f) g_i(z) + \eps (f)$ such that $\sum_{i \in [d]} |c_i(f)| \leq \lambda$, $|\eps (f) | \leq \eps$ and $g_1, \ldots g_d \in \Gcal$. Consequently,
\begin{align}
\max_{f \in \Fcal} \left| \sum_{t=1}^T f(x_t) (y_t - p_t(x_t)) \right| 
&= \max_{f \in \Fcal} \left| \sum_{t=1}^T \left( \sum_{i \in [d]} c_i(f) g_i(x_t) (y_t - p_t(x_t)) \right) + \eps(f)(y_t - p_t(x_t)) \right| \\
&\leq \max_{f \in \Fcal} \left| \sum_{i \in [d]} c_i(f) \left( \sum_{t=1}^T  g_i(x_t) (y_t - p_t(x_t)) \right) \right| + \left| \sum_{t=1}^T \eps(f)(y_t - p_t(x_t)) \right| \\
&\leq \max_{f \in \Fcal} \sum_{i \in [d]} |c_i(f)| \left| \sum_{t=1}^T g_i(x_t) (y_t - p_t(x_t)) \right| + \eps T  \tag{since $|\eps (f) | \leq \eps$}\\
&\leq  \max_{f \in \Fcal} \left( \sum_{i \in [d]} |c_i(f)| \right) \max_{i \in [d]} \left| \sum_{t=1}^T g_i(x_t) (y_t - p_t(x_t)) \right| + \eps T \\
&= \lambda \max_{g \in \Fcal_\mathrm{basis}} \left| \sum_{t=1}^T g(x_t) (y_t - p_t(x_t)) \right| + \eps T 
\end{align}
The bound for multiaccuracy error in the distributional setting follows by an identical calculation,
substituting expectations over $\D$ for sums over $t \in \{1,\ldots,T\}$.
\end{proof}

To illustrate the application of \Cref{lem:ma-basis-decom}, we note the following consequence for $(\Lcal,\Hcal)$-online omniprediction when 
$|\Hcal| < \infty$ and $\Delta \Lcal$ has a finite approximate basis. 
\begin{proposition}[Omniprediction for Infinite Loss Classes and Finite Hypothesis Classes]\label{thm:infinite-omni}
Let $\Hcal$ be a finite class of hypothesis functions. Let $\Lcal$ be a (possibly infinite) class of loss functions whose discrete derivative class $\Delta \Lcal$ admits a finite $\eps$-basis of size $d$ with coefficient norm $\lambda$. There exists a forecaster that guarantees $(\Lcal, \Hcal)\omnireg \leq O\left( \lambda \sqrt{T \ln (|\Hcal| \cdot dT)} + \eps T \right)$.
\end{proposition}

\begin{proof}[Proof of \Cref{thm:infinite-omni}]
Let $\Gcal$ denote the finite $\eps$-basis of $\Delta\Lcal$ of size $d$ and coefficient norm $\lambda$. 
By \Cref{lem:postcomposition}, $\Gcal \circ \Hcal$ is an $\eps$-basis of 
$\Delta \Lcal \circ \Hcal$ with the same sparsity and coefficient norm. 
We run \Cref{alg:online-omni} with hypothesis
class $\Gcal \circ \Hcal$ instead of $\Delta\Lcal \circ \Hcal$. 

By \Cref{thm:finite-omni}, this algorithm guarantees expected proper calibration error and expected $(\Gcal \circ \Hcal)$-multiaccuracy error of $O \left(\sqrt{T \ln(|\Hcal| \cdot d T)} \right)$. 

Now, we use \Cref{lem:ma-basis-decom} to bound the $(\Delta\Lcal \circ \Hcal)$-multiaccuracy error in terms of the $(\Gcal \circ \Hcal)$-multiaccuracy error. 
\[
(\Delta\Lcal \circ \Hcal) \maerr \leq \lambda (\Gcal \circ \Hcal)\maerr (\pxy) + \eps T
\]
Combining these bounds gives the desired bound of $O \left(\lambda \sqrt{T \ln(|\Hcal| \cdot dT)} + \eps T \right)$ on the $(\Delta\Lcal \circ \Hcal)$-multiaccuracy error of the forecaster.
\end{proof}
\begin{proposition}[Omniprediction for Infinite Loss and Hypothesis Classes] \label{thm:infinite-oracle-omni}
    For any loss class $\Lcal$ and hypothesis class $\Hcal$, if we are given an online weak agnostic learning
    oracle for $\Gcal \circ \Hcal$ where $\Gcal$ is an $\eps$-basis of $\Delta\Lcal \circ \Hcal$ with 
    coefficient norm $\lambda$, then the application of \Cref{alg:online-omni} with hypothesis class
    $\Gcal \circ \Hcal$ yields an omnipredictor with 
    $(\Lcal, \Hcal)\omnireg \leq O\left( \sqrt{T \ln T} + \eps T + \lambda \cdot \owalreg_{\Gcal \circ \Hcal}(T) \right)$.
\end{proposition}
\begin{proof}
    The regret bound is a direct application of \Cref{thm:online-omni} combined with \Cref{lem:ma-basis-decom}.
\end{proof}

In \Cref{sec:online-to-batch}, we will need to make use of a generalization of 
\Cref{def:approx-basis} that allows approximating functions in $\Fcal$ using 
\emph{infinite} linear combinations of elements of the function class $\Gcal$.
For future reference, we provide the generalization here.
\begin{definition} \label{def:approx-family}
    Let $\Gamma$ be a set and $\Fcal = \{ f : \Gamma \to [-1,1] \}$ a
    class of functions on $\Gamma$. Let $(\Omega,\Sigma)$ be a measurable space and 
    $\Gcal = \left\{ g_{\omega} : \Gamma \to [-1,1] \, \mid \, \omega \in \Omega \right\}$
    a class of functions on $\Gamma$ indexed by $\Omega$. 
    We say $\Fcal$ is $(\Gcal,\eps)$-spanned with 
    coefficient norm $\lambda$ if for every $f \in \Fcal$
    there is a signed measure $c_f$ on $(\Omega,\Sigma)$
    such that $\| c_f \| = | c_f | (\Omega)  \leq \lambda$ and
    \[
        \forall x \in \Gamma \quad
        \left| f(x) - \int_{\Omega} g_{\omega}(x) \, \mathrm{d}c_f(\omega) \right| \leq \eps  .
    \]
\end{definition}

\Cref{lem:ma-basis-decom} generalizes to 
$(\Gcal,\eps)$-spanned classes under this definition,
as follows.
\begin{lemma} \label{lem:ma-infin-decom}
    Let $(\Omega,\Sigma)$ be a measurable space and
    and $\Gcal = \left\{ g_\omega: \Gamma \rightarrow [0,1] \; \mid \; \omega \in \Omega \right\}$ 
    a class of functions indexed by $\Omega$. 
    Suppose $\Fcal = \{ f: [0,1] \rightarrow [-1,1]\}$ 
    is a class of functions that is 
    $(\Gcal,\eps)$-spanned with coefficient norm $\lambda$.
    Then for any sequence of $T$ predictions $\pb$ and context, outcome pairs $\xb, \yb$ 
    \[
    \Fcal\maerr (\pxy) \leq \lambda \cdot \Gcal\maerr (\pxy) + \eps T
    \]
    Similarly, in the distributional setting, for any distribution $\D$ on $\X \times [0,1]$ and any predictor $p : \X \to [0,1]$, 
    \[
    \Fcal\maerr_\Dcal(p) \leq \lambda \cdot \Gcal \maerr_\Dcal (p) + \eps .
    \]
\end{lemma}
\begin{proof} 
We can write every $f \in \Fcal$ as follows: $f(z) = \int_{\Omega} g_{\omega}(z) \, \mathrm{d}c_f(\omega) + \eps_f (z)$ such that $|c|(\Omega) \leq \lambda$, $\| \eps_f \|_{\infty} \leq \eps$ and $g_\omega \in \Gcal$  for all $\omega \in \Omega$. Consequently,
\begin{align}
\sup_{f \in \Fcal} \left| \sum_{t=1}^T f(x_t) (y_t - p_t(x_t)) \right| 
&= \sup_{f \in \Fcal} \left| \sum_{t=1}^T \left( \int_{\Omega} g_\omega(x_t) (y_t - p_t(x_t)) \, \mathrm{d}c_f(\omega)  \right) + \eps_f(x_t)(y_t - p_t(x_t)) \right| \\
&\leq \sup_{f \in \Fcal} \left\{ \, \left| \int_{\Omega} \left( \sum_{t=1}^T  g_\omega(x_t) (y_t - p_t(x_t)) \right) \, \mathrm{d}c_f(\omega) \right| 
      + \left| \sum_{t=1}^T \eps_f(x_t) (y_t - p_t(x_t)) \right| \, \right\} \\
&\leq \sup_{f \in \Fcal} \int_{\Omega} \left| \sum_{t=1}^T g_\omega(x_t) (y_t - p_t(x_t)) \right| \, \mathrm{d}c_f(\omega) \; + \; \eps T  \tag{since $|\eps_f(x_t) | \leq \eps, \, |y_t - p_t(x_t)| \leq 1$}\\
&\leq  \sup_{f \in \Fcal} \|c_f\| \cdot \sup_{\omega \in \Omega} \left| \sum_{t=1}^T g_\omega(x_t) (y_t - p_t(x_t)) \right| \; + \; \eps T \\
&\leq \lambda \sup_{g \in \Gcal} \left| \sum_{t=1}^T g(x_t) (y_t - p_t(x_t)) \right| + \eps T 
\end{align}
The bound for multiaccuracy error in the distributional setting follows by an identical calculation,
substituting expectations over $\D$ for sums over $t \in \{1,\ldots,T\}$.
\begin{align}
\sup_{f \in \Fcal} \left| \E_{(x,y) \sim \Dcal} \left[ f(x) (y - p(x)) \right] \right| 
&= \sup_{f \in \Fcal} \left|  \E_{(x,y) \sim \Dcal} \left[ \left( \int_{\Omega} g_\omega(x) (y - p(x)) \, \mathrm{d}c_f(\omega)  \right) + \eps_f(x)(y - p(x)) \right] \, \right| \\
&\leq \sup_{f \in \Fcal} \left\{ \, \left| \int_{\Omega}  \E_{(x,y) \sim \Dcal} \left[ g_\omega(x) (y - p(x)) \right] \, \mathrm{d}c_f(\omega) \right| 
+ \left| \E_{(x,y) \sim \Dcal} \left[ \eps_f(x) (y - p(x)) \right] \, \right| \, \right\} \\
&\leq \sup_{f \in \Fcal} \int_{\Omega} \left| \E_{(x,y) \sim \Dcal} \left[ g_\omega(x) (y - p(x)) \right] \, \right| \, \mathrm{d}c_f(\omega) \; + \; \eps  \tag{since $|\eps_f(x) | \leq \eps, |y - p(x)| \leq 1$}\\
&\leq  \sup_{f \in \Fcal} \|c_f\| \cdot \sup_{\omega \in \Omega} \left| \E_{(x,y) \sim \Dcal} \left[ g_\omega(x) (y - p(x)) \right] \, \right| \; + \; \eps \\
&\leq \lambda \sup_{g \in \Gcal} \left| \E_{(x,y) \sim \Dcal} \left[ g(x) (y - p(x)) \right] \, \right| + \eps 
\end{align}
\end{proof}

 \section{Omniprediction for Notable Loss Classes}

This section instantiates \Cref{thm:infinite-omni} for concrete classes of losses, 
by exhibiting approximate bases with appropriate parameters. In particular, we focus on the classes of
all $1$-Lipschitz convex losses ($\lcvx$), all 1-Lipschitz 
losses ($\llip$), and all bounded variation losses ($\lbv$) 
including proper losses ($\lprop$).
The bases we develop in this section will also be used in developing efficient offline algorithms for omniprediction with respect to these infinite classes.

\subsection{Online Omniprediction bounds for Notable Loss Classes}

\begin{theorem} \label{thm:omni-notable}
Let $\Hcal = \{ h : \X \to [0,1] \}$ be a hypothesis class. 
\begin{enumerate}
    \item \label{omni:lcvx} Given an online weak agnostic learning oracle
for $\mathrm{ReLU}^{1/T} \circ \Hcal$, \Cref{alg:online-omni} implements an online omnipredictor for 
loss class $\lcvx$ with regret
\begin{equation} \label{eq:omni:lcvx}
(\lcvx, \Hcal)\omnireg \leq O\left( \sqrt{T \ln T} + \owalreg_{\mathrm{ReLU}^{1/T} \circ \Hcal}(T)\right) .
\end{equation}
    \item \label{omni:llip} Given an online weak agnostic learning oracle
    for $\thrclass{1/T} \circ \Hcal$, \Cref{alg:online-omni} implements an online omnipredictor for 
loss class $\llip$ with regret
\begin{equation} \label{eq:omni:llip}
(\llip, \Hcal)\omnireg \leq O\left( \sqrt{T \ln T} + \owalreg_{\mathrm{\thrclass{1/T}} \circ \Hcal}(T)\right) .
\end{equation}
    \item \label{omni:lbv} Suppose the hypotheses in $\Hcal$ are $\Gamma$-valued, where $\Gamma$ is a finite subset of $[0,1]$
    containing $\{0,1\}$. Given an online weak agnostic learning
    oracle for $\thrclass{\Gamma} \circ \Hcal$, \Cref{alg:online-omni} implements an online omnipredictor for 
loss class $\lbv$ with regret
\begin{equation} \label{eq:omni:lbv}
(\lbv, \Hcal)\omnireg \leq O\left( \sqrt{T \ln T} + \owalreg_{\mathrm{\thrclass{\Gamma}} \circ \Hcal}(T)\right) .
\end{equation}
\end{enumerate}
If $\Hcal$ is a finite hypothesis class, then the omniprediction bounds in parts~\ref{omni:lcvx}
and~\ref{omni:llip} above are $O \left( \sqrt{T \ln(|\Hcal| \cdot T)} \right)$, and the bound in part~\ref{omni:lbv}
is $O \left(\sqrt{T \ln(|\Gamma| \cdot |\Hcal| \cdot T)} \right)$.
\end{theorem}
\begin{proof}
    The bounds~\eqref{eq:omni:lcvx}, \eqref{eq:omni:llip}, \eqref{eq:omni:lbv} 
    follow directly from application of \Cref{thm:online-omni} with 
    \Cref{cor:cvx-basis}, \Cref{lem:lip-basis}, and \Cref{lem:dthr-lbv-basis},
    respectively. The bounds for finite hypothesis classes follow 
    by applying \Cref{thm:infinite-omni}.
\end{proof}
To aid in interpreting the regret bounds~\eqref{eq:omni:lcvx}-\eqref{eq:omni:lbv}, we remind the reader 
that for any class $\Ccal$ there exists a (not necessarily computationally efficient) online
weak agnostic learning oracle satisfying $\owalreg_{\Ccal}(T) = O(T \cdot \mathsf{srad}_T(\Ccal))$, 
where $\mathsf{srad}$ denotes sequential Rademacher complexity.

\subsection{Approximate Bases for Notable Loss Classes}

In this subsection we exhibit approximate bases for the loss classes listed above. 
The quantitative consequences for omniprediction will be detailed in the following section. 

\paragraph{Convex Lipschitz Loss Functions.}
Let $\Fcal_{\mathrm{cvx}}$ be the class of all convex 1-Lipschitz functions.
\begin{lemma}[\cite{gopalan2024regression}]\label{lem:cvx-basis}
For all $\eps > 0$, 
$\Fcal_{\mathrm{cvx}}$ admits a finite $\eps$-basis of ReLU functions of size $\tilde{O}(1/\eps^{2/3})$ with coefficient norm $2$.
\end{lemma}

Convex functions are not closed under linear combinations. This means that not all functions in $\Delta\lcvx$ will be convex, therefore we cannot apply the result from section above. 
However, the fact that this class is derived from the difference of two convex functions still allows us to derive useful upper bounds.

\begin{corollary}[of \Cref{lem:cvx-basis}]\label{cor:cvx-basis}
For all $\eps>0$, $\Delta\lcvx$ admits a finite $\eps$-basis of ReLU functions of size $\tilde{O}(1/\eps^{2/3})$ with coefficient norm $4$.
\end{corollary}
\begin{proof}
    Let $\Gcal \subset \mathrm{ReLU}$ be a finite $\eps/2$-basis for $\Fcal_{\mathrm{cvx}}$ of size $\tilde{O}(1/\eps^{2/3})$ with coefficient norm $2$.
    Denote the elements of $\Gcal$ as $g_1,\ldots,g_d$ where $d = \tilde{O}(1/\eps^{2/3})$. 
    If $f \in \Delta\lcvx$ then $f(x) = f_1(x) - f_0(x)$ where $f_0, f_1 \in \Fcal_{\mathrm{cvx}}$. Let $c_{0i}, c_{1i} \; (1 \le i \le d)$
    be coefficients such that $\| f_0 - \sum_{i=1}^d c_{0i} g_i \|_{\infty} \leq \eps/2,$  
    $\| f_1 - \sum_{i=1}^d c_{1i} g_i \|_{\infty} \leq \eps/2,$ 
    $\sum_{i=1}^d |c_{0i}| \leq 2,$ and $\sum_{i=1}^d |c_{1i}| \leq 2.$ 
    Then for the coefficients $c_i = c_{1i} - c_{0i}$ we have
    \begin{align*}
        \left\| f - \sum_{i=1}^d c_i g_i \right\|_{\infty} & = 
        \left\| (f_1 - f_0) - \sum_{i=1}^d (c_{1i} - c_{0i}) g_i \right\|_{\infty} \leq 
        \left\| f_1 - \sum_{i=1}^d c_{1i} g_i \right\|_{\infty} + \left\| f_0 - \sum_{i=1}^d c_{0i} g_i \right\|_{\infty} \leq \eps \\
        \sum_{i=1}^d |c_i| & \leq \sum_{i=1}^d |c_{0i}| + \sum_{i=1}^d |c_{1i}| \leq 4,
    \end{align*}
    which confirms that $\Gcal$ is a finite $\eps$-basis of ReLU functions for $\Delta\lcvx$ of size $\tilde{O}(1/\eps^{2/3})$ and coefficient norm $4$.
\end{proof}

\paragraph{Lipschitz Loss Functions.} Every 1-Lipschitz loss function can be $\eps$-approximated by a finite weighted sum of 
ReLU functions, but the coefficient norm of this approximation is not bounded by a constant (independent of $\eps$). To
achieve approximation with bounded coefficient norm, it is necessary to use threshold functions. 

\begin{definition} 
    For $\theta \in [0,1]$ let $\mathrm{Th}_{\theta}$ denote the $\{ \pm 1 \}$-valued 
    threshold function $\mathrm{Th}_{\theta}(v) = \sgn(v-\theta),$ with the 
    convention that $\mathrm{Th}_{\theta}(\theta) = 1$. 
    For a subset $\Gamma \subseteq [0,1]$ let $\thrclass{\Gamma} = \{ \mathrm{Th}_{\theta} \, | \, \theta \in \Gamma \}$.
    When $\Gamma(\eps)$ is the set of integer multiples of $\eps$ in $[0,1]$, we will abbreviate
    $\thrclass{\Gamma(\eps)}$ as $\thrclass{\eps}$. We will also abbreviate 
    $\thrclass{[0,1]}$ as $\allthr$.
\end{definition} 

\begin{lemma}[Threshold Basis for Lipschitz Functions]\label{lem:lip-basis}
    For the class $\llip$ of 1-Lipschitz loss functions, the class $\Gcal = \thrclass{\eps/2}$ is a
    finite $\eps$-basis for $\Delta \llip$ of size $\lceil \frac{2}{\eps} + 1 \rceil$ and coefficient norm 4. 
\end{lemma}
\begin{proof}[Proof of \Cref{lem:lip-basis}]
To simplify notation we will prove $\thrclass{\eps}$ is a 
 $2\eps$-basis for $\Delta \llip$ of size $\lceil \frac{1}{\eps} + 1 \rceil$ and coefficient norm 2. 
(The lemma follows by reinterpreting $\eps$ in this proof as $\eps/2$ in the lemma statement.)
For $\ell \in \llip$ the function $f(p) = \Delta \ell(p) = \ell(p,1) - \ell(p,0)$ is 2-Lipschitz and $[-1,1]$-valued. 
We can construct a piecewise constant function $\hat{f}$ by setting $\hat{f} (x) = f \left( \lfloor \frac{x}{\eps} \rfloor \cdot \eps \right)$ so that $|f(x) - \hat{f}(x)| \leq 2 \eps$. We'll now express $\hat{f}$ as a linear combination of functions in $\thrclass{\eps}$. Let $g_i(v)$ denote the function $\mathrm{Th}_{i\eps} \in \thrclass{\eps}$.
\[
\hat{f}(x) = \hat{f} (0) g_0 (x) + \sum_{i \in [\lceil 1/\eps \rceil]} (\hat{f}(i \eps) - \hat{f} ((i-1)\eps)) g_i (x) .
\]
To see this, observe that we can simplify the RHS to
\[
\hat{f} (0) (g_0 (x) - g_1 (x)) + \sum_{i \in [\lceil 1/\eps \rceil]} \hat{f} (i \eps) (g_i (x) - g_{i+1} (x))
\]
Since $(g_i (x) - g_{i+1} (x))$ is 1 only when $i \eps \leq  x < (i+1)\eps$ and $\hat{f} (x)$ is constant in this interval, we get the desired result.
Now we need to bound the coefficient norm. We know that $|\hat{f} (0)| \leq 1$ and since $f$ is 1-lipschitz, $|(\hat{f}(i\eps) - \hat{f} ((i-1)\eps))| \leq \eps$ for all $i$. Thus, the coefficient norm is bounded by 2.
\end{proof}

\paragraph{Bounded Variation Losses and Proper Losses.}
Let $\lbv$ denote the class of bounded variation losses, defined as follows.
\begin{definition}[Bounded Variation]
  A function $f : [0,1] \to \mathbb{R}$ has \emph{bounded variation} if the quantity 
  \[ V(f) = \sup \left\{ \left. \sum_{i=1}^n |f(x_i) - f(x_{i-1})| \; \right| 0 = x_0 < x_1 < \cdots < x_n = 1 \right\} \]
  is finite. The class $\lbv$ of bounded variation losses consists of all loss functions $\ell(p,y)$ taking values in $[-1,1]$
  that satisfy $V(\Delta \ell) \leq 2$.
\end{definition}
Bounded variation losses are an extremely general family that includes all other loss
classes considered in this paper. 
\begin{lemma} \label{lem:lbv}
The class $\lbv$ includes all 1-Lipschitz losses,
convex losses (with values in $[-1/4, \, 1/4]$, 
regardless of whether or not they are Lipschitz continuous), 
and proper losses taking values in $[-1,1]$.
\end{lemma}
\begin{proof}
It is clear from the definition that $V(f) \leq 2$ when 
$f$ is monotone and $[-1,1]$-valued, and also when $f$ is 2-Lipschitz.
Hence, the class $\llip$ of 1-Lipschitz losses is in $\lbv$ (since 
$\Delta \ell$ is 2-Lipschitz whenever $\ell$ is 1-Lipschitz) and 
the class $\lprop$ of proper losses is in $\lbv$ (since $\Delta \ell$
is monotone whenever $\ell$ is proper). For a convex function
$f$ taking values in $[-1/4, \, 1/4]$, if $x^*$ denotes a point in $[0,1]$ where $f$ attains its
minimum, then $f$ admits a representation of the form
\[
    f(x) = f(x^*) + f_0(x) + f_1(x)
\]
where $f_0, f_1$ are monotone non-increasing and non-decreasing
functions (respectively) from $[0,1]$ to $[0,1/2]$ satisfying
$f_0(x) = 0$ for $x \geq x^*$ and $f_1(x) = 0$ for $x \leq x^*$.
From this representation it is clear that $V(f) \leq V(f_0) + V(f_1) \leq 1.$
If $\ell(p,y)$ is a convex loss taking values in $[-1/4, \, 1/4]$ then 
$\Delta \ell$ is a difference of two convex $[-1/4, \, 1/4]$-valued
functions, so $V(\Delta \ell) \leq 2.$
\end{proof} 
For bounded variation losses we have the following approximate basis.
\begin{lemma} \label{lem:lbv-basis}
    For any $\eps>0$, the class $\allthr$ of all threshold functions on $[0,1]$
    is a $\eps$-basis for $\Delta \lbv$ with sparsity $\lceil 2 / \eps + 1 \rceil$
    and coefficient norm 3.
\end{lemma}
\begin{proof}
    For $\ell \in \lbv$ let $f = \Delta \ell$, and consider the 
    sequence $x_0 < x_1 < x_2 < \cdots < x_{n}$ defined inductively
    by setting $x_0 = 0$ and $x_{m+1} = \inf \{ x > x_m \, : \, |f(x) - f(x_m)| > \eps \}$
    for all $m \geq 0$ such that the set in question is non-empty.
    The sequence ends with the first element 
    $x_{n}$ such that $\{ x > x_{n} \, : \, |f(x) - f(x_{n})| > \eps \}$
    is empty. For notational convenience we define $x_{n+1}=1$.
    From the definition of $V(f)$ we see that 
    $V(f) \geq n \eps$, from which we deduce $n \leq 2 / \eps$.

    Define a piecewise-constant function $\hat{f} : [0,1] \to [-1,1]$ 
    by setting $\hat{f}(x) = f(x_m)$ where $x_m$ is the 
    maximum element of $\{x_0, \ldots, x_{n+1}\} \cap [0,x]$.
    By construction, the inequality $|f(x) - \hat{f}(x)| \leq \eps$
    holds for all $x \in [x_m, x_{m+1}]$, for all $m \in \{0,\ldots,n\}$.
    The union of the intervals $[x_m, x_{m+1}]$ equals $[0,1]$, so 
    $\| f - \hat{f} \|_{\infty} \leq \eps$. 
    
    For $i=0,1,\ldots,n$, let $g_i$ denote the threshold function
    $\mathrm{Th}_{x_i}$. As in the proof of \Cref{lem:lip-basis}
    we have 
    \[ \hat{f}(x) = \hat{f} (0) g_0 (x) + \sum_{i=1}^{n+1} \left( \hat{f}(x_i) - \hat{f} (x_{i-1}) \right) g_i (x) . \]
    The coefficient norm is bounded by $|\hat{f}(0)| + \sum_{i=1}^{n+1} |\hat{f}(x_i) - \hat{f} (x_{i-1}|
    \leq 1 + V(f) \leq 3.$
\end{proof}
The approximate basis $\allthr$ is unfortunately not finite. 
This difficulty is inherent: the class $\Delta \lprop$ includes, 
for each $\theta \in (0,1)$, a step function with a step of height 1 
at $\theta$. To approximate such a step function within $\eps$ in
the $\infty$-norm, one must use a function having a jump discontinuity 
at $\theta$, provided $\eps < 1/2$. Hence, any $\eps$-approximate 
basis for $\Delta \lprop$ must include functions with jump 
discontinuities at every $\theta \in (0,1)$.
However, this difficulty can be overcome for the class
$\Delta \lbv \circ \Hcal$ whenever $\Hcal$ is a class of hypothesis functions
taking values in a finite subset of $[0,1]$.
\begin{lemma} \label{lem:dthr-lbv-basis}
If $\Hcal$ is a class of hypothesis functions
taking values in a finite set $\Gamma$ with 
$\{0,1\} \subseteq \Gamma \subset [0,1]$,
then for all $\eps>0$ the class 
$\thrclass{\Gamma} \circ \Hcal$ is a 
$\eps$-basis for $\Delta \lbv \circ \Hcal$
with sparsity $\lceil 2 / \eps + 1 \rceil$
and coefficient norm~3.
\end{lemma}
\begin{proof}
    From \Cref{lem:postcomposition} and \Cref{lem:lbv-basis} 
    we know that $\allthr \circ \Hcal$ is a 
    $\eps$-basis for $\Delta \lbv \circ \Hcal$
    with sparsity $\lceil 2 / \eps + 1 \rceil$
    and coefficient norm 3. However, for every
    $\theta \in [0,1]$, if $\gamma$ is the 
    minimum element of $\Gamma \cap [\theta,1]$
    then $\mathrm{Th}_{\theta} \circ h = 
    \mathrm{Th}_{\gamma} \circ h$ for every 
    $h \in \Hcal$. Hence, $\thrclass{\Gamma} \circ \Hcal =
    \allthr \circ \Hcal$ and the lemma follows.
\end{proof}

\begin{corollary}[of \Cref{lem:lbv-basis}]
Let $\Hcal$ be a class of hypothesis functions. Then 
 \[
 \mathsf{rad}_m (\Delta \lbv \circ \H) = O(\mathsf{rad}_m (\Th \circ \Hcal))
 \]   
\end{corollary}

\begin{proof}
By Lemma \ref{lem:lbv-basis}, for any $\epsilon > 0$, the class $\allthr$ of threshold functions is an $\epsilon$-basis for $\Delta \lbv$ with sparsity $s = O(1/\epsilon)$ and coefficient norm at most 3. This means every $g \in \Delta \lbv$ can be written as $g(\cdot) = \sum_{j=1}^s \alpha_j t_j (\cdot )$ where each $t_j \in \allthr$.
By the definition of Rademacher complexity,
\[
\mathsf{rad}_m (\Delta \lbv \circ \Hcal) = \mathbb{E}_{\sigma} \left[ \sup_{f \in \Delta \lbv \circ \Hcal} \frac{1}{m} \sum_{i=1}^{m} \sigma_i f(x_i) \right].
\]
Substituting the basis representation of $f$, we obtain
\[
\sup_{f \in \Delta \lbv \circ \Hcal} \frac{1}{m} \sum_{i=1}^{m} \sigma_i f(x_i) = \sup_{\|\alpha\|_1 \leq 3} \frac{1}{m} \sum_{j=1}^{s} \alpha_j \sum_{i=1}^{m} \sigma_i t_j(h(x_i)).
\]
Using the convexity of $\sup$, we obtain
\[
\mathsf{rad}_m (\Delta \lbv \circ \Hcal) \leq 3 \mathsf{rad}_m (\allthr \circ \Hcal).
\]

The reverse direction follows from the fact that $\Th \circ \Hcal$ in $\Delta \lbv \circ \Hcal$ (recall for each $v$-shaped function, $\Delta \ell_v (\cdot) = \Th_v (\cdot)$), this implies
\[
\mathsf{rad}_m (\Th \circ \Hcal) = O(\mathsf{rad}_m (\Delta \lbv \circ \Hcal)).
\]
Combining both bounds gives the desired result.
\end{proof}

 \section{Offline Omniprediction} \label{sec:offline-omni}

In this section, we show how the observation that proper calibration and multiaccuracy are sufficient for omniprediction leads to optimal sample complexity bounds for omniprediction in the offline setting.
In particular, the algorithm we describe here will use a sample complexity that, for a given loss class $\L$, depends near-optimally on the sample complexity to cover $\Delta \L \circ \H$.
In \Cref{sec:online-to-batch}, we describe an online-to-batch procedure that outputs a \textit{randomized} omnipredictor by running \Cref{alg:online-omni} on i.i.d.\ samples from the distribution $\Dcal$. 
In all, we establish the following result.

\begin{theorem}
\label{thm:offline}
    There exists an sample-efficient algorithm $\A$ that for any distribution $\D$ supported on $\X \times \{0,1\}$, for any class of loss functions $\L \subseteq \lbv$ that is $(\Gcal,T^{-1/2})$-spanned with coefficient norm $\lambda$, any hypothesis class $\H$, and $\eps > 0$, learns an $(\L,\H,\eps)$-omnipredictor with the following properties:
    \begin{itemize}
        \item $\A$ returns a randomized omnipredictor that mixes over $\poly(1/\eps)$ postprocessed hypotheses from $ \H$.
        \item $\A$ uses $m \le \tilde{O}(\lambda \cdot d_{\Delta \Gcal \circ \Hcal}/\eps^2)$ samples drawn i.i.d.\ from $\D$, where $d_{\Delta \G \circ \H}$ denotes the VC dimension of $\Delta \G \circ \H$ or the fat-shattering dimension at scale $\eps$ in the case of real-valued class.
\end{itemize}
In particular, for any class $\Lcal \subseteq \lbv$ of bounded-variation losses --- including 
    proper losses, convex losses, and 1-Lipschitz losses --- 
    the sample complexity of $(\lbv,\Hcal)$-omniprediction scales with the statistical complexity of  $\allthr \circ \Hcal$, and 
    the sample complexity of $(\lcvx,\Hcal)$-omniprediction scales with the statistical complexity of $\mathrm{ReLU}^{1/T} \circ \Hcal$.
\end{theorem}

This theorem follows from \Cref{thm:o2b}, using \Cref{cor:owal_lit} to instantiate the online weak agnostic learner.

\subsection{Learning Randomized Omnipredictors via Online-to-Batch Conversion}\label{sec:online-to-batch}
To establish offline omnipredictors, we prove the following technical result.
\begin{theorem}\label{thm:o2b}
Let $\Hcal$ be a class of hypothesis functions.
Let $\Lcal$ and $\Gcal$ be classes of loss functions such that $\Lcal$ is $(\Gcal,\eps)$-spanned with coefficient norm $\lambda$ for some $\eps \leq \frac{1}{\sqrt{T}}$.
Given a \textit{online} weak agnostic learner for $\Delta \Gcal \circ \Hcal$, failure probability $\delta$, and a sequence of $T$ i.i.d samples $(x_1,y_1), \ldots (x_T,y_T) \sim \Dcal$, \Cref{alg:online-omni} 
outputs a sequence of predictors $p_1, \ldots, p_T$ such that the randomized predictor $\hat \pb = \text{unif} \{p_1, \ldots, p_T \}$ satisfies
\begin{align*}
\E_{p \sim \hat{\pb}} \E_{(x,y) \sim \Dcal} \left[ \ell ( k_\ell (p(x)), y) \right] \leq \min_{h \in \Hcal} \E_{(x,y) \sim \Dcal} \left[ \ell ( h(x), y) \right] 
&+ 
O \left( \lambda \cdot \mathsf{rad}_T (\Delta \Gcal \circ \Hcal) + \lambda \sqrt{\frac{\ln T/\delta}{T}} \right) \\
&+ \frac{\lambda}{T} \cdot \owalreg_{\Delta \Gcal \circ \Hcal} (T)
\end{align*}
with probability $1 - \delta$ over the randomness of the algorithm and the sampling from $\Dcal$. Moreover, each returned predictor can be represented as $p_t = v_t \circ q_t$, where $q_1, \ldots, q_T$ are the predictors returned by the online weak agnostic learner and $v_t : [-1,1] \to [0,1]$ is a post-processing function.
\end{theorem}
This theorem is the main result that we prove across this section.
We break the proof into a series of technical lemmas.
The lemmas serve to bound the online regret, which quantifies the error of $\hat \pb$ at approximating the empirical statistics, and separately bound the generalization of the empirical statistics to their distributional quantities.
First, we give the proof of the theorem assuming our subsequent lemmas, followed by technical sections to establish the lemmas.
\begin{proof}[Proof of \Cref{thm:o2b}]
Since $\Lcal$ is $(\Gcal,\eps)$-spanned with coefficient norm $\lambda$, then $\Delta\Lcal$ is $(\Delta \Gcal, \eps)$-spanned with coefficient norm $2 \lambda$.
In \Cref{thm:online-omni}, we show that, given an online weak agnostic learner for $\Delta \Gcal \circ \Hcal$, \Cref{alg:online-omni} outputs a sequence of predictions $p_1, \ldots, p_T$ such that 
\[
 \E [\pcal (\pxy)]  + \E [(\Delta \Gcal \circ \Hcal) \maerr (\pxy)] \leq O \left( \sqrt{T \ln T/\delta} \right) + \owalreg_{\Delta \Gcal \circ \Hcal} (T)
\]
and together with \Cref{lem:ma-basis-decom}, this implies
\[
\E [(\Lcal, \Hcal)\omnireg (\pxy)] \leq O \left( \lambda \sqrt{T \ln T/\delta} \right) + \lambda \cdot \owalreg_{\Delta \Gcal \circ \Hcal} (T)
\]
That is, for all $\ell \in \Lcal, h \in \Hcal$,
\begin{equation}\label{eq:omni_reg}
\left[ \frac{1}{T} \sum_{t=1}^T \ell (k_\ell (p_t(x_t)), y_t) - \frac{1}{T} \sum_{t=1}^T \ell (h (x_t), y_t) \right] \leq O \left( \lambda \sqrt{\frac{\ln T/\delta}{T}} \right) + \frac{\lambda}{T} \cdot \owalreg_{\Delta \Gcal \circ \Hcal}^\delta (T)
\end{equation}
Now we would like to show that this also implies a bound on the expected omniprediction error under the true distribution.
\sloppy By \Cref{lem:pred_conv} and \Cref{lem:hyp_conv}, we know that with probability at least $1-\delta$,
\[
\sup_{\ell \in \Lcal} \;\; \left| \frac{1}{T} \sum_{t=1}^T \ell (k_\ell (p_t (x_t)), y_t) - \frac{1}{T} \sum_{t=1}^T \E_{(x,y) \sim \Dcal} [\ell (k_\ell (p_t (x)), y)] \right| \leq \frac{1}{T} + O \left( \sqrt{\frac{\ln \frac{T}{\delta}}{T}} \right)
\] since our algorithms output predictions that are multiples of $1/T$ and 
\[
\sup_{\ell \in \Lcal, h \in \Hcal} \;\; \left| \frac{1}{T} \sum_{t=1}^T \ell (h (x_t), y_t) - \E_{(x,y) \sim \Dcal} [\ell (h(x), y)] \right| \leq  O \left( \lambda \cdot \mathsf{rad}_T (\Delta \Gcal \circ \Hcal) + \lambda  \sqrt{\frac{\ln \frac{1}{\delta}}{T}} \right) .
\]
Plugging into the \Cref{eq:omni_reg}, we obtain that with probability at least $1-\delta$, 
\begin{align}
\sup_{\ell \in \Lcal, \, h \in \Hcal}
&\left[ \frac{1}{T} \sum_{t=1}^T \E_{(x,y) \sim \Dcal} [\ell (k_\ell (p_t(x), y)] - \E_{(x,y) \sim \Dcal} \ell (h (x), y) \right] \\
&\leq O( \lambda \cdot \mathsf{rad}_T (\Delta \Gcal \circ \Hcal) ) + O \left( \lambda \sqrt{\frac{\ln T/\delta}{T}} \right) + \frac{\lambda}{T} \cdot \owalreg_{\Delta \Gcal \circ \Hcal}^\delta (T)
\end{align}
The claim that each predictor $p_t$ can be expressed as $v_t \circ q_t$ follows the guarantees from the Augmented Proper Calibration algorithm in \Cref{lem:proper-recal}.
\end{proof}

\subsection{Implementing an Online Weak Agnostic Learner with Offline Sample Complexity}
Online weak agnostic learning with i.i.d.\ features is a special case of the hybrid online learning problem \citep{lazaric2009hybrid, 
wu2022expectedworstcaseregret}.
These prior results establish near-optimal dependence on the offline sample complexity of the class. In \Cref{cor:owal_lit}, we present this online weak agnostic learner, which is based on a multiplicative weights algorithm over a cover of size exponential in the VC (or fat-shattering) dimension of the class.

\begin{corollary}[of Theorem 3 of \cite{wu2022expectedworstcaseregret}]\label{cor:owal_lit}
Let $\Ccal$ be a class of hypothesis functions. Consider the setting where $x_1, \ldots, x_T$ are generated i.i.d.\ but revealed sequentially from a fixed distribution $\Dcal$. There exists a multiplicative weights algorithm that returns a sequence of hypothesis $c_1, \ldots, c_T \in \Ccal$ such that for any sequence of $y_1, \cdots, y_T$ with probability at least $1 - \delta$,
\[
   \sum_{t=1}^T c_t(x_t) \cdot y_t - \inf_{c \in \mathcal{C}} \sum_{t=1}^T c(x_t) \cdot y_t  \leq \tilde{O} \left(\sqrt{T d_\Ccal } \right) + \tilde{O}\left(\sqrt{T \ln 1/\delta} \right).
\]
where $d_\Ccal$ represents the VC dimension of $\Ccal$ in the case of a binary hypothesis class or the fat-shattering dimension at scale $1/\sqrt{T}$ in the case of a real-valued hypothesis class.
\end{corollary}
The proof of \Cref{cor:owal_lit} is deferred to \Cref{sec:owal-off}.

\subsection{Uniform Convergence Results for Online-to-Batch Conversion}
Next, we establish the generalization bounds necessary for the empirical omniprediction statistics to converge to their distributional quantities.

\begin{lemma}[Uniform Convergence for discretized predictors]\label{lem:pred_conv}
Let $p_1, \ldots, p_T: \Xcal \rightarrow \{ 0, \eps, 2\eps \ldots, 1\}$ be a sequence of $T$ predictors and let $\{ (x_1, y_1), \ldots, (x_T, y_T) \}$ be a sequence of i.i.d samples drawn from a distribution $\Dcal$. Then for any class of loss functions $\Lcal$, the following holds with probability at least $1 - \delta$,
\[
\left| \frac{1}{T} \sum_{t=1}^T \ell (k_\ell (p_t (x_t)), y_t) - \frac{1}{T} \sum_{t=1}^T \E_{(x,y) \sim \Dcal} [\ell (k_\ell (p_t (x)), y)] \right| \leq \eps + O \left( \sqrt{\frac{\ln \frac{1}{\eps \delta}}{T}} \right)
\]
for all $\ell \in \Lcal$
\end{lemma}
\begin{proof}[Proof of \Cref{lem:pred_conv}]
Observe that it suffices to show that this holds for $\lprop$. This is because, for any class of loss functions $\Lcal$, the derived class $\{ \ell \circ k_\ell : \ell \in \Lcal\}$ is a subset of $\lprop$.
Recall that for every $\ell \in \lprop$, there exists nonnegative coefficients $c_v(\ell)$ such that $\int_0^1 c_v (\ell) dv \leq 2$ and
\[
\ell (p,y) = \int_{0}^1 c_v(\ell) \ell_v (p, y) dv
\]
Consequently, it suffices to show that this holds for all $\ell_v$ for $v \in [0,1]$. To see this observe that for any $\ell \in \lprop$
\begin{align}
&\left| \frac{1}{T} \sum_{t=1}^T \ell (p_t (x_t), y_t) - \frac{1}{T} \sum_{t=1}^T \E_{(x,y) \sim \Dcal} [\ell (p_t (x), y)] \right| \\
&= \left| \frac{1}{T} \sum_{t=1}^T \int_{0}^1 c_v(\ell) \ell_v (p_t (x_t), y_t) dv - \frac{1}{T} \sum_{t=1}^T \E_{(x,y) \sim \Dcal} \left[\int_{0}^1 c_v(\ell) \ell_v (p_t(x), y) dv \right] \right| \\
&\leq \int_{0}^1 c_v(\ell) \left| \frac{1}{T} \sum_{t=1}^T  \ell_v (p_t (x_t), y_t) - \frac{1}{T} \sum_{t=1}^T \E_{(x,y) \sim \Dcal} \left[ \ell_v (p_t(x), y) \right] \right| dv
\end{align}

Now we'll take advantage of the fact that $p_t (x)$ is in $\{ 0, \eps, 2\eps \ldots, 1\}$ for all $x \in \Xcal$. We'll show that for all $v \in [0,1]$, there exists $v' \in \{ 0, \eps, 2\eps \ldots, 1\}$ such that for all $x,y \in \Xcal \times \Ycal$, $|\ell_v (p(x), y) - \ell_{v'} (p(x), y)| \leq \eps$ for all $v$. In particular, $v' = \eps \lceil \frac{v}{\eps} \rceil$. We'll also adopt the convention that $\mathrm{sgn}(0)=1$. We'll first make the observation that because $p(x) \in \{0, \eps, \ldots, 1\}$, $\sgn (v - p(x)) = \sgn (v' - p(x))$ by our choice of $v'$. Completing the argument, we have that 
\begin{align}
\ell_v (p(x), y) - \ell_{v'} (p(x), y)
&= (y - v) \sgn (v - p(x)) - (y - v') \sgn (v' - p(x))  \\
&= (y - v) \sgn (v - p(x)) - (y - v') \sgn (v - p(x)) \tag{since $\sgn (v - p(x)) = \sgn (v' - p(x))$} \\
&\leq (v - v') \sgn (v - p(x)) \\
&\leq \eps
\end{align}

Now we'll show that for all $v' \in \{ 0, \eps, 2\eps \ldots, 1\}$, 
\[
\left| \frac{1}{T} \sum_{t=1}^T \ell_v (p_t (x_t), y_t) - \frac{1}{T} \sum_{t=1}^T \E_{(x,y) \sim \Dcal} \left[ \ell_v (p_t(x), y) \right] \right| \leq O \left( \sqrt{\frac{\ln \frac{1}{\eps \delta}}{T}} \right)
\]
Define $Z_t (v') = \ell_{v'} (p_t (x_t), y_t) - \E_{(x,y) \sim \Dcal} [\ell_{v'} (p_t (x), y)]$. Observe that $Z_t (v')$ for $t = 1, \ldots, T$ forms a martingale difference sequence with bounded variance of 2. Thus, applying Azuma-Hoeffding's inequality, together with the union bound over all $v' \in \{ 0, \eps, 2\eps \ldots, 1\}$, we get that with probability at least $1 - \delta$, 
\[
\forall v' \in \{ 0, \eps, 2\eps \ldots, 1\} \quad 
\left| \frac{1}{T} \sum_{t=1}^T Z_t (v') \right| \leq  O \left( \sqrt{\frac{\ln \frac{1}{\eps \delta}}{T}} \right)
\]
Concluding the proof, we have that for any $v \in [0,1]$, with probability at least $1 - \delta$,
\[
\left| \frac{1}{T} \sum_{t=1}^T  \ell_v (p_t (x_t), y_t) - \frac{1}{T} \sum_{t=1}^T \E_{(x,y) \sim \Dcal} \left[ \ell_v (p_t(x), y) \right] \right| \leq \eps + O \left( \sqrt{\frac{\ln \frac{1}{\eps \delta}}{T}} \right)
\]
Consequently, for any $\ell \in \lprop$,
\begin{align}
&\left| \frac{1}{T} \sum_{t=1}^T \ell (p_t (x_t), y_t) - \frac{1}{T} \sum_{t=1}^T \E_{(x,y) \sim \Dcal} [\ell (p_t (x), y)] \right| \\
&\leq \int_{0}^1 c_v(\ell) \left| \frac{1}{T} \sum_{t=1}^T  \ell_v (p_t (x_t), y_t) - \frac{1}{T} \sum_{t=1}^T \E_{(x,y) \sim \Dcal} \left[ \ell_v (p_t(x), y) \right] \right| dv \\
&\leq 2 \eps + O \left( \sqrt{\frac{\ln \frac{1}{\eps \delta}}{T}} \right)
\end{align}
Since $\ell \circ k_\ell$ is in $\lprop$, this completes the result.
\end{proof}

\begin{lemma}[Uniform Convergence for $\Gcal$-spanning loss classes]\label{lem:hyp_conv}
Let $\Lcal$ be a $(\Gcal,\eps)$-spanned class of loss functions with coefficient norm $\lambda$ for some
$\eps \leq T^{-1/2}$. 
Let $\{ (x_1, y_1), \ldots, (x_T, y_T) \}$ be a sequence of i.i.d.\ samples drawn from a distribution $\Dcal$. Then for a hypothesis class $\Hcal$, the following holds with probability at least $1 - \delta$,
\[
\left| \frac{1}{T} \sum_{t=1}^T \ell (h (x_t), y_t) - \E_{(x,y) \sim \Dcal} [\ell (h(x), y)] \right| \leq O \left( \lambda \cdot \mathsf{rad}_T (\Gcal \circ \Hcal) \right) + O \left( \lambda \sqrt{\frac{\ln \frac{1}{\delta}}{T}} \right)
\]
uniformly for every $\ell \in \Lcal$ and $h \in \Hcal$.
\end{lemma}
\begin{proof}[Proof of \Cref{lem:hyp_conv}]
Since every $\ell \in \Lcal$ is $\eps$-close to the span of $\Gcal \circ \Hcal$, the proof will show that uniform convergence extends 
from $\Gcal \circ \Hcal$ to all losses $\ell \in \Lcal$ and hypotheses $h \in \Hcal$.
In particular, standard application of uniform convergence results (Chapter 26 of \cite{shalev2014understanding}) show that the following holds with probability at least $1 - \delta$: \begin{align}
  \label{eq:bobby.z}
\forall g \circ h \in \Gcal \circ \Hcal, \, z \in \{0,1\} \;\; \left| \frac{1}{T} \sum_{t=1}^T g (h (x_t),z) - \E_{(x,y) \sim \Dcal} [g(h(x),z)] \right| & \leq \mathsf{rad}_T (\Gcal \circ \Hcal) + O \left( \sqrt{\frac{\ln \delta^{-1}}{T}} \right) \\
\intertext{and}
  \label{eq:bobby.yz}
\forall g \circ h \in \Gcal \circ \Hcal \, z \in \{0,1\} \;\; \left| \frac{1}{T} \sum_{t=1}^T y_t \cdot g (h (x_t),z) - \E_{(x,y) \sim \Dcal} [y \cdot g(h(x),z)] \right| & \leq \mathsf{rad}_T (\Gcal \circ \Hcal) + O \left( \sqrt{\frac{\ln \delta^{-1}}{T}} \right)
\end{align}
By definition of $(\Gcal,\eps)$-spanning, we know that for every $\ell \in \Lcal$, there exists 
a signed measure $c_{\ell}(\omega)$ such that $\| c_{\ell} \| \leq \lambda$ and 
such that the function $\hat{\ell} : [0,1] \times \{0,1\} \to \mathbb{R}$ defined by 
\[
    \hat{\ell}(p,y) = \int_{\Omega} g_{\omega}(p,y) \, \mathrm{d}c_{\ell}(\omega) 
\]
satisfies $\| \ell - \hat{\ell} \|_{\infty} \leq \eps$.
Then it follows that for every $\ell \in \Lcal$ and $h \in \Hcal$
\begin{align} \nonumber
    \left| \frac{1}{T} \sum_{t=1}^T \ell (h (x_t), y_t) - \E_{(x,y) \sim \Dcal} [\ell (h(x), y)] \right| 
    & \leq 2 \eps \, + \, 
    \left| \frac{1}{T} \sum_{t=1}^T \hat{\ell} (h (x_t), y_t) - \E_{(x,y) \sim \Dcal} [\hat{\ell} (h(x), y)] \right| \\
    \nonumber
    & = 2 \eps \, + \, 
    \left| \int_{\Omega} \left( 
        \frac{1}{T} \sum_{t=1}^T  g_{\omega}(h (x_t), y_t) - \E_{(x,y) \sim \Dcal} [g_{\omega} (h(x), y)] 
    \right) \; \mathrm{d} c_{\ell}(\omega) \right| \\
    \nonumber
    & \leq 2 \eps \, + \, 
    \int_{\Omega} \left| 
        \frac{1}{T} \sum_{t=1}^T  g_{\omega}(h (x_t), y_t) - \E_{(x,y) \sim \Dcal} [g_{\omega} (h(x), y)] 
    \right| \; \mathrm{d} |c_{\ell}|(\omega) .\\
    \label{eq:bobby.1}
    & \leq 2 \eps \, + \, 
    \sup_{g \in \Gcal} \left| 
        \frac{1}{T} \sum_{t=1}^T  g(h (x_t), y_t) - \E_{(x,y) \sim \Dcal} [g(h(x), y)] 
    \right| \cdot \| c_{\ell} \| .
\end{align}
Observe that for any $g : [0,1] \times \{0,1\} \to \mathbb{R}$ and all $(p,y) \in [0,1] \times \{0,1\}$ we have
\[
    g(p,y) = g(p,0) + y g(p,1) - y g(p,0) .
\]
Hence, assuming~\eqref{eq:bobby.z} and~\eqref{eq:bobby.yz} hold, we have that for all $g \in \Gcal$,
\begin{align*}
    \left| \frac{1}{T} \sum_{t=1}^T  g(h (x_t), y_t) - \E_{(x,y) \sim \Dcal} [g (h(x), y)] \right|
    & \leq
    \left|  \frac{1}{T} \sum_{t=1}^T  g(h (x_t), 0) - \E_{(x,y) \sim \Dcal} [g (h(x), 0)] \right| \\
    & \quad + 
    \left|  \frac{1}{T} \sum_{t=1}^T  y_t \cdot g(h (x_t), 1) - \E_{(x,y) \sim \Dcal} [y \cdot g (h(x), 1)] \right| \\
    & \quad +
        \left|  \frac{1}{T} \sum_{t=1}^T  y_t \cdot g(h (x_t), 0) - \E_{(x,y) \sim \Dcal} [y \cdot g (h(x), 0)] \right| \\
    & \leq     3 \, \mathsf{rad}_T (\Gcal \circ \Hcal) + O \left( \sqrt{\frac{\ln \delta^{-1}}{T}} \right) .
\end{align*}
Substituting this bound back into inequality~\eqref{eq:bobby.1}, we obtain
\begin{align}
        \left| \frac{1}{T} \sum_{t=1}^T \ell (h (x_t), y_t) - \E_{(x,y) \sim \Dcal} [\ell (h(x), y)] \right| 
    & \leq 2 \eps + \left[ 3  \mathsf{rad}_T (\Gcal \circ \Hcal) + O \left( \sqrt{\frac{\ln \delta^{-1}}{T}} \right) \right]
    \| c_{\ell} \|  \\
    &\leq 2 \eps + 3 \lambda \mathsf{rad}_T (\Gcal \circ \Hcal) + 
    O \left( \lambda \sqrt{\frac{\ln \delta^{-1}}{T}} \right) 
\end{align}
and the result follows by our assumption that $\eps \leq T^{-1/2}.$
\end{proof}

 \section{Oracle-Efficient Offline Omniprediction} 
\label{sec:oracle-omni}

In this final section, we describe an algorithm that achieves omniprediction in the offline setting, returning predictors with efficient representations, using only an offline ERM oracle.
Our construction is modeled after previous algorithms, but critically, we replace the OWAL oracle with a new abstraction that we call a Distributional Online Weak Agnostic Learner.
We give an efficient implementation of this oracle given access to an ERM oracle, via a Frank-Wolfe reduction to solve an entropy-regularized ERM problem.

To begin, we introduce the new hybrid learning oracle.
\begin{definition}\label{def:dowal}
A \dowalboth is an online algorithm initialized with a hypothesis class $\Ccal$, a labeling function
class $\Rcal$, 
and a failure probability $\delta$. It is initialized with a dataset of samples $\{(x_t,y_t)\}_{t=1}^T$.
Each timestep $t$:
\begin{enumerate}
    \item The algorithm computes a distribution $Q_t$. When called as a subroutine inside another algorithm, it
    outputs hypothesis functions $q_t: \Xcal \rightarrow [-1,1]$ drawn independently at random from $Q_t$.
    \item The algorithm receives an adversarial labeling function $r_t: \Xcal \times \Ycal \rightarrow [-1,1] \in \Rcal$
\end{enumerate}
As in \Cref{def:owal} we assume a simultaneous-move (SM) protocol: both $Q_t$ and $r_t$ may only depend on data $\{q_s \, : \, s<t\}$ and $\{r_s \, : \, s < t \}$ sampled by the algorithm and adversary, respectively, in interaction rounds strictly prior to round $t$.

We say an algorithm implements a \dowal with regret bound \( \dowalreg_{\Ccal, \Rcal}^{m,\delta} (T) \) if given an initial dataset of samples $(x_i, y_i)_{i=1}^m$ drawn i.i.d.\ from a fixed distribution $\Dcal$, it guarantees with probability at least $1 - \delta$
\[
\max_{c \in \Ccal} \sum_{t=1}^T \E_\Dcal[c(x) r_t(x,y)] \;\leq\; \sum_{t=1}^T \E_\Dcal[q_t(x) r_t(x,y)] \;+\; \dowalreg_{\Ccal,\Rcal}^{m, \delta}(T) .
\]
Note that, by \Cref{lem:resampling}, a \dowal satisfying such a regret bound
in the simultaneous-move (SM) setting can be transformed into a \dowal for the delayed-label
(DL) setting that satisfies an analogous regret
bound, 
with an additional $O(\sqrt{T \log(T/\delta)}$ term on the right side, using $k$-fold
resampling with $k = O(T \log(T/\delta))$.

For the {\dowal} used in this paper, the relationship between the hypothesis class $\Ccal$
and the labeling function class $\Rcal$ is generally as follows. 
\begin{equation} \label{eq:rcal} 
\Rcal = \big\{ r : \Xcal \times \Ycal \to [-1,1] \; \mid \; r(x,y) = y - u(c(x)), \, c \in \Ccal, u : [-1,1] \to [0,1] \big\} .
\end{equation}
We will use the term ``\dowal for $\Ccal$'' to refer to a \dowal for a pair $(\Ccal,\Rcal)$ satisfying 
the relationship in Equation~\eqref{eq:rcal}.

Finally, we say a \dowal is \emph{proper} if its outputs $q_1,\ldots,q_T$ belong to the
class $\Ccal$. We may denote the output sequence of a proper \dowal by $c_1,\ldots,c_T$ rather
than $q_1,\ldots,q_T$ to emphasize that the outputs belong to the function class $\Ccal$.
\end{definition}

Using this oracle, we can describe our novel approach to distributional omniprediction.
The resulting algorithm, \Cref{alg:oracle-omni}, is an oracle-efficient implementation of our algorithmic framework, tailored to the distributional setting.
As before, we invoke the distributional OWAL with the multiaccuracy class $\C = \Delta \L \circ \H$, but here, we use a label class $\Rcal$ to encompass the residual functions $y - p_t(x)$.
Critically, to prove the \dowal guarantees, we need to ensure that the prediction functions $\{p_t\}$ live in some class of functions about which we can reason.

\begin{algorithm}[h!]
\caption{Oracle-Efficient Proper Calibration + Multiaccuracy Using \dowal}
\label{alg:oracle-omni}
\begin{algorithmic}[1]
\Procedure{OfflineProperCalMA}{$T$}
    \State \textbf{Input:} A dataset of $2T$ samples $\{(x_t,y_t)\}_{t=1}^{2T}$, drawn i.i.d.\ from $\D$. 
    \State \textbf{Input:} A failure probability parameter, $\delta>0$.
    \State Split the samples into two equal-sized sets, $D_{\mathrm{apcal}} = \{(x_t,y_t) \mid t \leq T\}$
    and $D_{\mathrm{dowal}} = \{(x_t,y_t) \mid t > T\}$.
    \State \textbf{Input:} Initialize \Cref{alg:eg_dowal_max} instance, \dowalsymbol, with $k$-fold resampling (\Cref{lem:resampling}) for \( k = O(T \log(T/\delta)) \), using function class $\Ccal$ and samples $D_{\mathrm{dowal}}$.
    \State \textbf{Input:} Initialize \Cref{alg:proper_recal} instance, \pcalalgsymbol, with failure probability $\delta$
    \State \textbf{Receive} $c_1$ from \dowalsymbol
    \For{$t = 1$ to $T$}
        \State \textbf{Send} $c_t$ to \pcalalgsymbol
        \State \textbf{Receive} $p_t$ from \pcalalgsymbol
        \State \textbf{Send} $(x_t,y_t) \in D_{\mathrm{apcal}}$ to \pcalalgsymbol
        \State \textbf{Send} labeling function $r_t (x,y) = y - p_t (x)$ to \dowalsymbol
        \State \textbf{Receive} $c_{t+1}$ from \dowalsymbol
    \EndFor
    \State \textbf{Output} randomized predictor $\hat \pb = \mathrm{unif}(p_1, \ldots, p_T)$.
\EndProcedure
\end{algorithmic}
\end{algorithm}

\begin{theorem}
\label{thm:offline_bv}
    \Cref{alg:oracle-omni} is an oracle-efficient algorithm that for any distribution $\D$ supported on $\X \times \{0,1\}$, any class of functions $\Ccal \subseteq \{ \Xcal \rightarrow [-1,1]\}$ that is closed under negation, learns an predictor that is approximately proper calibrated and $\Ccal$-Multiaccurate given $2T$ i.i.d.\ samples from $\D$. The algorithm has the following properties:
    \begin{itemize}
        \item It returns a randomized predictor that mixes over $T$ predictors, each represented as $p_t = v_t \circ q_t$, where $v_t : [-1,1] \to [1/T]$ is a monotone postprocessing function and $q_t \in \mathrm{conv}_k(\Ccal)$ is a convex combination of at most $k = O\left(T\log(T/\delta) \right)$ functions in $\Ccal$.
        \item It is oracle-efficient, making $O(T^{5/2})$ calls to an \emph{offline} ERM oracle for $C$.
        \item The proper calibration error of the predictor scales as $O \left(\sqrt{\ln (T / \delta) / T} \right)$.
        \item The $\Ccal$-multiaccuracy error of the predictor scale with the \emph{offline} Rademacher complexity of a derived class of $\Ccal$: $\eps(T) = O \left(\mathsf{rad}_T(\Ccal) + \mathsf{rad}_T(\Th^{1/T} \circ \mathrm{conv}_k (\Ccal)) \right) + O \left(\sqrt{\ln (T / \delta) / T} \right)$.
    \end{itemize}
\end{theorem}
Substituting $\Ccal$ for $\{+,-\} \cdot \Th \circ \Hcal$ gives the result in \Cref{res:offline_bv}. Substituting $\Ccal$ for $\{+,-\} \cdot \Delta \Lcal \circ \Hcal$ gives oracle-efficient omniprediction for any class of loss functions.

We begin with a high-level proof that invokes technical lemmas proved subsequently.
We justify the claims on representation complexity, oracle-efficiency, and statistical complexity.
\begin{proof}[Proof of \Cref{thm:offline_bv}]
    The claim about the omnipredictor's representation complexity is justified by the following observations.
    \begin{enumerate}
    \item By \Cref{lem:eg_dowal_max}, \dowalsymbol is a proper \dowal, producing functions $c_t$ in $\Ccal.$ After resampling (\Cref{lem:resampling}), the returned function $q_t \in \mathrm{conv}_k(\Ccal)$ is a convex combination of at most $k = \log(T/\delta)/\eps^2$ functions in $\Ccal$.
    \item By \Cref{lem:proper-recal}, \pcalalgsymbol returns predictors $p_t$ in $U \circ \mathrm{conv}_k (\Ccal)$, where $U$ is the class of all monotone functions 
    from $[-1,1]$ to $[1/T]$. In other words, $p_t = v_t \circ q_t$, where $v_t : [-1,1] \to [1/T]$ is a monotone postprocessing function and $q_t \in \mathrm{conv}_k(\Ccal)$
    \end{enumerate}
    The claim about oracle-efficiency is justified by the following observations.
    \begin{enumerate}
    \item \Cref{alg:eg_dowal_max} is called $T$ times. Each call involves solving an entropy-regularized ERM problem using
    \Cref{alg:frank-wolfe} with parameters $m=T$ and $\eps = \eta = T^{-1/2}.$
    \item By \Cref{lem:frank-wolfe}, each call to \Cref{alg:frank-wolfe} makes $O(m / \eps) = O(T^{3/2})$ calls to the
    ERM oracle for $\Ccal$.
\end{enumerate}

    The bound $\pcal_{\D}(\hat \pb) = O(\sqrt{T \log T / \delta})$ is a direct consequence of 
    \Cref{lem:proper-recal}. Note that the splitting of the sample set into 
    $D_{\mathrm{apcal}}$ and $D_{\mathrm{dowal}}$ is vital here. From the standpoint
    of the notation in \Cref{response:apcal}, the samples $\{(x_t,y_t)\}_{t=T+1}^{2T}$
    that constitute $D_{\mathrm{dowal}}$ are part of the adversary's random seed
    $r_{\mathrm{adv}}$. Thus, the fact that $q_t$ may depend on all of the
    samples in $D_{\mathrm{dowal}}$ is consistent with the online restrictions
    placed on the \pcalalg and its adversary by \Cref{lem:proper-recal}. 

    The bound on $\Ccal$-multiaccuracy error of $\hat \pb$ will follow by quantifying
    the oracle regret of the \dowal implemented by \Cref{alg:eg_dowal_max}. Recall from
    \Cref{lem:eg_dowal_max} that the algorithm with $m$ samples and failure probability
    $\delta$ satisfies the oracle regret bound
    \[
        \dowalreg_{\Ccal,\Rcal}^{m,\delta}(T) = O \left( \sqrt{T \log 1/\delta} \right) + O(T \cdot \mathsf{rad}_T(\Ccal \cdot \Rcal) ).
    \]
    Below, in \Cref{lem:rad-postprocess} we prove the Rademacher complexity bound
    \begin{equation} \label{eq:rad-thr}
        \mathsf{rad}_T (\Ccal \cdot \Rcal) \in O\left(\mathsf{rad}_T(\Ccal) + \mathsf{rad}_T(\Th^{1/T} \circ \mathrm{conv}_k (\Ccal)) \right) .
    \end{equation}
    Hence, 
    \begin{align*}
        \Ccal\maerr_{\D}(\hat \pb) & = 
        \max_{c \in \Ccal} \left\{ \frac1T \sum_{t=1}^T \E_\Dcal [c(x) (y - p_t(x))] \right\} \\
        & = \max_{c \in \Ccal} \left\{ \frac1T \sum_{t=1}^T \E_\Dcal [c(x) r_t(x,y)] \right\}
        & \text{by our choice of $r_t$} \\
        & \leq \frac1T \sum_{t=1}^T \E_\Dcal [c_t(x) r_t(x,y)]  \; + \; 
        O\left( \sqrt{\tfrac{\ln (1/\delta)}{T}} \right) + O(\mathsf{rad}_m(\Ccal \cdot \Rcal)) 
        & \text{by \Cref{lem:eg_dowal_max}} \\
        & = \frac1T \sum_{t=1}^T \E_\Dcal [c_t(x) (y - p_t(x))]  \; + \; 
        O\left( \sqrt{\tfrac{\ln (1/\delta)}{T}} \right) + O(\mathsf{rad}_m(\Ccal \cdot \Rcal)) \\
        & \leq
        O\left( \sqrt{\tfrac{\ln (T/\delta)}{T}} \right) + O \left(\mathsf{rad}_T(\Ccal) + \mathsf{rad}_T(\Th^{1/T} \circ \mathrm{conv}_k (\Ccal)) \right)
        & \text{by \Cref{lem:proper-recal} and Eq.~\eqref{eq:rad-thr}} 
    \end{align*}
    which justifies the stated error bound $\eps(T)$.
\end{proof}

Next, we prove \Cref{lem:rad-postprocess}, which establishes the generalization bounds invoked in the proof of \Cref{thm:offline_bv}.
This generalization bound, while relatively short to prove, is a critical and delicate part of our argument.

\begin{lemma}\label{lem:rad-postprocess}
Let $\Ccal \subseteq \{ c:\Xcal \rightarrow [-1,1] \}$ be a class of functions. Let
$U$ denote the class of all monotone functions $[-1,1] \to [1/T]$.
Let $\Rcal$ be the class of bivariate functions $r(x,y) = y - u(c(x))$ for some $u \in U, c \in \Ccal$. Then
\[
\mathsf{rad}_m (\Ccal \cdot \Rcal) \in  O \left(\mathsf{rad}_m(\Ccal) + \mathsf{rad}_m(\Th^{1/T} \circ \Ccal) \right)
\]
\end{lemma}

\begin{proof}[Proof of \Cref{lem:rad-postprocess}]
Fix a sample $S = \{(x_1, y_1), \dots, (x_m, y_m)\}$. By definition of empirical Rademacher complexity,
\[
\mathsf{rad}_m(\Ccal \cdot \Rcal) = \mathbb{E}_\epsilon \left[ \sup_{c, c' \in \Ccal, u \in U} \frac{1}{m} \sum_{i=1}^m \epsilon_i \cdot c'(x_i) \cdot (y_i - u(c(x_i))) \right].
\]

Every monotone function $u: [-1,1] \to [1/T]$ can be written as a convex combination of threshold functions in $\Th^{1/T}$, where
\[
\Th^{1/T} = \left\{ \mathbf{1}\{ z \leq t \} : t \in \{-1, -1 + 1/T, \dots, 1\} \right\}.
\]
Therefore, for any $c \in \Ccal$, we have $u \circ c(x) = \sum_j \alpha_j \cdot \mathbf{1}\{c(x) \leq t_j\}$ with $\sum_j \alpha_j \leq 1$, implying $u \circ c \in \text{conv}(\Th^{1/T}) \circ \Ccal$. Since $\mathsf{rad}_m(\text{conv}(F)) = \mathsf{rad}_m(F)$,
\[
\mathsf{rad}_m(U \circ \Ccal) = \mathsf{rad}_m(\Th^{1/T} \circ \Ccal).
\]

Define $f(x) = \left( c'(x), u(c(x)) \right) \in [-1,1]^2$, and define $h(a,b) = a \cdot (y - b)$. $h$ is $O(1)$-Lipschitz over $[-1,1]^2$. Let $\Fcal = \{ f(x) = (c'(x), u(c(x))) : c, c' \in \Ccal, u \in U \}$. Applying Corollary 1 of \cite{maurer2016},
\[
\mathbb{E}_\epsilon \sup_{f \in \Fcal} \sum_{i=1}^m \epsilon_i h(f(x_i)) \leq O(1) \cdot \mathbb{E}_{\epsilon_{ij}} \sup_{f \in \Fcal} \sum_{i=1}^m \sum_{j=1}^2 \epsilon_{ij} f_j(x_i).
\]
But $f_1(x_i) = c'(x_i)$ and $f_2(x_i) = u(c(x_i))$, so this becomes
\[
O(1) \cdot \left( \mathsf{rad}_m(\Ccal) + \mathsf{rad}_m(U \circ \Ccal) \right) \leq O(1) \cdot \left( \mathsf{rad}_m(\Ccal) + \mathsf{rad}_m(\Th^{1/T} \circ \Ccal) \right).
\]
We obtain
\[
\mathsf{rad}_m(\Ccal \cdot \Rcal) \in O \left( \mathsf{rad}_m(\Ccal) + \mathsf{rad}_m(\Th^{1/T} \circ \Ccal) \right). \qedhere
\]
\end{proof}

\begin{corollary}[Corollary 1 of \cite{maurer2016}]
Let $\mathcal{X}$ be any set, $(x_1, \ldots, x_n) \in \mathcal{X}^n$, let $\mathcal{F}$ be a class of functions $f : \mathcal{X} \to \ell_2$ and let $h_i : \ell_2 \to \mathbb{R}$ have Lipschitz norm $L$. Then
\[
\mathbb{E} \sup_{f \in \mathcal{F}} \sum_i \epsilon_i h_i(f(x_i)) 
\leq \sqrt{2} L \, \mathbb{E} \sup_{f \in \mathcal{F}} \sum_{i,k} \epsilon_{ik} f_k(x_i),
\]
where $\epsilon_{ik}$ is an independent doubly indexed Rademacher sequence and $f_k(x_i)$ is the $k$-th component of $f(x_i)$.

All random variables are assumed to be defined on some probability space $(\Omega, \Sigma)$. We use $\ell_2$ to denote the Hilbert space of square summable sequences of real numbers. The norm on $\ell_2$ and the Euclidean norm on $\mathbb{R}^K$ are denoted with $\|\cdot\|$.
\end{corollary}

\subsection{Implementing a \dowal using FTRL}
\label{sec:impl_dowal}

This subsection presents and analyzes a \dowal that makes use of a 
subroutine called an \emph{entropy-regularized ERM oracle over $\Ccal$}, defined as follows. 

\begin{definition} \label{def:reg-erm}
An entropy-regularized ERM oracle is initialized with a class of functions 
$\Ccal : \X \to [-1,1]$ and an approximation parameter $\varepsilon$. 
The oracle takes, as input, a set of pairs 
$(x_1,w_1),\ldots,(x_m,w_m) \in \X \times \R$. It outputs an explicit convex combination of elements of $\Ccal$,
i.e.~a sequence of hypotheses $c_1,\ldots,c_k \in \Ccal$ and coefficients $\alpha_1,\ldots,\alpha_k \in [0,1]$
summing up to 1, such that the weighted average $c = \sum_{i=1}^k \alpha_i c_i$ minimizes (within $\eps$) the function
$\sum_{i=1}^{m} c(x_i) w_i \, + \, \sum_{i=1}^m (c(x_i)+2) \log(c(x_i)+2).$
\end{definition}
We use $c(x_i)+2$ rather than $c(x_i)$ in the regularizer because $c(x_i)$ takes values in $[-1,1]$, 
whereas $\log(z)$ is only defined when $z$ is strictly
positive.

In \Cref{sec:frank-wolfe}
below, we show how to use the Frank-Wolfe method to 
implement an $\eps$-approximate regularized ERM oracle using $\poly(m)$ calls to a standard ERM 
oracle. 

\paragraph{Overview of \Cref{alg:eg_dowal_max}}
The algorithm implements a \dowal using the Follow The Regularized Leader (FTRL) approach.
At each timestep \(t\), we call the entropy-regularized ERM oracle using a dataset
that labels the points $x_1,\ldots,x_m \in \X$ with labels $w_1,\ldots,w_m$
derived from the past adversarial functions \(\{r_1,\dots,r_{t-1}\}\). The oracle finds an 
approximate minimizer of the cumulative regularized empirical loss. It outputs this approximate
minimizer represented as a convex combination of elements of $\Ccal$. To select the predictor \(c_t\) at each step, 
we sample a single function from \(\Ccal\) according to the convex combination’s weights. 
This guarantees that \(c_t \in \Ccal\), i.e.~the algorithm is a \emph{proper} \dowal.

\begin{algorithm}[ht]
\caption{\dowallong using Exponentiated Gradient}
\label{alg:eg_dowal_max}
\begin{algorithmic}[1]
\Require i.i.d.\ samples \(S = \{(x_i,y_i)\}_{i=1}^m \sim \Dcal^m\), time horizon \(T\), failure probability \(\delta\)
\Ensure Sequence of predictors \(\{c_t\}_{t=1}^T\) where each \(c_t \in \Ccal\)
\State Set \(\eta \gets \sqrt{1 / T}\) 
\State Initialize distribution over \(\Ccal\) by calling the entropy-regularized ERM oracle (in \Cref{alg:frank-wolfe}):
\[
\tilde{c}_1 \gets {\argmin}_{c \in \Ccal}^{\eta} \left\{ \sum_{i=1}^m (c(x_i) + 2) \log (c(x_i) + 2) \right\}.
\] \label{algline:c1init}
Obtain from the oracle the convex combination $\{c_j,\alpha_j\}$ representing $\tilde{c}_1$.
\For{$t = 1$ to $T$}
    \State Let \(\{ (c_j, \alpha_j) \}\) be the current convex combination representing the \(\tilde{c}_{t}\).
    \State Draw \(c_t \in \Ccal\) by sampling according to weights \(\{\alpha_j\}\).
    \State Output \(c_t\).
    \State Receive adversary function \(r_t \in \Rcal\).
    \State Update the convex combination by calling the entropy-regularized ERM oracle (in \Cref{alg:frank-wolfe}):
    \[
    \tilde{c}_{t+1} \in {{\argmin}_{c \in \mathrm{conv}(\Ccal)}^{\eta}} \left\{- \eta \sum_{s=1}^{t} \sum_{i=1}^m c(x_i) r_s(x_i,y_i) \;+\; 
    \sum_{i=1}^m (c(x_i) + 2) \log (c(x_i) + 2) \right\}
    \] \label{algline:argmin} 
    \State Obtain from the oracle the updated convex combination \(\{ (c_j', \alpha_j') \}\) representing \(\tilde{c}_{t+1}\).
\EndFor
\State \textbf{return:} Sequence of predictors \(\{c_t\}_{t=1}^T\)
\end{algorithmic}
\end{algorithm}
Lines~\ref{algline:c1init} and~\ref{algline:argmin} of the algorithm use the notation $\argmin^{\eta}$. This 
denotes the set of all points where the indicated function attains a value within $\eta$ of its global minimum.

\label{sec:eg_dowal}
\begin{lemma}[Exponentiated Gradient \dowal]\label{lem:eg_dowal_max}
Given $m$ i.i.d samples from $\Dcal$, an entropy-regularized ERM oracle over $\Ccal$ (implemented in \Cref{alg:frank-wolfe}) and failure probability $\delta > 0$, \Cref{alg:eg_dowal_max} implements a high probability distributional \dowal using a dataset of size $m$ and guarantees the following with probability at least $1-\delta$:
\[
\max_{c \in \Ccal} \sum_{t=1}^T \E_\Dcal [c(x) r_t(x,y)] \leq \sum_{t=1}^T \E_\Dcal [c_t(x) r_t(x,y)] + O\left(\sqrt{T \ln {1/\delta}}\right) + O(T \cdot \mathsf{rad}_m(\Ccal \cdot \Rcal))
\] 
where 
$\Ccal \cdot \Rcal$ denotes the set of all functions on $\Xcal \times \Ycal$ of the form $g(x,y) = c(x) r(x,y)$ for
$c \in \Ccal, \, r \in \Rcal$.
The algorithm runs in time $O(m)$ per timestep, making $O(1)$ calls to the regularized ERM oracle. Moreover, it implements
a \emph{proper} \dowal: each $c_t$ chosen by the algorithm is in $\Ccal$.
\end{lemma}

\begin{proof}[Proof of \Cref{lem:eg_dowal_max}]
The proof follows the regret analysis of Follow-the-Regularized-Leader (FTRL) under entropy regularization combined with standard Rademacher generalization arguments, applied with high probability. 

For each $c \in \Ccal$ define
$
    u(c) = 
    (c(x_1), \, c(x_2), \, \ldots, \, c(x_m))       
$
and let $\mathcal{K}$ denote the convex hull of the vectors $\{u(c) \, \mid \, c \in \Ccal \}.$
Define the regularizer 
\[
\psi (u) = \sum_{i =1}^m (u_i + 2) \log (u_i + 2)
\]
and the regularized objective function
\[
    g_t(u) = - \eta \sum_{s=1}^{t-1} \sum_{i=1}^m u_i r_s(x_i,y_i) + \psi(u)
\]
and let 
\[ 
  u_t \in \argmin_{u \in \mathcal{K}} \left\{ g_t(u) \right\} . 
  \]    
Let $\nabla_t \in \R^m$ denote the vector whose 
$i^{\mathrm{th}}$ coordinate is $- r_t(x_i,y_i)$. 
The standard analysis of FTRL, \Cref{lem:ftrl}, shows that 
for all $u^* \in  \mathcal{K}$:
\begin{equation} \label{eq:ftrl-bound-for-u}
\sum_{t=1}^T \left\langle \nabla_t, u_t - u^* \right\rangle \leq 2 \eta \sum_{t=1}^T {\left\| \nabla_t \right\|^*_t}^2 \; + \; \frac{\psi(u^*) - \psi(u_1)}{\eta} .
\end{equation}
Here, the norm $\| \cdot \|^*_t$ denotes the dual norm defined by the regularizer $\psi$ at $u_t$, 
i.e.~$\| \nabla_t \|^*_t = \| \nabla_t \|_{\nabla^{-2} \psi(u_t)}$. For the entropy regularizer
$\psi(u)$, the inverse Hessian matrix $\nabla^2 \psi(u)$ is a diagonal matrix whose $i^{\mathrm{th}}$
diagonal entry is $\frac{1}{u_i+2}.$ Since $\mathcal{K} \subseteq [-1,1]^m$, we have 
$1 \leq u_i + 2 \leq 3$ for any $u \in \mathcal{K}$. This implies 
\begin{equation} \label{eq:dual-norm}
    \| \nabla_t \|_{\nabla^{-2} \psi(u_t)}^2 \leq 3 \| \nabla_t \|_2^2 \leq 3 m,
\end{equation}
where the last inequality follows because each of the $m$ coordinates of $\nabla_t$
belongs to the interval $[-1,1]$. 

The regularizer $\psi$ attains values between 0 and $m \log(3)$, so 
$\frac{\psi(u^*) - \psi(u_1)}{\eta} \leq \frac{m \log(3)}{\eta}$. Substituting this
bound and the bound~\eqref{eq:dual-norm} into~\eqref{eq:ftrl-bound-for-u}, we obtain
\begin{equation}
    \label{eq:ftrl-regret}
    \sum_{t=1}^T \left\langle \nabla_t, \, u_t \right\rangle \; - \; 
    \sum_{t=1}^T \left\langle \nabla_t, \, u^* \right\rangle \leq 6 \eta m T + \tfrac{m \log(3)}{\eta}
    \leq 8 m \sqrt{T}
\end{equation}
by our choice of $\eta = \sqrt{1/T}.$

Next we bound the regret of playing an $\eta$-approximate minimizer of $g_t(u)$ in each step, 
rather than the exact minimizer. Observe that $g_t(u)$ is a linear function of $u$ plus $\psi(u)$, so 
$\nabla^2 g_t(u) = \nabla^2 \psi(u)$. Earlier we calculated that for all $u \in \mathcal{K}$,
the Hessian matrix $\nabla^2 \psi(u)$ is a diagonal matrix with entries between $\frac13$ and $1$, hence
$g_t$ is $\left( \frac{1}{3} \right)$-strongly convex. Since $u_t$ is the global minimizer of 
$g_t$, it follows from strong convexity that 
\begin{align}
    \nonumber
    \tfrac{1}{6} \| u(\tilde{c}_t) - u_t \|_2^2 & \leq 
    g_t(u(\tilde{c}_t) - g_t(u_t) \leq \eta  \\
    \nonumber
    \| u(\tilde{c}_t) - u_t \|_2^2 & \leq 6 \eta \\
    \nonumber
    \left\langle \nabla_t , \, u(\tilde{c}_t) - u_t \right\rangle & \leq 
    \| \nabla_t \|_2 \| u(\tilde{c}_t) - u_t \|_2 \leq \sqrt{6 \eta m} \\
    \label{eq:ftrl-apx-minimizer}
    \sum_{t=1}^T \left\langle \nabla_t , \, u(\tilde{c}_t) \right\rangle \; - \; 
    \sum_{t=1}^T \left\langle \nabla_t , \, u_t \right\rangle & \leq 
    \sqrt{6 \eta m} \cdot T = \sqrt{6 m T}  .
\end{align}

\sloppy Now we bound the regret of playing the sampled $c_t$ instead of $\tilde{c}_t$. We'll do this by showing that 
$\sum_{t=1}^T \frac{1}{m} \sum_{i = 1}^m c_t(x_i) r_t(x_i, y_i)$ converges to $\sum_{t=1}^T \frac{1}{m} \sum_{i = 1}^m \tilde{c}_t(x_i) r_t(x_i, y_i)$ using Azuma Hoeffding. Define the martingale difference sequence $Z_t = \frac{1}{m} \sum_{i = 1}^m c_t(x_i) r_t(x_i, y_i) - \frac{1}{m} \sum_{i = 1}^m \tilde{c}_t(x_i) r_t(x_i, y_i)$. This forms a bounded martingale difference sequence because the $c_t$ is chosen only based on information from the past and $\tilde{c}_t = \E [c_t]$. Thus, applying Azuma Hoeffding, we have that with probability $1 - \delta/2$,
\begin{equation} \label{eq:ftrl-azuma}
\left| \sum_{t=1}^T \frac{1}{m} \sum_{i = 1}^m c_t(x_i) r_t(x_i, y_i) - \sum_{t=1}^T \frac{1}{m} \sum_{i = 1}^m \tilde{c}_t(x_i) r_t(x_i, y_i) \right| \leq O \left(\sqrt{T \ln 1/\delta} \right)
\end{equation}
Let us define the empirical gain for any hypothesis $c \in \Ccal$ at time $t$:
\[
\hat{G}_t(c) = \frac{1}{m}\sum_{i=1}^m c(x_i)r_t(x_i, y_i) 
\]
and let $\hat{c}^* = \argmax_{c \in \Ccal} \sum_{t=1}^T \hat{G}_t(c)$
and $u^* = u(\hat{c}^*)$.
With probability at least $1 - \delta/2$, the 
following regret bound holds:
\begin{align*}
  \sum_{t=1}^T \hat{G}_t(c_t) & \geq 
  \sum_{t=1}^T \hat{G}_t(\tilde{c}_t) \; - \; O \left(\sqrt{T \ln 1/\delta} \right)
  & \text{by~\eqref{eq:ftrl-azuma}} \\
  & = - \frac1m \sum_{t=1}^T \left\langle \nabla_t, \, u(\tilde{c}_t) \right\rangle \; - \; O \left(\sqrt{T \ln 1/\delta} \right) \\
  & \geq - \frac1m \sum_{t=1}^T \left\langle \nabla_t, \, u_t \right\rangle \; - \; O \left(\sqrt{T / m } \, +  \,\sqrt{T \ln 1/\delta} \right)
  & \text{by~\eqref{eq:ftrl-apx-minimizer}} \\
  & \geq - \frac1m \sum_{t=1}^T \left\langle \nabla_t, \, u^* \right\rangle \; - \; O \left(\sqrt{T} \, + \, \sqrt{T / m } \, +  \,\sqrt{T \ln 1/\delta} \right)
  & \text{for any $u^* \in \mathcal{K}$, by~\eqref{eq:ftrl-regret}} \\
  & \geq - \frac1m \sum_{t=1}^T \left\langle \nabla_t, \, u^* \right\rangle \; - \; O \left(\sqrt{T \ln 1/\delta} \right)
  & \text{simplifying $O(\cdot)$ expression} \\
  & = \sum_{t=1}^T \hat{G}_t(\hat{c}^*) \; - \; O \left(\sqrt{T} \ln 1/\delta \right) .
\end{align*}
Now consider the population gain:
\[
G_t(c) = \E_{(x,y)\sim\Dcal}[c(x)r_t(x,y)]
\]
The population regret of playing the sampled $c_t$ can be decomposed as:
\[
\max_{c \in \Ccal} \sum_{t=1}^T G_t(c) - \sum_{t=1}^T G_t(c_t) \leq \underbrace{\max_{c \in \Ccal} \sum_{t=1}^T \hat{G}_t(c) - \sum_{t=1}^T \hat{G}_t(c_t)}_{\text{FTRL regret}} + \underbrace{\sum_{t=1}^T (\hat{G}_t(c_t) - G_t(c_t)) + \sum_{t=1}^T (G_t(c^*) - \hat{G}_t(c^*))}_{\text{Generalization error}},
\]
where $c^* = \argmax_{c \in \Ccal} \sum_{t=1}^T G_t(c)$.

Using standard uniform convergence bounds based on Rademacher complexity, we have that with probability $1 - \delta/2$, for any $c \in \Ccal, r \in \Rcal$:
\begin{equation} 
\left| \frac{1}{m} \sum_{i = 1}^m c(x_i) r(x_i, y_i) -  \E_{(x,y) \sim \Dcal} [c(x) r(x, y)] \right| \leq O \left(\sqrt{\frac{ \ln 1/\delta}{m}} \right) + O(\mathsf{rad}_m(\Ccal \cdot \Rcal))
\end{equation}
Therefore,
\[
\left| (G_t(c) - \hat{G}_t(c)) \right| \leq O \left(\sqrt{\frac{ \ln 1/\delta}{m}} \right) + O(\mathsf{rad}_m(\Ccal \cdot \Rcal)),
\]
where $\mathsf{rad}_m(\Ccal \cdot \Rcal)$ is the Rademacher complexity of the product class $\Ccal \cdot \Rcal$.

Combining the bounds for the regret bounds and the generalization error, we have that with probability at least $1 - \delta$:
\[
\max_{c \in \Ccal} \sum_{t=1}^T G_t(c) - \sum_{t=1}^T G_t(c_t) \leq O \left(\sqrt{T \ln 1/\delta} \right) + O(T \cdot \mathsf{rad}_m(\Ccal \cdot \Rcal)).
\]
\end{proof}

\begin{lemma}[\cite{hazan2023oco}]\label{lem:ftrl}
Given a closed, convex decision set $\Kcal$ and a strongly convex, smooth  and twice differentiable regularizer $R : \Kcal \rightarrow \R$. For every $u \in \Kcal$, the Follow the Regularized Leader algorithm attains the following regret bound
\[
\mathrm{Reg}(T) \leq 2 \eta \sum_{t=1}^T {\|\nabla_t\|^*_t}^2 + \frac{R(u) - R(u_1)}{\eta}
\]
where $u_t \in \Kcal$ is the action at timestep $t$,
$\nabla_t$ is the gradient at $u_t$ of the loss function at time $t$,
and $\mbox{\( \|\cdot\|^*_t \)} = \|\cdot\|_{\nabla^{-2} R(u_t)}$ 
represents the dual norm defined by the regularizer $R$ at $u_t$.
\end{lemma}

\subsection{Frank-Wolfe reduction to ERM oracle}
\label{sec:frank-wolfe}

In this section we explain how to implement a regularized ERM oracle over $\Ccal$ by
a sequence of calls to a standard ERM oracle over $\Ccal$. The algorithm for the regularized
ERM oracle is presented below as \Cref{alg:frank-wolfe}. Its analysis is summarized by the
following lemma.

\begin{lemma}[Frank-Wolfe with explicit convex combination]
\label{lem:frank-wolfe}
Given a dataset \(\{(x_i, y_i)\}_{i=1}^m\), a class of functions \(\Ccal \subseteq \{ c : \Xcal \to [-1,1] \}\), a ERM oracle for \(\Ccal\), and parameters \(\eta, \epsilon > 0\), \Cref{alg:frank-wolfe} returns an \(\epsilon\)-approximate solution \(c^*\) to the entropy-regularized ERM problem
\[
 {\argmin_{c \in \Ccal}}^\eps \Biggl\{- \eta \sum_{i=1}^m c(x_i)y_i + \sum_{i=1}^m (c(x_i) + 2) \log (c(x_i) + 2) \Biggr\}
\]
after \(O(m/\epsilon)\) iterations. Moreover, the returned \(c^*\) is a convex combination of functions in \(\Ccal\), i.e., \(c^* = \sum_{j} \alpha_j c_j\) with \(\sum_j \alpha_j = 1\) and \(\alpha_j \geq 0\), and the algorithm also provides the weights \(\{\alpha_j\}\).
\end{lemma}

\paragraph{Overview of \Cref{alg:frank-wolfe}:}

Each iteration of \Cref{alg:eg_dowal_max} requires (approximately) solving a 
constrained concave minimization problem $\min \{ g(u) \, \mid \, u \in \mathcal{K} \}$
where 
\begin{equation} \label{eq:obj-frank-wolfe} 
    g(u) = - \eta \sum_{s=1}^{t} \sum_{i=1}^m r_s(x_i,y_i) u_i \;+\; \sum_{i=1}^m (u_i + 2) \log (u_i + 2) 
\end{equation}
and, as in the proof of~\eqref{lem:eg_dowal_max}, $\mathcal{K}$ denotes the convex hull of the set of vectors 
$
    u(c) = 
    (c(x_1), \, c(x_2), \, \ldots, \, c(x_m))       
$
as $c$ ranges over $\Ccal$.
In this section we assume we are given a \emph{ERM oracle} for $\Ccal$, that is,
an algorithm for selecting the $c \in \Ccal$ that minimizes $\sum_{i=1}^m w_i c(x_i)$ for a 
given set of $(x_i,w_i)$ pairs. This is equivalent to an oracle for minimizing linear functions
over $\mathcal{K}$. \Cref{alg:frank-wolfe} below uses the ERM oracle to implement
the Frank-Wolfe method, also known as conditional gradient descent, for approximately minimizing 
the convex function $g(u)$ over $\mathcal{K}$.
At each iteration, the algorithm computes the gradient of the objective function, invokes the ERM oracle to find an extreme point in the original function class \(\Ccal\), and then updates the current solution by forming a convex combination. After \( O \big( \frac{m}{\eps} \big) \) iterations, it returns an \(\eps\)-approximate solution to the original problem, expressed as a convex combination of functions from \(\Ccal\), together with the associated weights.
Since \Cref{alg:eg_dowal_max} calls \Cref{alg:frank-wolfe} using the parameters
$m=T, \, \eps = \eta = \sqrt{1/T}$, we see that each such call requires $O(T^{3/2})$
iterations, hence $O(T^{3/2})$ calls to the ERM oracle.

\begin{algorithm}
\caption{Frank-Wolfe for Entropy Regularized ERM with Explicit Convex Combination}
\label{alg:frank-wolfe}
\begin{algorithmic}[1]
\Procedure{FrankWolfe}{$\{(x_i, y_i)\}_{i=1}^m, \Ccal, \eta, \eps$}
    \State Initialize $c_1$ to an arbitrary function in $\Ccal$
    \State Let $\alpha_1 = 1$ and let $\alpha_j = 0$ for all $j > 1$
    \For{$t = 1, 2, \ldots, T$}
        \State For each $i \in [m]$, set $w_i = - \eta y_i + \log (c_t(x_i) + 2) + 1$
        \State Call the ERM oracle on $\{(x_i, w_i)\}_{i=1}^m$ to obtain $s_t \in \Ccal$
        \State Set $\gamma_t = \frac{2}{t+1}$
        \State Update $c_{t+1} = (1-\gamma_t) c_t + \gamma_t s_t$
        \State Update the weights of the convex combination:
        \[
        \text{For } j < t: \alpha_j \leftarrow (1-\gamma_t) \alpha_j, \quad
        \alpha_t \leftarrow \alpha_t + \gamma_t
        \]
        \State Since $c_1$ was chosen initially, we now have a convex combination:
        \[
        c_{t+1} = \sum_{j=1}^{t} \alpha_j s_j + \biggl(1 - \sum_{j=1}^t \alpha_j\biggr) c_1.
        \]
        (Initially, $\alpha_1 = 1$ and no $s_j$ are chosen, so $c_1$ is the starting point. After each iteration, we redistribute weights accordingly. Note that the $\alpha_j$ for $j>1$ were initially zero and only become nonzero when their corresponding $s_j$ appears.)
    \EndFor
    \State \textbf{return} $c_1, s_1, \ldots, s_T$ and the final weights $(\alpha_1, \alpha_2, \ldots, \alpha_T)$
\EndProcedure
\end{algorithmic}
\end{algorithm}

\begin{lemma}[Conditional Gradient Descent; \citep{hazan2023oco}]\label{lem:cgd}
Let $K \subset \R^n$ with bounded $\ell_2$ diameter $R$. Let $f$ be a $\beta$-smooth function on $K$, then the sequence of points $x_t \in K$ computed by the conditional gradient descent algorithm satisfies
\[
f(x_t) - f(x^*) \leq \frac{2 \beta R^2}{t + 1}
\]
for all $t \geq 2$ where $x^* \in \argmin_{x \in K} f(x)$
\end{lemma}

\begin{proof}[Proof of \Cref{lem:frank-wolfe}]
The objective function $g : \mathcal{K} \to \R$ defined in Equation~\eqref{eq:obj-frank-wolfe} 
is well-defined and differentiable on $\mathcal{K}$.
Its gradient is
\[
\nabla g(u) = \begin{bmatrix}
- \eta z_1 + \log (u_1 + 2) + 1 \\
- \eta z_2 + \log (u_2 + 2) + 1 \\
\vdots \\
- \eta z_m + \log (u_m + 2) + 1
\end{bmatrix}
\]
where $z_i$ denotes the sum $z_i = \sum_{s=1}^t r_s(x_i,y_i)$.

Now we compute the smoothness parameter and the diameter of the domain.
Over the domain \([-1, 1]\), the second derivative of \( (u_i+2) \log(u_i + 2) \) with respect to \( u_i \) is at most \( 1 \). Thus, \( g \) is \(\beta\)-smooth with \(\beta = 2 \). The \(\ell_2\)-diameter of \( K \) is at most \( R = 2 \sqrt{m} \).
Applying \Cref{lem:cgd} to \( g \) on \( K \):
\[
g(u_T) - g(u^*) \leq \frac{2 \beta R^2}{T+1} \leq \frac{16 m}{T+1}.
\]
Choosing \( T \geq \frac{16 m}{\eps} \) 
leads to $g(u_T) - g(u^*) < \eps$.
\end{proof}

\clearpage
\bibliography{ref.bib}
\clearpage
\appendix

\section{Relating Proper Calibration to Existing Notions of Calibration}
\label{app:comparisons}

In this appendix, we document how Proper Calibration compares to prior notions of calibration.
The standard notion of calibration, $\ell_1$-calibration, being the strongest notion generally considered, implies proper calibration.
While other notions previously considered do not imply proper calibration, proper calibration implies a number of existing notions such as U-calibration, smooth calibration, and $\ell_\infty$-calibration. 

The one exception is the recent notion of Calibration Decision Loss (CDL) introduced by \citet{hu2024calibrationerrordecisionmaking} to guarantee swap regret simultaneously for decision-making.
We show that CDL is incomparable to bounded proper calibration, and thus, incomparable to Decision OI.
In particular, CDL cannot be used in the existing Loss OI framework to guarantee omniprediction, but it also need not be satisfied by our omnipredictors.

We will simplify notation by dropping the contexts $x_t$, that is, we will only consider predictions $p_t = p_t(x_t)$ and outcomes $y_t$.

\subsection{Calibration is stronger than Proper Calibration}

Recall, one way to define $\ell_1$-calibration is as weighted calibration according to \emph{all} weight functions.
\begin{definition*}[$\ell_1$-calibration]
For a sequence of $T$ predictions $\pb$ and outcomes $\yb$,
\[
    \ell_1\calerr (\pb,\yb) = \left| \sup_{w:[0,1] \to [-1,1]} \sum_{t=1}^T w(p_t) (y_t - p_t)\right|.
\]
\end{definition*}

From this definition, we immediately conclude that the $\ell_1$-calibration error upper bounds the proper calibration error, as $\wprop \subseteq \{w:[0,1]\to[-1,1]\}$ is a strict subset of all functions from $[0,1]$ to $[-1,1]$.

\paragraph{Proper Calibration does not imply $\ell_1$-calibration.}
Here we argue that proper calibration is weaker than $\ell_1$-calibration by exhibiting a sequence of predictions, where the proper calibration error is bounded by a constant but the $\ell_1$-calibration error grows asymptotically with $T$.

\begin{example}\label{ex:ell1}
    Consider a sequence of $T$ outcomes $\yb$ that come in $m = 2k$ epochs of equal length, for an even $k$.
    For each $i \in \{1,2,\hdots,2k\}$, in the $i$th epoch the prediction for each time period $p_t = i/m$.
    In the first $k$ epochs, where $i \in \{1,\hdots,k\}$, if $i$ is even, the outcomes are all $0$; if $i$ is odd, then an $i/k$ fraction are $1$.
    In the second $k$ epochs where $i \in \{k+1,\hdots,2k\}$, if $i$ is even, then a $(i-k)/k$ fraction are $1$; if $i$ is odd, the outcomes are all $1$.
\end{example}
In this example, the predictions and outcomes are chosen such that the sign of the difference $y_i-p_i$ alternates in each epoch.
\begin{gather*}
\begin{array}{c|cccccccccc}
i&1&2&3&\hdots&k&k+1&k+2&\hdots&2k-1&2k  \\
\hline
\pb&1/m&2/m&3/m&\hdots&k/m&(k+1)/m&(k+2)/m&\hdots&(2k-1)/m&2k/m  \\
\hline
\yb&1/k&0&3/k&\hdots&0&1&2/k&\hdots&1&1\\
\hline
y_i-p_i&+1/m&-2/m&+3/m&\hdots&-1/2&1/2-1/m&-1/2+2/m&\hdots&-1/m&0
\end{array}
\end{gather*}
To analyze the calibration regrets for this sequence, consider the difference between predicted values in the $i$th epoch and the outcomes.

\paragraph{Upper bound on proper calibration regret.}
Consider the difference between the predicted values in the $i$th epoch and the outcomes.
For the first $k$ epochs, where $i \in \{1,\hdots,k\}$, the difference between $y_i-p_i$ in the $i$th epoch is given as:
$$\frac{(-1)^{i+1}\cdot i}{m}$$
For the second $k$ epochs, where $i \in \{k+1,\hdots,2k\}$,  the difference between $y_t-p_t$ is given as:
$$\frac{(-1)^{2k-i+1}\cdot(2k-i)}{m} = \frac{(-1)^{j+1}\cdot j}{m}$$
for $j = 2k-i$, ranging from $j=k-1$ down to $0$.

The proper calibration error of this sequence can be bounded as follows.
\begin{align*}
    \sup_{s \in \set{\pm}, \theta \in [0,1]} \sum_{t =1}^{T} s \cdot \Th_\theta(p_t) \cdot (y_t - p_t)
    &= \frac{T}{m} \cdot \sup_{s \in \set{\pm}, \theta \in [0,1]} \sum_{i =1}^{m} s \cdot \Th_\theta(p_i) \cdot (y_i - p_i)\\
    &= \frac{T}{m} \cdot \sup_{s \in \set{\pm}, \theta \in [0,1]} s \cdot \left(\sum_{i =1}^{k} \Th_\theta(p_i) \cdot \frac{(-1)^{i+1} \cdot i}{m} + \sum_{j =0}^{k-1} \Th_\theta(p_{2k-j}) \frac{(-1)^{j+1}\cdot j}{m}\right)\\
    &\le \frac{2 T}{m} \cdot \card{\sum_{i =0}^{k} \frac{(-1)^{i+1} \cdot i}{m}}
\end{align*}
where the final inequality follows from the fact that the choice of threshold that maximizes the regret is in between $p_k = k/m = 1/2$ and $p_{k+1} = k+1/m = 1/2 + 1/m$, with a negative sign $s = -1$.

The magnitude of $\card{\sum_{i=1}^k (-1)^{i+1}\cdot i}$ scales linearly with $k$, so in all, the regret is bounded as follows.
\begin{gather*}
    \frac{2 T}{m} \cdot \card{\sum_{i =0}^{k} \frac{(-1)^{i+1} \cdot i}{m}} \le \frac{2 T}{m} \cdot \frac{O(k)}{m} = O(T/k)
\end{gather*}

\paragraph{Lower Bound on $\ell_1$-calibration regret.}
To track the $\ell_1$-calibration error, each term from the sums above contribute with their absolute value within the summation.
\begin{align*}
    \sup_{w:[0,1]\to[-1,1]} \sum_{t=1}^T w(p_t) (y_t - p_t)
    &= \sum_{t=1}^T \card{y_t - p_t}\\
    &\ge \frac{2T}{m} \cdot \sum_{i=1}^{k-1} \card{\frac{i}{m}}
\end{align*}
The magnitude of $\sum_{i=1}^{k-1} \card{i}$ scales quadratically with $k$, so in all, the regret is bounded as follows.
\begin{gather*}
    \frac{2 T}{m} \cdot \sum_{i =1}^{k-1} \card{\frac{i}{m}} \ge \frac{2 T}{m} \cdot \frac{\Omega(k^2)}{m} = \Omega(T)
\end{gather*}

\paragraph{Separation.}
In our construction, there are at least $2k^2$ time steps (as each epoch needs to reason about fractions of outcomes to precision $1/k$), so taking $k = \Theta(\sqrt{T})$, the proper calibration regret is upper bounded by $O(\sqrt{T})$, whereas the $\ell_1$-calibration regret is lower bounded by $\Omega(T)$.

\subsection{Proper calibration is stronger than other prior notions}

Here, we show that proper calibration implies a number of prior notions of calibration, but is not implied by them.
For each of $\ell_\infty$-calibration, U-calibration, and smooth calibration, we give examples that show upper bounding regret with respect to any of these notions is insufficient to give the same asymptotic upper bound for proper calibration regret.

\subsubsection*{$\ell_\infty$-calibration}
To define $\ell_\infty$-calibration, we have to fix prediction level sets to be some discretization of the $[0,1]$ interval.
\begin{definition*}
Fix $m \in \Nbb$ and consider predictions that live in multiples of $1/m$, $p \in \{0, 1/m, 2/m, \ldots, 1 \}$.
For such a sequence of $T$ predictions $\pb$ and outcomes $\yb$
\[
\ell_\infty\calerr(\pb,\yb) = \max_{v \in \{0, 1/m,\ldots, 1\}} \left| \sum_{t=1}^T \mathbf{1} (p_t = v) (y_t - p_t) \right|
\]
\end{definition*}

\paragraph{Proper Calibration implies $\ell_\infty$-calibration.} 

Fix a discretization parameter $m$, a predictor that predicts values in $\{0,1/m, 2/m, \ldots, 1\}$. First we note that for any $i \in [m]$
\begin{gather*}
    \mathbf{1}(p = i/m) = \frac{1}{2} \left(\Th_{i/m}(p) - \Th_{(i+1)/m}(p)\right)
\end{gather*}
Thus, the $\ell_\infty$ calibration error can be bounded as follows:
\begin{align*}
& \max_{v \in \{0, 1/m, 2/m, \ldots, 1\}} \left| \sum_{t=1}^T \mathbf{1} (p_t = v) \cdot (y_t - p_t) \right| \\
& \quad = \frac{1}{2} \cdot \max_{p \in \{0, 1/m, 2/m, \ldots, 1\}} \left| \sum_{t=1}^T \left(\Th_{i/m}(p_t) - \Th_{(i+1)/m}(p_t)\right) \cdot (y_t - p_t) \right| \\
& \quad \le \sup_{s \in \set{\pm}, \theta \in [0,1]} \sum_{t=1}^T s \cdot \Th_\theta(p_t) \cdot (y_t - p_t)
\end{align*}

\paragraph{$\ell_\infty$-calibration does not imply Proper Calibration.}
The separation between $\ell_\infty$-calibration and proper calibration leverages the fact that predictions that are proper calibrated cannot be consistently biased in the same direction across prediction intervals, whereas $\ell_\infty$-calibration only cares about the maximum deviation over prediction intervals.

\begin{example} \label{ex:biased}
    Consider a sequence of $T$ outcomes $\yb$ that comes in $m$ epochs of equal length.
    In the $i$th epoch for each $i \in \{1,\hdots,m\}$, an $i/m$ fraction of the outcomes are $1$ and the remainder are $0$.
    Consider a sequence of predictions $\pb$ supported on $m$ predictions $\{0,1/m,2/m,\hdots,(m-1)/m\}$, where for each time period in the $i$th epoch, the prediction $p_t = (i-1)/m$.
\end{example}

In this example, we can lower bound the proper calibration error by the bias of the predictor.
Consider the threshold at $0$, $\Th_0(p) = 1$, and consider the regret associated with this trivial threshold.
\begin{align*}
    \sup_{s \in \set{\pm}, \theta \in [0,1]} \sum_{t=1}^T s \cdot \Th_\theta(p_t) \cdot (y_t - p_t)  &\ge \sum_{t =1}^T (y_t - p_t)\\
    &= \frac{T}{m}\cdot \sum_{i=1}^m \left(\frac{i}{m} - \frac{i-1}{m}\right)\\
    &= \frac{T}{m} \cdot m \cdot \frac{1}{m}\\
    &= \frac{T}{m}
\end{align*}
In other words, the predictions are consistently biased, and proper calibration detects this bias, with regret $T/m$.

The $\ell_\infty$-calibration error, however, only detects the maximum deviation.
\begin{align*}
    \max_{v \in \{0, 1/m,\ldots, 1\}} \left| \sum_{t=1}^T \mathbf{1} (p_t = v) (y_t - p_t) \right|
    &= \frac{T}{m}\cdot \max_{i\in [m]} \left(\frac{i}{m} - \frac{i-1}{m}\right)\\
    &= \frac{T}{m} \cdot \left(\frac{i}{m} - \frac{i-1}{m}\right)\\
    &= \frac{T}{m^2}
\end{align*}
In our construction, there are at least $m^2$ time steps (as each epoch needs to reason about fractions of outcomes to precision $1/m$), so taking $m = \Theta(\sqrt{T})$, the $\ell_\infty$-calibration error is upper bounded by $O(1)$, whereas the proper calibration error is lower bounded by $\Omega(\sqrt{T})$.

\subsubsection*{U-Calibration}
Recall that U-Calibration measures calibration in terms of the worst-case proper loss regret.
\begin{definition*}[\citet{kleinberg2023u}]
    For a sequence of $T$ predictions $\pb$ and outcome $\yb$, let $p^* = \frac{1}{T} \sum_{t} y_t$; then,
\[
\mathrm{UCal} (\pb, \yb) = \sup_{\ell \in \lprop} \left[ \sum_{t=1}^T \ell (p_t, y_t) - \ell (p^*, y_t) \right]
\]
\end{definition*}

\paragraph{Proper Calibration implies U-Calibration.}\label{par:prop_to_ucal}
We bound the U-calibration error in terms of the proper calibration error for any sequence.
To do this, first consider the following equality;
fixing $p_t,y_t$, for any $p \in [0,1]$:
\begin{align*}
    \E_{y \sim \Ber(p_t)}[\ell(p,y)] - \ell(p,y_t)
    &= \left(p_t \cdot \ell(p,1) + (1-p_t) \cdot \ell(p,0)\right) - \left(y_t \cdot \ell(p,1) + (1-y_t) \cdot \ell(p,0)\right)\\
    &= \left(p_t \cdot \Delta \ell(p) + \ell(p,0)\right) - \left(y_t \cdot \Delta \ell(p) + \ell(p,0)\right)\\
    &= (p_t-y_t) \cdot \Delta \ell(p)
\end{align*}
Then, consider expanding the U-calibration error as follows.
\begin{align*}
    \sup_{\ell \in \lprop} \left[ \sum_{t=1}^T \ell (p_t, y_t) - \ell (p^*, y_t) \right]
    &= \sup_{\ell \in \lprop} \left[ \sum_{t=1}^T \ell (p_t, y_t) - \E_{y \sim \mathrm{Ber}(p_t)}[\ell (p_t, y)] + \E_{y \sim \mathrm{Ber}(p_t)}[\ell (p_t, y)] - \ell (p^*, y_t) \right]\\
    &\le \sup_{\ell \in \lprop} \left[ \sum_{t=1}^T \ell (p_t, y_t) - \E_{y \sim \mathrm{Ber}(p_t)}[\ell (p_t, y)] + \E_{y \sim \mathrm{Ber}(p_t)}[\ell (p^*, y)] - \ell (p^*, y_t) \right] \\
    &= \sup_{\ell \in \lprop} \left[ \sum_{t=1}^T (y_t - p_t) \Delta \ell (p_t) + (y_t - p_t) \Delta \ell (p^*) \right] \\
    &\leq \sup_{\ell \in \lprop} \left| \sum_{t=1}^T (y_t - p_t) \Delta \ell (p_t) \right| + \sup_{\ell \in \lprop} \left| \sum_{t=1}^T (y_t - p_t) \Delta \ell (p^*) \right|\\
    &\le 2 \cdot \sup_{\ell \in \lprop} \left| \sum_{t=1}^T (y_t - p_t) \Delta \ell (p_t) \right|
\end{align*}
where the final inequality follows by the fact that for all $\ell \in \lprop$, the weight function $\Delta\ell(p^*)$ is a constant across all $t$.
This constant function is realizable by some choice of proper loss applied to $p_t$.
Thus, the U-Calibration error is upper bounded by a constant factor of the proper calibration error.

\paragraph{U-Calibration does not imply Proper Calibration.}
We present a sequence of predictions and outcomes where the U-calibration regret is $\le 0$, but the proper calibration error grows linearly in $T$.

\begin{example}\label{ex:u-cal}
Consider a sequence of $T$ predictions $\yb$ where $y_t = 1$ if $t > T/2$ and $0$ otherwise. Also a sequence of $T$ predictions $\pb$ where $p_t = 0.9$ if $t > T/2$, and $0.1$ otherwise.
\end{example}

First we show that the sequence of predictions are U-calibrated. To do this, we use the upper bound in \cite{kleinberg2023u} that $\ucal \leq 2 \vcal$ and bound $\vcal$. For any $v \in [0,1]$, we show that $\sum_{t=1}^T \ell_v (p_t, y_t) - \ell_v (0.5, y_t) < 0$ (since $0.5$ is the best fixed prediction). Plugging in the definition of $\ell_v (p,y) = (y - v) \mathrm{sgn} (v - p)$
\[
\sum_{t=1}^T \ell_v (p_t, y_t) = \frac{T}{2} (-v) \mathrm{sgn}(v - 0.1) + \frac{T}{2} (1-v) \mathrm{sgn}(v - 0.9)
\]
Similarly, the second term simplifies to 
\[
\sum_{t=1}^T \ell_v (0.5, y_t) = \frac{T}{2} (-v) \mathrm{sgn}(v - 0.5) + \frac{T}{2} (1-v) \mathrm{sgn}(v - 0.5)
\]
Now, we consider four cases for v:

\begin{itemize}
\item $v \leq 0.1$: In this case, all the $\mathrm{sgn}(v-p)$ terms have the same sign, so the difference between the two sums is 0.
\item $v \in (0.1, 0.5)$: The second terms of both expressions cancel out, but the first term of the second expression becomes positive, making the overall difference negative.
\item $v \in (0.5, 0.9]$: Similar to case 2, the second terms cancel out, and the first term of the second expression becomes positive, making the overall regret negative
\item $v > 0.9$: Similar to case 1, all the $\mathrm{sgn}(v-p)$ terms have the same sign, so the difference is 0.
\end{itemize}
In all cases, we have shown that $\sum_{t=1}^T \ell_v (p_t, y_t) - \ell_v (0.5, y_t) < 0$. Therefore, the predictor is U-calibrated.

Now, we show that the sequence of predictions has linear proper calibration error. To show this, it suffices to consider weight functions $\Delta \ell_v (p) = \mathrm{sgn}(v - p)$ for $v = 0.5$. The weighted calibration error for this weight function is
\[
\sum_{t=1}^T \Delta \ell_v (p_t) (y_t - p_t)
\]
we can split the sum into two parts
\[
\sum_{t=1}^T \mathrm{sgn} (v - p_t) (y_t - p_t) = T/2 (0 - 0.1) + T/2 (-1)(1 - 0.9) = -0.1 T
\]
Taking the absolute value shows that the proper calibration error of the predictor is linear in $T$.

\subsubsection*{Smooth Calibration}
Recall that Smooth Calibration is defined as a notion of weighted calibration using Lipschitz weight functions.
\begin{definition*}[\citet{kakade2008deterministic,foster2018smooth}]
For a sequence of $T$ predictions $\pb$ and outcome $\yb$,
\[
\mathrm{SmoothCal} (\pb, \yb) =  \max_{w \in \Wcal_\mathrm{Lip}} \left| \sum_{t=1}^T w(p_t(x_t)) (y_t - p_t(x_t)) \right|
\]
where $\Wcal_\mathrm{Lip}$ is the set of $1$-Lipschitz functions from $[0,1]$ to $[-1,1]$.
\end{definition*}

\paragraph{Proper Calibration implies Smooth Calibration.}
Proper calibration can be characterized in terms of signed threshold functions.
This characterization shows that proper calibration guarantees weighted calibration with respect to weight functions defined by the difference of monotone functions, as signed thresholds form a basis.
All Lipschitz functions can be expressed as a difference of monotone functions, so proper calibration implies smooth calibration.

Consequently, this implication further implies that proper calibration implies an upper bound on the lower distance to calibration \citep{blasiok2023unifying}.

\paragraph{Smooth Calibration does not imply Proper Calibration.}
This separation follows from the fact that proper calibration can encode highly non-Lipschitz tests, by the characterization in terms of threshold functions.
The following example highlights this difference.

\begin{example}\label{ex:smooth}
Consider a sequence of $T$ predictions $\yb$ where $y_t = 1$ if $t > T/2$ and $0$ otherwise. Also a sequence of $T$ predictions $\pb$ where $p_t = 1/2 - \eps$ if $t > T/2$, and $1/2 + \eps$ otherwise.
\end{example}

To upper bound the smooth calibration error, we consider an arbitrary weight function $w \in \Wcal_\mathrm{lip}$ and evaluate its weighted calibration error
\[
\sum_{t=1}^T w (p_t) (y_t - p_t) = \frac{T}{2} w(1/2 - \eps) (1/2 + \eps) + \frac{T}{2} w(1/2 + \eps) (- 1/2 - \eps)
\]
which simplifies to $\frac{T}{2} (w(1/2 - \eps) - w(1/2 + \eps)) (1/2 + \eps)$.
Since $w$ is $1$-lipschitz, the difference $(w(1/2 - \eps) - w(1/2 + \eps))$ is bounded by $\eps$ in magnitude. Thus, taking absolute values, we obtain that for all $1$-lipschitz weight functions $w$
$\sum_{t=1}^T w (p_t) (y_t - p_t)  \leq \eps T$

To lower bound the proper calibration error, it suffices to evaluate a single weight function $w(p) = \mathrm{sgn} (v - p)$ for $v = 0.5$.
\[
\sum_{t=1}^T \mathrm{sgn} (0.5 - p_t) (y_t - p_t) = T/2 (1/2 + \eps) + T/2 (-1)(- 1/2 - \eps) \geq T/4
\]

Thus, we can make the smooth calibration error arbitrarily small by setting $\eps$ appropriately, however, the proper calibration error will always remain linear in $T$.

\subsection{Calibration Decision Loss is incomparable to Proper Calibration}

\cite{hu2024calibrationerrordecisionmaking} introduced the notion of Calibration Decision Loss (CDL), which is a notion of swap regret simultaneous for decision-making losses.
For a sequence of predictions $\pb$ and outcomes $\yb$, and a time step $t$ where $p_t = i$, define the empirical swap prediction as follows. $$\hat{p_t} = \frac{1}{\card{\{s \in [T] : p_s = i\} }} \sum_{s=1}^T y_s \cdot \1[p_s = i]$$
CDL has a number of equivalent formulations (up to constant factors), but one is based on Bregman divergences between the selected predictions and the empirical swap predictions.
The CDL takes a worst-case choice over Bregman divergences derived from proper losses, and specifically, the V-Shaped losses.
\begin{gather*}
    \mathrm{CDL}(\pb,\yb) = \sup_{v \in [0,1]} \sum_{t=1}^{T}B_v(p_t,\hat{p_t})
\end{gather*}
where $B_v(p,q) = u_v(q) - u_v(p) + \nabla u_v(p) \cdot (p-q)$ is a Bregman divergence defined by the potential function $u_v(p) = -\E_{y \sim \Ber(p)}[\ell_v(y,p)]$ defined as the Bayes risk for the (negative) loss of the V-shaped loss $\ell_v$.\footnote{Note to better compare to proper calibration, we formulate these quantities in terms of our parameterization of the V-shaped losses, whereas \citet{hu2024calibrationerrordecisionmaking} formulate the divergence in terms of proper scoring functions.  Our formulation is equivalent to theirs up to a small constant factor.}

Some consequences of this definition that make comparison to proper calibration possible include the following.
\begin{itemize}
    \item $-\ell_v(y,p) = (y-v) \cdot \sgn(p-v)$
    \item $-\Delta \ell_v(p) = \sgn(p-v)$
    \item $u_v(p)
= -p \cdot \Delta \ell_v(p) - \ell_v(0,p)
=  (p-v) \cdot \sgn(p-v)$
    \item $\nabla u_v(p) = -\Delta \ell_v(p) = \sgn(p-v)$
\end{itemize}

With these facts in place, we can expand the Bregman divergence $B_v(p,q)$ as follows.
\begin{align*}
    B_v(p,q) &= u_v(q) - u_v(p) + \grad u_v(p) \cdot (p-q)\\
    &= (q-v) \cdot \sgn(q-v) - (p-v) \cdot \sgn(p-v) + \sgn(p-v) \cdot (p-q)\\
    &= \left(\sgn(q-v) - \sgn(p-v)\right) \cdot (q-v)
\end{align*}
That is, the Bregman divergence is equal to $0$ when $q$ and $p$ are on the same side of $v$, and otherwise is equal to $2 \cdot \card{q-v}$.

\paragraph{Calibration Decision Loss does not imply Proper Calibration.}
This separation follows from the fact that CDL does not detect consistent bias.
In particular, \Cref{ex:biased} gives a separation.
As we showed earlier, in this example, the proper calibration regret is lower bounded by $\Omega(\sqrt{T})$.
Thus, we need only to upper bound the CDL to show a separation.

Note that every prediction $p_t$ is close to $\hat{p_t}$; that is, $\card{p_t - \hat{p_t}} \le 1/m$.
Further, the prediction intervals the empirical swap intervals do not cross.
Formally, consider two epochs $i < j$; for all $s$ in the $i$th epoch and all $t$ in the $j$th epoch,
\begin{gather*}
    p_s < \hat{p_s} \le p_t < \hat{p_t}.
\end{gather*}
Thus, for any threshold $\theta \in [0,1]$, there is at most one epoch $i$ such that the prediction $(i-1)/m$ and empirical swap prediction $i/m$ are on opposite sides of $\theta$.
Thus, the CDL is upper bounded by the contribution from this epoch---matching the $\ell_\infty$-calibration regret.
\begin{align*}
    \mathrm{CDL}(\pb,\yb) &= \sup_{\theta \in [0,1]} \sum_{t=1}^T B_\theta(p_t,\hat{p_t})\\
    &\le \frac{T}{m} \cdot \max_{i \in [m],~ \eps \in (0,1/m)} ~B_{i/m-\eps}\left(\frac{i-1}{m},\frac{i}{m}\right)\\
    &\le \frac{T}{m} \cdot \frac{1}{m}\\
    &= \frac{T}{m^2}
\end{align*}
Again, taking $m = \Theta(\sqrt{T})$, the Calibration Decision Loss is upper bounded by $O(1)$, but the proper calibration regret scales as $\Omega(\sqrt{T})$.

\paragraph{Proper Calibration does not imply Calibration Decision Loss.}
In fact, \Cref{ex:ell1} shows a separation between proper calibration error and calibration decision loss.
As before, taking $k = \Theta(\sqrt{T})$, the proper calibration regret is upper bounded by $O(\sqrt{T})$.
Thus, it remains to lower bound the CDL.

Consider choosing the Bregman divergence $B_v$ for $v = 1/2$.
When $v$ sits between the predicted value and the empirical swap prediction, then by the definition of $B_v$, we get a contribution to the CDL.
In particular, in the first $k$ epochs, the odd epochs from $i = k/2+1$ to $k-1$ will be split and contribute to the CDL as:
\begin{align*}
    2 \cdot \card{\hat{p}_t - 1/2}
    &= 2 \cdot \card{\frac{i}{k} - \frac{1}{2}}\\
    &= \frac{2i - k}{k}\\
    &= \frac{2j}{k}
\end{align*}
for $j = i-k/2$ from $j = 1$ to $k/2-1$.
In the second $k$ epochs, the even epochs from $i = k+2$ to $3k/2$ will be split and contribute to the CDL as:
\begin{align*}
    2 \cdot \card{\hat{p}_t - 1/2}
    &= 2 \cdot \card{\frac{(i-k)}{k} - \frac{1}{2}}\\
    &= \frac{\card{2(i-k) - k}}{k}\\
    &= \frac{\card{2i - 3k}}{k}\\
    &= \frac{2j}{k}
\end{align*}
for $j = 3k/2-i$ from $j = 0$ to $k/2-2$.

In combination, the CDL can be expressed as follows.
\begin{align*}
\sup_{v \in [0,1]}\sum_{t=1}^T B_v(p_t,\hat{p}_t)
&\ge \frac{T}{m} \cdot \left(\sum_{\substack{j:\textrm{ odd }\\ 0 < j < k/2}} \frac{2j}{k} + \sum_{\substack{j:\textrm{ even }\\ 0 \le j < k/2}} \frac{2j}{k} \right)\\
&= \frac{T}{m} \cdot \sum_{i=0}^{k/2-1} \frac{2j}{k}
\end{align*}
The summation $\sum_{i=0}^{k/2-1} j$ scales quadratically in $k$ (with $m = 2k$), so overall, the regret is lower bounded as follows.
\begin{gather*}
    \frac{T}{m} \cdot \sum_{i=0}^{k/2} \frac{2j}{k} \ge \frac{T}{m} \cdot \frac{\Omega(k^2)}{k} = \Omega(T)
\end{gather*}
Thus, an upper bound on the proper calibration regret does not imply the same bound on the CDL.

 \section{Online Multiaccuracy}
\label{sec:online_ma}

In this section, we introduce online multiaccurary as an important primitive for our algorithms for online omniprediction, as well as future applications.
For completeness, we present algorithms for achieving online multiaccuracy for finite hypothesis classes and for infinite hypothesis classes via online weak agnostic learning oracles.

\subsection{Multiaccuracy for Finite Hypothesis Classes}
\begin{theorem}[Multiaccuracy]\label{thm:finite-ma}
Given a finite class of hypothesis functions $\Hcal$, \Cref{alg:finite-ma} guarantees expected $\Hcal$ multiaccuracy error of $O \left(\sqrt{T \ln|\Hcal|} \right)$
\end{theorem}

\paragraph{Overview of \Cref{alg:finite-ma}:}\label{par:finite-ma-overview}
The algorithm is based on Blackwell's Approachability Theorem. We define a two player game where the adversary player selects $z_t = (x_t, y_t) \in \Xcal \times \{0,1\}$ and the learner selects $p_t: \Xcal \rightarrow [0,1]$. Both players are allowed to play randomized strategies but since the learner observes $x_t$, we can simplify things and only consider $y_t \in \{0,1\}$ and $p_t = p_t (x_t)$. We design the payoff vector of this game to reflect our objective of $\Hcal$ multiaccuracy respectively. Define
\[
u_{h, s} (p_t, z_t) = s(y_t - p_t) h(x_t)  \quad \text{for} \ h \in \Hcal, s \in \{+,-\}
\]
Observe that after T rounds of interaction, 
$\Hcal \maerr (\pxy) = \max_{h,s} \sum_{t \in [T]} u_{h,s} (p_t, z_t)$. Therefore, we design the learner's target set to be the set of all vectors $u$ whose coordinates is bounded by $0$.

We use exponential weights update method in \Cref{line:finite-ma-exp} to generate sequence of halfspaces $w^t$ with coordinates for every $h \in \Hcal, s \in \{+,-\}$. Given a halfspace $w^t$, the algorithm follows the strategy described from \Cref{line:ma_str_start} to \Cref{line:ma_str_end}.

\begin{algorithm}[h!]
\caption{Multiaccuracy Algorithm}
\label{alg:finite-ma}
\textbf{Input:} Hypothesis class $\Hcal \subseteq \{h: \Xcal \rightarrow [0,1] \}$ \\
\textbf{Input:} Sequence of samples $\{ (x_1, y_1), \ldots, (x_T, y_T) \}$ \\
\textbf{Output:} Sequence of (randomized) predictors $p_1, \ldots, p_T$
\begin{algorithmic}[1]
\For{each $t \in [T]$}
    \State Let $w_{h, s}^t := \frac{\exp \left( \eta \sum_{i=1}^{t-1} u_{h, s} (p_i, z_i)\right)}{ \sum_{h', s'} \exp \left( \eta \sum_{i=1}^{t-1} u_{h', s'} (p_i, z_i)\right) 
    }$ for all $h \in \Hcal, s \in \{+,-\}$
    \label{line:finite-ma-exp}
\State Compute
\[
f(x_t) = \sum_{h,s} w^t_{h,s} \cdot s \cdot h(x_t)
\] \label{line:f_eq}
    \If{$f(x_t) \leq 0$} \label{line:ma_str_start}
    \State Predict $p_t (x_t) = 0$
    \ElsIf{$f(x_t) > 0$}
    \State Predict $p_t (x_t) = 1$
    \EndIf \label{line:ma_str_end}
    \State Observe $x_t$, predict $p_t(x_t)$, and then observe $y_t$
\EndFor
\end{algorithmic}
\end{algorithm}

\begin{lemma}[Halfspace Approachability]\label{lem:finite-ma-halfspace}
Given a halfspace $w$, the strategy described in \Cref{line:ma_str_start} to \Cref{line:ma_str_end} outputs $p_t$ such that $\langle w, u (p_t, z_t) \rangle \leq 0$ for any choice of $z_t$
\end{lemma}
\begin{proof}[Proof of \Cref{lem:finite-ma-halfspace}]
We consider the cases in the strategy separately:
\begin{itemize}
\item Case 1: If $f(x_t) \leq 0$, predict $p_t (x_t) = 0$.
Thus, for $h \in \Hcal, s \in \{+, -\}$
\[
w_{h,s} u_{h,s} (p_t, z_t) = s(y_t - p_t) h(x_t) w_{h,s} = s y_t h(x_t) w_{h,s}^t
\]
Summing over values of $h, s$, we get 
\[
\langle w, u (p_t, z_t) \rangle = y_t f(x_t) \leq 0 
\quad \text{for any choice of} \ y \in \{0, 1\}
\]

\item Case 2: If $f(x_t) > 0$, predict $p_t (x_t) = 1$.
Thus, for $h \in \Hcal, s \in \{+, -\}$
\[
w_{h,s} u_{h,s} (p_t, z_t) = s(y_t - p_t) h(x_t) w_{h,s} = s y_t (h(x_t)) w_{h,s}
\]
Summing over values of $h, s$, we get 
\[
\langle w, u (p_t, z_t) \rangle = (y_t - 1) f(x_t) \leq 0 \quad \text{for any choice of} \ y \in \{0, 1\}
\]
\end{itemize}
\end{proof}

\begin{lemma}[Exponential Weight Updates \cite{arora2012multiplicative}]\label{lem:exp_weight}
The exponential weight updates in \Cref{line:finite-ma-exp} provide a sequence of vectors $w^t$ such that 
\[
\max_{w: ||w||_1 = 1} \left\langle w, \sum_{t \in [T]} u (p_t, z_t) \right\rangle \leq \sum_{t=1}^T \langle w^t, u (p_t, z_t) \rangle + O \left( \sqrt{T \ln |\Hcal|} \right)
\]
\end{lemma}

\begin{proof}[Proof of \Cref{thm:finite-ma}]
\sloppy We wish to bound the multiaccuracy error, that is, $\max_{h,s} \sum_{t \in [T]} u_{h, s}$ where the expectation is over the randomness in the sampling of $p_t$. Note that this is at most $\max_{w: ||w||_1 = 1} \left\langle w, \sum_{t \in [T]} u (p_t, z_t) \right\rangle $.
\begin{align*}
\max_{w: ||w||_1 = 1} \left\langle w, \sum_{t \in [T]} u (p_t, z_t) \right\rangle  & \leq 
\sum_{t=1}^T \langle w^t, u (p_t, z_t) \rangle  + O \left( \sqrt{T \ln |\Hcal|} \right) \tag{by \Cref{lem:exp_weight}} \\
& = 
\sum_{t=1}^T \langle w^t, u (p_t, z_t) \rangle + O \left( \sqrt{T \ln |\Hcal|} \right)
\tag{by linearity of expectation} \\
& \leq \left( \sum_{t=1}^T 0 \right)  + O \left( \sqrt{T \ln |\Hcal|} \right) \tag{by \Cref{lem:finite-ma-halfspace}} \\
& \leq O \left( \sqrt{T \ln |\Hcal|} \right)
\end{align*}
\end{proof}

\subsection{Achieving Online Multiaccuracy using Online Weak Agnostic Learner}
\label{sec:online-ma}

\begin{theorem}[Multiaccuracy with Online Weak Agnostic Learner]\label{thm:oracle-ma}
Given a \textit{possibly infinite} class of hypothesis functions $\Ccal$, and an online weak agnostic learner for $\Ccal$ with failure probability $\delta$
and regret bound $\owalreg_\Ccal^\delta (T)$, 
\Cref{alg:oracle-ma} guarantees, with probability at least $1 - \delta$, a 
$\Ccal$-multiaccuracy error bound of
$ \owalreg_\Ccal^\delta (T).$
\end{theorem}

\begin{algorithm}[h!]
\caption{Multiaccuracy with an online weak agnostic learner}
\label{alg:oracle-ma}
\textbf{Input:} Hypothesis class $\Ccal \subseteq \{c: \Xcal \rightarrow [-1,1] \}$ 
\\
\textbf{Input:} Online Weak Agnostic Learner for $\{-1,+1\} \cdot \Ccal$, denoted $\Acal_\mathsf{OWAL}$ \\
\textbf{Input:} Sequence of samples $\{ (x_1, y_1), \ldots, (x_T, y_T) \}$ \\
\textbf{Output:} Sequence of (randomized) predictors $p_1, \ldots, p_T$
\begin{algorithmic}[1]
\State Initialize OWAL, $q_1 \leftarrow \Acal_\mathsf{OWAL}$
\For{$t = 1$ to $T$}
    \State Compute predictor $p_t (x) = \ind [q_t(x) > 0]$
    \State Observe $x_t$, Predict $p_t(x_t)$, Observe $y_t$
    \State Obtain $q_{t+1} \leftarrow \Acal_\mathsf{OWAL} (x_t, y_t - p_t(x_t))$ \Comment{execute one timestep of the OWAL}
\EndFor
\end{algorithmic}
\end{algorithm}

\begin{lemma}[Halfspace Approachability]\label{lem:ma-halfspace}
For every timestep $t \in [T]$, the predictor $p_t (x) = \ind \left[ q_t (x) > 0 \right]$ constructed from the OWAL output $q_t$ satisfies the following guarantee for all $x \in \Xcal, y \in \{0,1\}$:
\[
 q(x)(y - p_t (x)) \leq 0
\]
\end{lemma}
\begin{proof}[Proof of \Cref{lem:ma-halfspace}]
We consider two cases based on the sign of \(q_t(x)\): 
\begin{itemize} 
\item \(q_t(x) > 0.\) In this case \(p_t(x) = 1\). 
\begin{itemize} \item If \(y=1\), then \(y - p_t(x) = 0\), so \(q_t(x)\,(y - p_t(x)) = 0\). \item If \(y=0\), then \(y - p_t(x) = -1\), so \(q_t(x)\,(y - p_t(x)) < 0\). \end{itemize} In both subcases, \(q_t(x)\,(y - p_t(x)) \le 0\). 

\item \(q_t(x) \le 0.\) In this case \(p_t(x) = 0\). 
\begin{itemize} \item If \(y=0\), then \(y - p_t(x) = 0\), so \(q_t(x)\,(y - p_t(x)) = 0\). \item If \(y=1\), then \(y - p_t(x) = 1\), so \(q_t(x)\,(y - p_t(x)) \le 0\). \end{itemize} In both subcases, \(q_t(x)\,(y - p_t(x)) \le 0\). \end{itemize} Hence, for all \(x\) and all \(y \in \{0,1\}\), we have \(q_t(x)\,(y - p_t(x)) \le 0\)
\end{proof}

\begin{proof}[Proof of \Cref{thm:oracle-ma}]
The online weak agnostic learner outputs sequence of predictors for $\{+1,-1\} \cdot \Ccal$ $q_1, \ldots, q_T$ such that with probability at least $1 - \delta$
\begin{align*}
\max_{c \in \Ccal, s \in \{+,-\}} \sum_{t=1}^T s \cdot c(x_t) (y_t - p_t (x_t)) 
&\leq \sum_{t=1}^T q(x_t)(y_t - p_t (x_t)) +  \owalreg_\Ccal^\delta (T) \tag*{(OWAL guarantee)} \\
\\ 
&\leq \sum_{t=1}^T 0 +  \owalreg_\Ccal^\delta (T) \tag*{(by \Cref{lem:ma-halfspace})} \\
&\leq \owalreg_\Ccal^\delta (T)
\end{align*}
The claim follows by observing that $\Ccal\maerr (\pxy) = \max_{c \in \Ccal, s \in \{+,-\}} \sum_{t=1}^T s \cdot c(x_t) (y_t - p_t (x_t))$
\end{proof}

\section{Proper Calibrated Multiaccuracy Boosting}
\label{app:mcboost}

The work of \citet{gopalan2022loss} established the efficiency of the Loss OI omniprediction framework via boosting for simultaneous ($\ell_1$-)calibration and multiaccuracy, drawing on the original algorithms for multi-group fairness \citep{hkrr,kim2019multiaccuracy,pfisterer2021}.
In this appendix, we show that proper calibration auditing can replace $\ell_1$-calibration auditing within this framework with immediate efficiency gains.
While the algorithm only achieves $\eps^{-4}$ dependence on the approximation parameter, it produces deterministic omnipredictors, which may be desirable in comparison to the near-optimal omnipredictors we learn via \Cref{alg:oracle-omni}.

Before presenting the novel variant of the algorithm, we account for the sample complexity achieved by the algorithm presented in \citep{gopalan2022loss}.

\begin{theorem}[Sample Complexity bound in \citep{gopalan2022loss}]\label{thm:old-det-omni}
Given a class of loss functions $\Lcal$, a hypothesis class $\Hcal$ and a weak agnostic learning oracle for $C = (\Delta\Lcal \circ \Hcal)$, there exists an algorithm that outputs a  $(\Lcal, \Hcal, \eps)$-omnipredictor with probability $1- \delta$
\[
O \left( \frac{d_\Ccal + \ln(1/\delta)}{\eps^4} + \frac{\ln (1/\eps)}{\eps^{10}} \right)
\]
where $d_\Ccal$ represents the VC (or fat-shattering) dimension of the class $\Ccal$
\end{theorem}

As mentioned, the sample complexity improvement comes from designing an algorithm that achieves multiaccuracy and low proper calibration error instead of low $\ell_1$-calibration error. 
Core to this improvement is an efficient test for proper calibration.

\paragraph{Testing Proper Calibration via CDF estimation.}
Unlike $\ell_1$-calibration error, we can measure the proper calibration error of an arbitrary predictor $p$ regardless of the fact that we have infinitely many test weight functions. The key idea is to estimate the CDF of the predictor $F(v) = \E_{(x,y) \sim \Dcal} \ind [p(x) \geq v]$. Because the weight functions are monotone functions, the CDF contains all the relevant information for estimating $\E_{(x,y) \sim \Dcal} [w(p(x))]$. We can also express the threshold calibration error as 
\[
\max_{v \in [0,1]} \E_{(x,y) \sim \Dcal} [(y - p(x)) (2\ind [p(x) \geq v] - 1)]
\]

\begin{lemma}[Testing Proper Calibration via CDF estimation]\label{lem:proper-test-cdf}
Given a predictor $p$, sample access to a distribution $\Dcal$, and an error probability $\delta$, there exists an algorithm that returns True with probability $1 - \delta$ if $\pcal_\Dcal (p) \leq \eps$. The algorithm uses 
\[
O \left( \frac{\ln(1/\delta)}{\eps^2} \right)
\]
\end{lemma}
\begin{proof}[Proof of \Cref{lem:proper-test-cdf}]
\sloppy The algorithm first collects $ \frac{\ln(1/\delta)}{\eps^2}$ samples to estimate the CDF $F(v) = \E_{(x,y) \sim \Dcal} [\ind [p(x) \geq v]]$. Then the algorithm identifies the set $Q$ of the $\delta/2$-quantiles of the CDF. Then the algorithm collects fresh samples $ \frac{\ln(1/\delta)}{\eps^2}$ and estimates $\E_{(x,y) \sim \Dcal} [w_q (p(x)) (y - p(x))]$ for each quantile $q \in Q$ from the samples where $w_q (p) = \mathrm{sgn}(q - p)$. Then the algorithm returns True if any estimate is greater than $\delta$ and False otherwise.

Now we show that the algorithm is correct.
By DKW Inequality \citep{dvoretzky1956asymptotic}, we know the empirical CDF is close in infinity distance to the true CDF i.e $\| \hat F - F\|_\infty \leq \delta/2$. This implies that the value $\E_{(x,y) \sim \Dcal} [w_v (p(x)) (y - p(x))]$ for any $v \in [0,1]$ is $\delta/2$ close to $\E_{(x,y) \sim \Dcal} [w_q (p(x)) (y - p(x))]$ for some $q \in Q$. The rest of the claim follows by uniform convergence of the estimates for $q \in Q$.
\end{proof}

\paragraph{Testing Proper Calibration using a Weak Learner for One Dimensional Thresholds.}
\sloppy As defined in the first section, proper calibration can be equivalently captured by weighted calibration over $\{-1,1\}-$threshold functions. In the next lemma, we show how to test if a function is proper-calibrated using a weak agnostic learner for one dimensional threshold functions.

\begin{lemma}[Testing Proper Calibration via Weak Learning]\label{lem:proper-test-wl}
Given a predictor $p$, sample access to a distribution $\Dcal$, and an error probability $\delta$, there exists an algorithm uses a weak agnostic learner for $\wthres: \{ q \rightarrow \mathrm{sgn} (v - q) : v \in [0,1]\}$ and returns True with probability $1 - \delta$ if $\pcal_\Dcal (p) \leq \eps$. The algorithm uses 
\[
O \left( \frac{\ln(1/\delta)}{\eps^2} \right)
\]
\end{lemma}
\begin{proof}[Proof of \Cref{lem:proper-test-wl}]
We first describe the algorithm. The algorithm simply collects $m = O \left( \frac{\ln(1/\delta)}{\eps^2} \right)$ samples from $\Dcal$. Then the algorithm feeds the weak agnostic learner pairs $(u_i,v_i)_{i=1}^m$ where $u_i = p(x_i)$ and $v_i = (y_i - p(x_i))$ and returns True if the weak learner does not return $\bot$, False otherwise.
Now we prove the correctness of the algorithm.
Recall that the VC dimension of one dimensional threshold functions is 1. Therefore a weak learner for 1d threshold class only needs $O \left( \frac{\ln(1/\delta)}{\eps^2} \right)$ samples for this hypothesis class. If the weak learner does not return $\bot$, then there exists a weight function in $\wthres$ such that $\E_{(x, y)\sim \Dcal}[w(p(x))(y - p(x))] \geq \eps/2$, then the proper calibration error is at least that amount. 
\end{proof}

\paragraph{Overview of \Cref{alg:det-omni}.}
The algorithm follows the same boosting-style of \citep{gopalan2022loss}. Starting with a constant predictor $q_0 (x) = \frac{1}{2}$, it iteratively updates this predictor until it is both $C$-multiaccurate and proper calibrated (up to an additive error of $\eps$).
In each iteration $t$, the algorithm draws fresh samples from $\Dcal$ and calls the weak agnostic learning oracle for $C$ to check if $q_t$ is multiaccurate.
If such a $c$ exists, then it updates $q_t$ using $c$, then terminates the iteration.
If not, then the algorithm proceeds to check if $q_t$ is proper calibrated by calling the weak agnostic learning oracle for one dimensional thresholds.
If there exists a threshold function $w$ that correlates with $y - q_t(x)$ using $q_t(x)$ as input, then the algorithm updates $q_t$ using $w$.
If not, the algorithm terminates.

We improve the dependence on $\eps$ in this result significantly and obtain the following bound.
\begin{theorem}[Improved Sample Complexity Upper Bound]\label{thm:det-omni}
Given a class of loss functions $\Lcal$, a hypothesis class $\Hcal$ and a weak agnostic learning oracle for $C = (\Delta\Lcal \circ \Hcal)$, \Cref{alg:det-omni} outputs a  $(\Lcal, \Hcal, \eps)$-omnipredictor with probability $1- \delta$
\[
O \left( \frac{d_\Ccal + \ln(1/\delta)}{\eps^4} \right)
\]
where $d_\Ccal$ represents the VC (or fat-shattering) dimension of the class $\Ccal$ (as determined by the WAL oracle).
\end{theorem}

\begin{algorithm}[t!]
    \caption{PCal+MABoost}
    \label{alg:det-omni}
    \textbf{Input:}  Error parameters $\eps, \delta \in [0,1]$, Sample access to $\Dcal$ \\
    Oracle access to a Weak Agnostic learner for $C = (\Delta\Lcal \circ \Hcal)$, denoted $\mathsf{WAL}_{\Ccal, \eps}$\\
    Oracle access to a Weak Agnostic learner for one dimensional thresholds, denoted $\mathsf{WAL}_{\mathrm{Th}, \eps}$\\
    \textbf{Output:}  Predictor $q_T$.
    \begin{algorithmic}
    \State $t \gets 0$
    \State $q_0 (\cdot) \gets \frac{1}{2}$ \Comment{the algorithm can be initialized with an arbitrary predictor}
    \State $ma \gets \mathsf{false}$
    \State $pc \gets \mathsf{false}$
    \While{$\neg ma$ or $\neg pc$}
        \State $c_{t+1} \gets \mathsf{WAL}_{\Ccal, \eps} \left( \left\{x_i, y_i - q_t (x_i) \right\}_{i = 1}^m \right)$ \Comment{where $(x_i, y_i)_{i=1}^m$ are fresh samples from $\Dcal$}
        \If{$c_{t+1} = \bot$}
            \State $ma \gets \mathsf{true}$ 
        \Else
            \State $ma \gets \mathsf{false}$ 
            \State $h_{t+1} (\cdot) \gets q_t (\cdot) + \frac{\eps}{2} \cdot c_{t+1} (\cdot)$.
            \State $q_{t+1} (\cdot) \gets \Pi(h_{t+1} (\cdot))$ \Comment{where $\Pi$ projects $h_{t+1}$ onto $[-1,1]$}
            \State $t \gets t+1$.
            \State \textbf{break}
        \EndIf
        \State $w_{t+1} \gets \mathsf{WAL}_{\mathrm{Th}, \eps}  \left( \left\{q_t (x_i), y_i - q_t (x_i) \right\}_{i = 1}^m \right) $ \Comment{where $(x_i, y_i)_{i=1}^m$ are fresh samples from $\Dcal$}
        \If{$w_{t+1} = \bot$}
            \State $pc \gets \mathsf{true}$ 
        \Else
            \State $pc \gets \mathsf{false}$ 
            \State $h_{t+1} (\cdot) \gets q_t (\cdot) + \frac{\eps}{2} \cdot w_{t+1} (q_t (\cdot))$.
            \State $q_{t+1} (\cdot) \gets \Pi(h_{t+1} (\cdot))$ \Comment{where $\Pi$ projects $h_{t+1}$ onto $[-1,1]$}.
            \State $t \gets t+1$.
\EndIf
    \EndWhile
    \Return $q_t$.
\end{algorithmic}
\end{algorithm}

\begin{proof}[Proof of \Cref{thm:det-omni}]
First we show that when \Cref{alg:det-omni} terminates, then $q_T$ is an $(\Lcal, \Hcal, \eps)$-omnipredictor with probability at least $1 - \delta$. This follows from the weak agnostic learning guarantee. For the algorithm to terminate, both weak agnostic learners must return $\bot$. Since the weak agnostic learner for $C$ returns $\bot$, then every hypothesis in $C$ has correlation less than $\eps$ with $q_T$. Since the weak agnostic learner for one dimensional thresholds returns $\bot$, the proper calibration error of $q_T$ is at most $\eps$.

Now we show that \Cref{alg:det-omni} terminates after $8/\eps^2$ iterations. We do this by using the expected square distance of $q_t$ to the Bayes-optimal predictor $q^*$ as a potential function. We show that after each iteration, the squared distance reduces by $O(\eps^2)$. The change in potential from the $t$-th iteration to the $t+1$-iteration can be expressed as follows:
\[
\E [(q^* (x) - q_t (x))^2] - \E [(q^* (x) - q_{t+1} (x))^2]
\]
Since $q_{t+1} = q_t + \frac{\eps}{2} \cdot f$ where $f$ is either a $c \in C$ or $w \circ q_t$ for some $w \in \wthres$, we have
\begin{align*}
&\E [(q^* (x) - q_t (x))^2] - \E [(q^* (x) - q_t (x) + \tfrac{\eps}{2} \cdot f (x))^2] \\
&= \eps \cdot  \E [(q^* (x) - q_t (x)) f(x)] - \tfrac{\eps^2}{4} \cdot \E [f(x)^2]
\end{align*}
By the weak agnostic learning assumption, we know that $\E [(q^* (x) - q_t (x)) f(x)] \geq \frac{\eps}{2}$. We also know that $\E [f(x)^2] \leq 1$ since $f \in [-1,1]$. Thus, the change in potential is at least $\frac{\eps^2}{4}$. Since the expected squared distance at the start is less than 2, the total number of iterations $T$ is at most $8/\eps^2$.

The algorithm collects at most $\tfrac{\mathrm{VC} + \ln (1/\delta)}{\eps^2}$ fresh samples in each iteration.
Generalization from this number of samples is guaranteed by the semantics of the weak agnostic learning oracle.
Since the number of iterations $T$ is at most $8/\eps^2$, the total sample complexity is bounded by 
\[
O \left( \frac{\mathrm{VC}(C) + \ln(1/\delta)}{\eps^4} \right)
\]

\end{proof} \section{Implementing an OWAL with Offline Sample Complexity}\label{sec:owal-off}
The goal of this section is to recapitulate some key definitions and results from \citet{wu2022expectedworstcaseregret} and explain how they imply the existence of an online weak agnostic learner with regret that scales with the offline sample complexity.

\begin{definition}[Definition 2 of \cite{wu2022expectedworstcaseregret}]\label{def:s-cover}
We say a class $G$ of functions $\Xcal \to [0,1]$ is a \textit{stochastic global sequential cover} of a class $\mathcal{C} \subseteq [0,1]^{\mathcal{X}}$ with respect to the class $\mathcal{P}$ of joint distributions over $x_1, \ldots, x_T$ at scale $\alpha > 0$ and confidence $\delta > 0$, if for all $\mu \in \mathcal{P}$, we have
\[
\Pr_{x_1, \ldots, x_T \sim \mu}\left[\exists c \in \mathcal{C}, \forall g \in G, \exists t \in [T] \text{ s.t. } |c(x_t) - g(x_t)| > \alpha \right] \leq \delta.
\]

We define the minimal size of $G$ to be the \textit{stochastic global sequential covering number} of $\mathcal{C}$.
\end{definition}

\cite{wu2022expectedworstcaseregret} show that bounds on the stochastic global sequential covering number imply bounds on the regret for arbitraty convex, lipschitz losses.

\begin{lemma}[Theorem 3 of \cite{wu2022expectedworstcaseregret}]\label{lem:hybrid-ol} Let $\ell(\cdot, y)$ be convex, $L$-Lipschitz, and bounded by $1$ for any $y \in [0,1]$, and let $\mathcal{C}$ be a set of functions $\mathcal{X} \to [0,1]$. Let $G_\alpha^\delta$ be a stochastic global sequential covering of $\mathcal{H}$ at scale $\alpha$ and confidence $\delta$ with respect to a class of distributions $\mathcal{P}$. Then there exists an algorithm that outputs a sequence of predictions $\hat{y}_1, \ldots, \hat{y}_T$ such that
\[
\sup_{\mu \in \mathcal{P}} \mathbb{E}_{x_1, \ldots, x_T \sim \mu} \left[ \sup_{y_1, \ldots, y_T} \left( \sum_{t=1}^T \ell(\hat{y}_t, y_t) - \inf_{c \in \mathcal{C}} \sum_{t=1}^T \ell(c(x_t), y_t) \right) \right] \leq \inf_{0 \leq \alpha \leq 1} \left\{\alpha LT + \sqrt{\frac{T}{2} \log |G_\alpha^\delta| + 1} \right\}.
\]
\end{lemma}
The theorem follows from running a multiplicative weights algorithm over the stochastic cover of the hypothesis class. Now we state how the size of the cover scales with the combinatorial dimensions of the class.

\begin{lemma}[Theorem 6 of \cite{wu2022expectedworstcaseregret}] For any binary-valued class $\mathcal{C}$ with finite VC-dimension, there exists a global sequential covering set $G$ of $\mathcal{C}$ with respect to the class of all i.i.d. distributions over $x_1, \ldots, x_T$ at scale $\alpha = 0$ and confidence $\delta$ such that for $T \geq e^9$, we have
\[
\log |G| \leq 5 \, \mathrm{VC}(\mathcal{C}) \log^2 T + \log T \log(1/\delta) + \log T.
\]
\end{lemma}

\begin{lemma}[Theorem 17 of \cite{wu2022expectedworstcaseregret}]\label{lem:fat-cover} Let $\mathcal{C}$ be a class of functions $\mathcal{X} \to [0,1]$ with the $\alpha$-fat shattering number $d(\alpha)$. Then there exists a stochastic global sequential covering set $G$ of $\mathcal{C}$ with respect to the class of all i.i.d. distributions over $x_1, \ldots, x_T$ at scale $\alpha$ and confidence $\delta$ such that
\[
\log |G| \leq 8d(\alpha/32)(\log T \log(1/\alpha))^4 + \log T \log (\log T/\delta) + O(1),
\]
where $O(1)$ hides absolute constants that are independent of $\alpha$, $T$, and $\delta$.
\end{lemma}

Note that these bounds cannot be improved up to log factors. See \cite{wu2022expectedworstcaseregret} for more. 

\begin{proof}[Proof of \Cref{cor:owal_lit}]
The proof of the claim follows mainly from Theorem 3 of \cite{wu2022expectedworstcaseregret}, stated in \Cref{lem:hybrid-ol}. We can apply this theorem since our loss function is convex and 1-Lipschitz. The algorithm runs a multiplicative weights algorithm over a set of experts indexed by a stochastic global sequential covering (\Cref{def:s-cover}). The statement of the claim requires the outputs of the algorithm to be in $\Ccal$. Given an arbitrary stochastic covering $G$ of $\Ccal$ at scale $\alpha$, it is straightforward to construct a stochastic covering $G' \subset \Ccal$ at scale $2\alpha$. Finally, the regret guarantee in \Cref{lem:hybrid-ol} scales 
with the VC (or fat-shattering dimension) of the class $\Ccal$ after plugging the bounds on the stochastic global sequential covering to the guarantees of \Cref{lem:hybrid-ol}. 
To ensure that each output hypothesis is in $\Ccal$ \footnote{$\Ccal$ may not be closed under convex combinations}, at each timestep, the multiplicative weights algorithm samples $c_t$ from a distribution over $G'$ instead of predicting the weighted combination $\sum_{i} w_t(i) c_i$. This only adds a factor of $\tilde{O}\left(\sqrt{T \ln 1/\delta} \right)$ to the regret term due to an application of Azuma-Hoedding to show that $\sum_{t=1}^T c_t(x_t) \cdot y_t$ converges to $\sum_{t=1}^T \E [c_t(x_t)] \cdot y_t$. 

\end{proof}
 
\end{document}